\documentclass[nohyperref]{article}

\usepackage{microtype}
\usepackage{graphicx}
\usepackage{subcaption}
\usepackage{makecell}
\usepackage{multirow}
\usepackage{booktabs} 

\usepackage{hyperref}

\usepackage{enumitem}

\usepackage[accepted]{icml2022}

\usepackage{amsmath}
\usepackage{amssymb}
\usepackage{mathtools}
\usepackage{amsthm}
\usepackage{thmtools, thm-restate}

\usepackage[capitalize,noabbrev]{cleveref}

\usepackage[textsize=tiny]{todonotes}

\newcommand{\G}{\mathbf{G}}
\newcommand{\g}{\mathbf{g}}

\renewcommand{\t}{\mathbf{t}}

\newcommand{\x}{\mathbf{x}}

\newcommand{\y}{\mathbf{y}}
\newcommand{\z}{\mathbf{z}}

\newtheorem{rem}{Remark}

\newtheorem{assum}{Assumption}

\newtheorem{lemma}{Lemma}

\icmltitlerunning{Anarchic Federated Learning}

\begin{document}

\twocolumn[
\icmltitle{Anarchic Federated Learning}




\begin{icmlauthorlist}
    \icmlauthor{Haibo Yang}{osu}
    \icmlauthor{Xin Zhang}{isu}
    \icmlauthor{Prashant Khanduri }{osu,umn}
    \icmlauthor{Jia Liu}{osu}

\end{icmlauthorlist}

\icmlaffiliation{osu}{Department of Electrical and Computer Engineering, The Ohio State University, Columbus, OH 43210, USA;}
\icmlaffiliation{isu}{Department of Statistics, Iowa State University, Ames, IA 50011, USA;}
\icmlaffiliation{umn}{Department of Electrical and Computer Engineering, University of Minnesota, Minneapolis, MN 55455, USA}

\icmlcorrespondingauthor{Haibo Yang}{yang.5952@osu.edu}
\icmlcorrespondingauthor{Jia Liu}{liu@ece.osu.edu}

\icmlkeywords{Machine Learning, ICML}

\vskip 0.3in
]



\printAffiliationsAndNotice{}  

\allowdisplaybreaks

\begin{abstract}
Present-day federated learning (FL) systems deployed over edge networks consists of a large number of workers with high degrees of heterogeneity in data and/or computing capabilities, which call for flexible worker participation in terms of timing, effort, data heterogeneity, etc.
To satisfy the need for flexible worker participation, we consider a new FL paradigm called ``Anarchic Federated Learning'' (AFL) in this paper.
In stark contrast to conventional FL models, each worker in AFL has the freedom to choose i) when to participate in FL, and ii) the number of local steps to perform in each round based on its current situation (e.g., battery level, communication channels, privacy concerns).
However, such chaotic worker behaviors in AFL impose many new open questions in algorithm design.
In particular, it remains unclear whether one could develop convergent AFL training algorithms, and if yes, under what conditions and how fast the achievable convergence speed is.
Toward this end, we propose two Anarchic Federated Averaging (AFA) algorithms with two-sided learning rates for both cross-device and cross-silo settings, which are named AFA-CD and AFA-CS, respectively. 
Somewhat surprisingly, we show that, under mild anarchic assumptions, both AFL algorithms achieve the best known convergence rate as the state-of-the-art algorithms for conventional FL.
Moreover, they retain the highly desirable {\em linear speedup effect} with respect of both the number of workers and local steps in the new AFL paradigm.
We validate the proposed algorithms with extensive experiments on real-world datasets.
\end{abstract}


\section{Introduction} \label{sec:intro}

Federated Learning (FL) has recently emerged as an important distributed learning framework that leverages numerous workers to collaboratively learn a joint model~\citep{li2019federated,yang2019federated,kairouz2019advances}. 
Since the inception, FL systems have become increasingly powerful and are able to handle various heterogeneity in data, network environments, worker computing capabilities, etc.
Furthermore, most of the prevailing FL algorithms (e.g., FedAvg~\citep{mcmahan2016communication} and its variants~\citep{li2018fedprox,zhang2020fedpd,karimireddy2020scaffold,karimireddy2020mime,acar2021feddyn}) enjoy a desirable ``{\em linear speedup effect,}''
i.e., the convergence time to a first-order stationary point decreases linearly as the number of workers and local steps increases~\citep{stich2018local,yu2019linear,wang2018cooperative,khaled2019first,karimireddy2020scaffold,yang2021achieving,qu2020federated}.


However, to achieve these salient features, most of the existing FL algorithms have adopted a {\em server-centric} approach, i.e., the worker behaviors are tightly ``dictated'' by the server.
Such dictation is typically manifested in three aspects:
i) determine either all or a subset of workers to participate in each round of FL update; 
ii) fully control the timing for synchronization and whether to accept/reject information sent from the workers;
iii) precisely specify the algorithmic operations (e.g., the number of local steps performed at each worker before communicating with the server).
Despite achieving strong performance guarantees, 
these server-centric FL algorithms often implicitly rely on the following strong assumptions: (1) each worker is available for training upon the server's request and throughout a complete round; 
(2) all participating workers are willing to execute the same number of local updates and communicate with the server in a synchronous manner following a common clock. 
Unfortunately, in edge networks where many FL systems are deployed, these assumptions are restrictive or even problematic. 
First, many requested edge devices on the worker side may not be available in each round because of, e.g., communication errors or battery outages. 
Second, the use of synchronous communication and an identical number of local updates across all workers ignores the fact that worker devices in edge-based FL systems are heterogeneous in computation and communication capabilities.
As a result, stragglers (i.e., slow workers) could significantly slow down the training process. 
To mitigate the straggler effect, various robust FL algorithms have been developed. 
For example, the server in FedAvg~\citep{mcmahan2016communication} can simply ignore and drop the information from the stragglers to speedup learning.
However, this may lead to other problems such as wasted computation/energy~\citep{wang2019adaptive}, slower convergence~\citep{li2018fedprox}, or biased/unfair uses of worker data~\citep{kairouz2019advances}.  
Moreover, the synchronous nature of the server-centric approaches implies many networking problems (e.g., interference between workers, periodic traffic spikes, high complexity in maintaining a network-wide common clock).

The above limitations of the current server-centric FL approaches motivate us to propose a new paradigm in FL, which we call {\bf Anarchic Federated Learning} (AFL).
In stark contrast to server-centric FL, workers in AFL are completely {\em free} of the ``dictation'' from the server.
Specifically, each worker has complete freedom to choose when and how long to participate in FL without following any control signals from the server.
As a result, the information fed back from workers is inherently asynchronous.
Also, each worker can independently determine the number of local update steps to perform in each round based on its current local situation (e.g., battery level, communication channels, privacy concerns).
In other words, the amount of local computation at each worker is time-varying, device-dependent, and fully controlled by the worker itself.
Clearly, AFL has a much lower server-worker coordination complexity and avoids the aforementioned pitfalls in server-centric FL approaches.
However, AFL also introduces significant challenges in algorithmic design on the server-side because the server needs to work much harder to handle the chaotic worker behaviors in AFL (e.g., asynchrony, spatial and temporal heterogeneity in computing).
Toward this end, several foundational questions naturally arise:
{\em 1) Is it possible to design algorithms that converge under AFL?
2) Under what condition and how fast could the algorithms converge?
3) If the answer to 1) is "yes" and 2) can also be resolved, could such algorithms still achieve the desired "linear speedup effect" as in conventional FL?}

In this paper, our goal is to obtain a fundamental understanding to the above questions.
Our main contributions and key results are summarized as follows:

\begin{list}{\labelitemi}{\leftmargin=0.2em \itemindent=-0.0em \itemsep=.1em}
\vspace{-.1in}
\item We consider a new FL paradigm called Anarchic Federated Learning (AFL), where the workers are allowed to engage in training {\em at will} and choose the number of local update steps based on their own time-varying situations (computing resources, energy levels, etc.).
This {\em loose worker-server coupling} significantly simplifies the implementations and renders AFL particularly suitable for FL deployments in edge computing environments.
For any AFL algorithms under general worker information arrival processes and non-i.i.d. data across workers, we first establish a fundamental convergence error lower bound to characterize the effect of worker participation in the AFL system.

\item 
For AFL in the {\em cross-device} (CD) setting, we study the convergence of an anarchic federated averaging algorithm (AFA-CD), which is a natural counterpart of the popular FedAvg algorithm \citep{mcmahan2016communication} for server-centric FL.
Our analysis reveals that, under bounded maximum delay, AFA-CD converges to an error ball whose size matches the fundamental lower bound, with an $\mathcal{O}(1/\sqrt{mKT})$ convergence rate.
Here, $m$ is the number of collected workers in each round of update, $K$ is the local steps and $T$ is the total number of rounds.
We note that this convergence rate retains the highly desirable ``linear speedup effect'' in both worker's number $m$ and local steps $K$ under AFL.\footnote{To attain $\epsilon$-accuracy, it takes $\mathcal{O}(1/\epsilon^2)$ steps for an algorithm with an $\mathcal{O}(1/\sqrt{T})$ convergence rate, while needing $\mathcal{O}(1/m \epsilon^2)$ steps for another algorithm with an  $\mathcal{O}(1/\sqrt{mT})$ convergence rate (the hidden constant in Big-O is the same). In this sense, $\mathcal{O}(1/\sqrt{mT})$ implies a {\em linear speedup} in terms of $m$.}
Moreover, under the stronger condition of uniform workers' participation in AFL,
AFA-CD converges to a singleton stationary point at the same convergence rate order that matches the state-of-the-art of server-centric FL.

\item For AFL in the {\em cross-silo} (CS) setting, we study the convergence of a CS version of AFA (AFA-CS), where the special features of CS allow one to leverage historical feedback information and variance reduction techniques.
We show that AFA-CS achieves an improved convergence rate of $\mathcal{O}(1/\sqrt{MKT})$, 
where $M$ is the total number of workers. 
This suggests that, not only can ``linear speedup'' be achieved under AFA-CS, the speedup factor also depends on the total number of workers $M$ instead of the number of collected workers $m$ in each round ($M > m$).

\item We validate both AFA algorithms with extensive experiments on CV and NLP tasks and explore the effect of the asynchrony and local step number in AFL. 
We also numerically show that AFA can be integrated with various advanced FL techniques (e.g., FedProx~\citep{li2018fedprox} and SCAFFOLD~\citep{karimireddy2020scaffold}) to further enhance the AFL performance.

\end{list}


The rest of the paper is organized as follows.
In Section~\ref{sec:related_work}, we review related work.
In Section~\ref{sec:afl}, we introduce AFL and the AFA algorithms, which are followed by their convergence analysis in Section~\ref{sec:convergence}.
We present the numerical results in Section~\ref{sec:numerical} and conclude the work in Section~\ref{sec:conclusion}. 
Due to space limitation, we relegate all proofs and some experiments to the supplementary material.

\section{Related Work} \label{sec:related_work}

\textbf{1) Server-Centric FL Algorithms:}
To date, one of the prevailing FL algorithms is Federated Averaging (FedAvg), which was first proposed in \citep{mcmahan2016communication} as a heuristic to improve communication efficiency and data privacy for FL.
Since then, there have been substantial follow-ups of FedAvg that focus on non-i.i.d. (heterogeneous) data (see, e.g., FedProx~\citep{li2018fedprox}, FedPD~\citep{zhang2020fedpd}, SCAFFOLD~\citep{karimireddy2020scaffold}, FedNova~\citep{wang2020fednova}, FedDyn~\citep{acar2021feddyn}, and MIME~\citep{karimireddy2020mime}), which are closely related to our work.
The main idea for these algorithms is to control the ``model drift'' (due to heterogeneous datasets and the use of multiple local update steps on the worker side).
While achieving various degrees of success in handling data heterogeneity, these algorithms are all server-centric and synchronous, which are more restrictive in edge-based settings (see discussions in Section~\ref{sec:intro}).

\textbf{2) FL with Flexible Worker Participation:}
To achieve high concurrency and avoid stragglers, researchers have proposed various FL methods with flexible worker participation, which can be roughly categorized into three classes:
The first class utilizes different local steps to accommodate worker heterogeneity, while maintaining a synchronous communication between the server and workers~\cite{wang2020fednova,ruan2021towards,avdiukhin2021federated}.
The second class is based on asynchronous distributed optimization~\cite{zhang2015deep,lian2018asynchronous,niu2011hogwild,agarwal2012distributed,paine2013gpu,xie2019asynchronous,zhang2020taming} (with identical local steps)~\cite{nguyen2021federated,avdiukhin2021federated}.
Specifically, \citet{xie2019asynchronous} proposed an asynchronous FL (FedAsync) method to tackle stragglers and heterogeneous latency, where the server continuously triggers one worker for local training.
However, this work did not consider the convergence performance of non-convex problems that are more relevant to learning.
~\citet{nguyen2021federated} utilized buffered asynchronous aggregation and achieved an  $O(\frac{1}{\sqrt{TK}})$ convergence rate for non-convex problems, but it was unclear whether a linear speedup in terms of $m$ is achievable.
~\citet{avdiukhin2021federated} proposed AsyncCommSGD by allowing asynchronous communication, while assuming an identical computation rate across all workers.
This work achieved an $O(\frac{1}{\sqrt{mKT}})$ convergence rate for non-convex problems under a bounded gradient assumption, matching the convergence rate of synchronous FedAvg.
The third class considers arbitrary device unavailability, though the server and workers still communicate in a synchronous fashion (i.e., following a system-wide common clock).
In this class, the algorithms in~\citep{gu2021fast} and ~\citep{yan2020distributed} achieve $O(\frac{1}{\sqrt{MKT}})$ and $O(\frac{1}{\sqrt{M^{0.5}T}})$, respectively. 
However, ~\citet{gu2021fast} required a Lipschitz Hessian assumption, where ~\citet{yan2020distributed} needed a bounded stochastic gradient assumption.

\textbf{3) Key Differences of AFL from Related Works:}
Comparing to the aforementioned related works, the AFL paradigm in this paper allows both heterogeneity: i)  different local steps across workers and ii) asynchronous communications between the server and workers.
In other words, the AFL paradigm subsumes all the above settings as special cases.
Moreover, AFL fundamentally differs from aforementioned FL algorithms in that the worker's participation in AFL and its local optimization process are completely {\em determined by the workers, and not by the sampling requests from the server.}
This is more practical since it allows workers to participate in FL under drastically different situations in the network states, charging/idle cycles, etc.
Due to the complex couplings between multiple sources of randomness and layers of heterogeneity in spatial and temporal domains in AFL, the training algorithm design for AFL and its theoretical analysis is far from a straightforward combination of existing FL techniques.
Interestingly, we show that the AFA algorithms (counterparts of FedAvg under AFL) achieve the same convergence rate without strong assumptions (e.g., Lipschitz Hessian condition in \cite{gu2021fast}).

\section{Anarchic Federated Learning} \label{sec:afl}

In this section, we first formally define the notion of AFL.
Then, we will present the AFA algorithmic framework, which contains two variants called AFA-CD and AFA-CS for cross-device and cross-silo AFL, respectively.

\textbf{1) Overview of Anarchic Federated Learning:} 
The goal of FL is to solve an optimization problem in the form of $\min_{x \in \mathbb{R}^d} f(x) := \frac{1}{M} \sum_{i=1}^{M} f_i(x)$, where $f_i(x) \triangleq \mathbb{E}_{\xi_i \sim D_i}[f_i(x, \xi_i)]$ is the local (non-convex) loss function associated with a local data distribution $D_i$ and $M$ is the total number of workers.
For the setting with heterogeneous (non-i.i.d.) datasets at the workers, we have $D_i \neq D_j, \text{if } i \neq j$.
In terms of the assumption on the size of workers, FL can be classified as cross-device FL and cross-silo FL~\citep{kairouz2019advances,wang2021field}. 
Cross-device FL is designed for large-scale FL with a massive number of mobile or IoT devices ($M$ is large).
As a result, the server can only afford to collect information from a subset of workers in each round of update and is unable to store workers' information across rounds.
In comparison, the number of workers in cross-silo FL is relatively small.
Although the server in cross-silo FL may still have to collect information only from a subset of workers in each round, it has enough capacity to store each worker's most recent information.

\begin{algorithm}[t!]
\caption{The General AFL Framework.} \label{alg:afl} 
\begin{algorithmic}
\STATE 
\emph{\bf At the Server (Concurrently with Workers):} 
\begin{enumerate}[topsep=0.4em, leftmargin=1.5em, itemsep=-0.5pt]
	\item (Concurrent Thread) Collect local updates returned from the workers.
	\item (Concurrent Thread) Aggregate local update returned from collected workers and update global model following some server-side optimization process.
\end{enumerate}

\vspace{0.1em}
\STATE 
\emph{\bf At Each Worker (Concurrently with Server):}
\begin{enumerate}[topsep=0.4em, leftmargin=1.5em, itemsep=-0.5pt]
\item Once decided to participate in the training, pull the global model with current timestamp.
\item Perform (multiple) local update steps following some worker-side optimization process.
\item Return the result and the associated pulling timestamp to the server, with extra processing if so desired.
\end{enumerate}
\end{algorithmic}
\end{algorithm}

The general framework of AFL is illustrated in Algorithm~\ref{alg:afl}.
Here, the server is responsible for: 
1) collecting the local updates returned from workers, 
and 2) aggregating the obtained updates once certain conditions are satisfied (e.g., upon collecting $m \in (0,M]$ local updates from workers) to update the global model.
Note that these two threads are {\em concurrent}, so it completely avoids system locking on server's side.
Also, idling is allowed at each worker between each two successive participations in training.
Whenever a worker intends to participate in the training, it first pulls the current model parameters from the server.
Then, upon finishing multiple local update steps (more on this later) by some worker-side optimization process (e.g., using stochastic gradients or additional information such as variance-reduced and/or momentum adjustments), the worker reports the results to the server (potentially with extra processing if so desired, e.g., compression for communication efficiency).

We remark that AFL is a general computing architecture that subsumes the conventional FL and asynchronous distributed optimization as special cases.
From an optimization perspective, the server and the workers may adopt independent optimization processes, thus enabling a much richer set of learning {\em ``control knobs''} (e.g., separated learning rates, separated batch sizes).
Specifically, each worker is able to completely take control of its own optimization process, even using a time-varying number of local update steps and optimizers, which depend on its local dataset and/or its device status (e.g., battery level, privacy preference).
More importantly, from the system level, the concurrent processes at worker and server side enable loose worker-server coupling and thus avoiding server-worker interlocking and reducing synchronization overhead.

\textbf{2) A Convergence Error Lower Bound for AFL:}
To thoroughly understand AFL, we will first obtain some fundamental insights on the {\em performance limit of any AFL training algorithms}.
Toward this end, we first state several assumptions that are needed for our theoretical analysis throughout the rest of this paper.

\begin{assum}($L$-Lipschitz Continuous Gradient) \label{a_smooth}
	There exists a constant $L > 0$, such that $ \| \nabla f_i(\x) - \nabla f_i(\y) \| \leq L \| \x - \y \|$, $\forall \x, \y \in \mathbb{R}^d$, and $i \in [M]$.
\end{assum}

\begin{assum}(Unbiased Local Stochastic Gradient) \label{a_unbias}
	Let $\xi^i$ be a random local data sample at worker $i$.
	The local stochastic gradient is unbiased, i.e.,
	$\mathbb{E} [\nabla f_i(\x, \xi^i)] = \nabla f_i(\x)$, $\forall i \in [m]$, where the expectation is taken over the local data distribution $D_i$.
\end{assum}

\begin{assum}(Bounded Local and Global Variances) \label{a_variance}
	There exist two constants $\sigma_L \geq 0$ and $\sigma_G \geq 0$, such that the variance of each local stochastic gradient estimator is bounded by
	$\mathbb{E} [\| \nabla f_i(\x, \xi^i) -  \nabla f_i(\x) \|^2] \leq \sigma_{L}^2, \forall i \in [M]$,
	and the global variability of local gradient of the cost function is bounded by 
	$\| \nabla f_i(\x) - \nabla f(\x) \|^2 \leq \sigma_{G}^2, \forall i \in [M]$.
\end{assum}
The first two assumptions are standard in the convergence analysis of non-convex optimization~(see, e.g., \citep{ghadimi2013stochastic,bottou2018optimization}).
For Assumption~\ref{a_variance}, the bounded local variance is also a standard assumption.
We utilize a universal bound $\sigma_G$ to quantify the data heterogeneity among different workers.
This assumption is also frequently used in the literature of FL with non-i.i.d. datasets ~\citep{reddi2020adaptive,wang2019adaptive,yang2021achieving} as well as in decentralized optimization~\citep{kairouz2019advances}.

To establish a fundamental convergence error lower bound, we consider the most general case where {\em no assumption} on the arrival processes of the worker information is made, except that each worker's participation in FL is independent of each other.
In such general worker information arrival processes, we prove the following lower bound of convergence error by constructing a worst-case scenario:

\begin{restatable}[Convergence Error Lower Bound for AFL with General Worker Information Arrival Processes] {theorem}{lb}
	\label{thm:lb}
	For any level of heterogeneity characterized by $\sigma_G$, there exists loss functions  satisfying Assumptions~\ref{a_smooth}-~\ref{a_variance} and a specific worker participation process for which the output ${\hat{\x}}$ of any convergent (and potentially random) FL algorithm satisfies: 
	\begin{align*}
		\mathbb{E}[ \| \nabla f(\hat{\x}) \|^2 ] = \Omega (\sigma_G^2).
	\end{align*}
\end{restatable}

\begin{rem}{\em (Proof in Appendix~\ref{appdx_lb})
	The lower bound in Theorem~\ref{thm:lb} indicates that no algorithms for AFL could converge to a stationary point under general worker information arrival processes, due to the significant system heterogeneity and randomness caused by such general worker information arrivals.
	The rationale is that there always exist objective value drifts owing to general worker information arrival processes in the worst-case scenario, which further lead to an inevitable error in convergence.
	We note that this lower bound is different from previous optimization lower bounds in FL ~\cite{karimireddy2020scaffold,woodworth2020minibatch,gu2021fast}.
	Our lower bound captures objective deviations due to worker participation while previous bounds focus on the optimization error with ideal worker participation (i.e., full worker or uniformly random worker participation).
	Considering a worst-case scenario in FL by removing such assumption of ideal worker participation, our lower bound also holds for non-i.i.d. FL including synchronous FedAvg and its variants, thus also providing insights for conventional FL. 
	To ensure convergence to a stationary point, extra assumptions for the worker information arrivals need to be made, e.g., uniformly distributed arrivals (see Theorem~\ref{thm:uniform}) and bounded delays (see Theorem~\ref{thm:silo}).
	}
\end{rem}


\section{The Anarchic Federated Averaging (AFA) Algorithms for AFL} \label{sec:convergence}

Upon obtaining a basic understanding of the training algorithm performance limit from the convergence error in Theorem~\ref{thm:lb}, in this section, we study convergence conditions and performance of two anarchic federated averaging (AFA) algorithms for cross-device (CD) and cross-silo (CS) settings in Section~\ref{subsec:cd} and \ref{subsec:cs}, respectively, both of which can be viewed as an extension of FedAvg under AFL.

\subsection{The AFA-CD Algorithm for Cross-Device AFL} \label{subsec:cd}

{\bf 1) The AFA-CD Algorithm:}
First, we consider the AFA-CD algorithm for the cross-device AFL setting.
As mentioned earlier, cross-device AFL is suitable for cases with a massive number of edge devices.
In each round of global model update, only a small subset of workers are used in the training.
The server is assumed to have {\em no historical information} of the workers.
As shown in Algorithm~\ref{alg:FedAvg_device}, AFA-CD closely follows the AFL architecture shown in Algorithm~\ref{alg:afl}.
Here, we use the standard stochastic gradient descent (SGD) method as the server- and worker-side optimizer.
In each update $t=0,\ldots,T-1$, the server waits until collecting $m$ local updates $\{ \G_i(\x_{t - \tau_{t,i}}) \}$ from workers to form a set $\mathcal{M}_t$ with $|\mathcal{M}_t| = m$, where $\tau_{t, i}$ represents the random delay of the local update of worker $i \in \mathcal{M}_t$ (Server Code, Line~1).
Once $\mathcal{M}_t$ is formed, the server aggregates all local updates $\G_i(\x_{t - \tau_{t,i}}), i \in \mathcal{M}_t$ and updates global model (Server Code, Line~2).
We count each global model update as one communication round for direct comparison with previous FL results. 
Meanwhile, for each worker, it pulls the current global model parameters with time stamp $\mu$ once it decides to participate in training (Worker Code, Line~1).
Each worker can then choose a desired number of local update steps $K_{t,i}$ (could be time-varying and device-dependent) to perform SGD updates for $K_{t,i}$ times, and then return the rescaled sum of all stochastic gradients with timestamp $\mu$ to the server (Worker Code, Lines~2--3).

\begin{algorithm}[t!]
\caption{AFA-CD Algorithm for Cross-Device AFL.} \label{alg:FedAvg_device} 
\begin{algorithmic}
\STATE 
\emph{\bf At the Server (Concurrently with Workers):}
\begin{enumerate}[topsep=0.4em, leftmargin=1.5em, itemsep=-0.5pt]
\item In the $t-$th update round, collect $m$ local updates $\{\G_i(\x_{t - \tau_{t,i}}), i \in \mathcal{M}_t \}$ returned from the workers to form the set $\mathcal{M}_t$, where $\tau_{t,i}$ represents the random delay of the worker $i$'s local update, $i \in \mathcal{M}_t$.
\item Aggregate and update:
$\G_t = \frac{1}{m} \sum_{i \in \mathcal{M}_t} \G_i(\x_{t - \tau_{t,i}})$, \quad
$\x_{t+1} = \x_t - \eta \G_t$.
\end{enumerate}
\STATE 
\emph{\bf At Each Worker (Concurrently with Server):}
\begin{enumerate}[topsep=0.4em, leftmargin=1.5em, itemsep=-0.5pt]
\item Once decided to participate in the training, retrieve the parameter $\x_{\mu}$ from the
server and its timestamp, set the local model: $\x_{\mu, 0}^i = \x_{\mu}$. \\
\item Choose a number of local steps $K_{t, i}$, which can be time-varying and device-dependent. 
Let $\x_{\mu, k+1}^i = \x_{\mu, k}^i - \eta_L \g_{\mu, k}^{i}$, where $\g_{\mu, k}^{i} = \nabla f_i(\x_{\mu, k}^i, \xi_{\mu, k}^i)$, $k=0,\ldots,K_{t,i}-1$.
\item Sum and rescale the stochastic gradients: $\G_i(\x_{\mu}) = \frac{1}{K_{t, i}} \sum_{j=0}^{K_{t, i}-1} \g_{\mu, j}^i$. Return $\G_i(\x_{\mu})$.
\end{enumerate}
\end{algorithmic}
\end{algorithm}

{\bf 2) Convergence Analysis of the AFA-CD Algorithm:}
We first analyze the convergence of AFA-CD under general worker information arrival processes.
We use $f_0 = f(\x_0)$ and $f_{*}$ to denote the initial and the optimal objective values, respectively.
We have the following convergence result for the AFA-CD algorithm (see proof details in Appendix~\ref{appdx_arbipdx_uniform}):

\begin{restatable}[AFA-CD with General Worker Information Arrival Processes] {theorem}{arbitrary}
\label{thm:arbitrary}
Suppose that the resultant maximum delay under AFL is bounded, i.e., $\tau := \max_{t \in [T], i \in \mathcal{M}_t} \{ \tau_{t,i} \} < \infty$.
Suppose that the server-side and worker-side learning rates $\eta$ and $\eta_L$ are chosen as such that the following conditions are satisfied:
$6 \eta_L^2 (2K_{t,i}^2 - 3K_{t,i} + 1)L^2 \leq 1, 180 \eta_L^2 K_{t, i} ^2 L^2 \tau < 1, \forall t, i$ and $2 L \eta \eta_L + 6 \tau^2 L^2 \eta^2 \eta_L^2 \leq 1$.
Under Assumptions~\ref{a_smooth}--\ref{a_variance}, the output sequence $\{ \x_t \}$ generated by AFA-CD with general worker information arrival processes satisfies:
\begin{align*}
\frac{1}{T} \sum_{t=0}^{T-1} \mathbb{E} \| \nabla f(\x_t) \|^2 \leq \frac{4(f_0 - f_*)}{\eta \eta_L T} + 4 \big( \alpha_L \sigma^2_L + \alpha_G \sigma_G^2 \big),
\end{align*}
where the constants $\alpha_L$ and $\alpha_G$ are defined as:
\begin{align*}
    \alpha_L &= \frac{L \eta \eta_L}{m} \frac{1}{T} \sum_{t=0}^{T-1} \frac{1}{K_t} + \frac{3 \tau^2 L^2 \eta^2 \eta_L^2}{m} \frac{1}{T} \sum_{t=0}^{T-1} \frac{1}{K_t} \\
    & \quad + \frac{15 \eta_L^2 L^2}{2} \frac{1}{T} \sum_{t=0}^{T-1} \bar{K}_t , \\
    \alpha_G &= \frac{3}{2} + 45 L^2 \eta_L^2 \frac{1}{T} \sum_{t=0}^{T-1} \hat{K}_t^2.
\end{align*}
Here,
\begin{align*}
     \frac{1}{K_t} \!=\! \frac{1}{m} \! \sum_{i \in \mathcal{M}_t} \! \frac{1}{K_{t,i}},
    \bar{K}_t \!=\! \frac{1}{m} \! \sum_{i \in \mathcal{M}_t} \! K_{t, i},
    \hat{K}_t^2 \!=\! \frac{1}{m} \! \sum_{i \in \mathcal{M}_t} \! K_{t, i}^2.
\end{align*}
\end{restatable}

The learning rates conditions imply that $\eta \eta_L = \mathcal{O}(\frac{1}{\tau L}) $ and $\eta_L^2 \leq \frac{1}{K_{t,i}^2}$, which is a natural extension of that in SGD. 
With Theorem~\ref{thm:arbitrary}, if we assume a constant local step number and proper learning rates, we immediately have the following convergence rate for AFA-CD, which implies the ``linear speedup effect'' in both $m$ and $K$.

\begin{restatable}[Linear Speedup to an Error Ball] {corollary}{arbitraryC}
\label{cor:arbitrary}
Suppose a constant local step $K$ for each worker, by setting $\eta_L = \frac{1}{\sqrt{T}}$, and $\eta = \sqrt{mK}$,
the convergence rate of AFA-CD with general worker information arrival processes is:
$$\mathcal{O}\bigg(\frac{1}{m^{1/2} K^{1/2} T^{1/2}} \bigg) + \mathcal{O} \bigg(\frac{\tau^2}{T} \bigg) + \mathcal{O}\bigg(\frac{K^2}{T} \bigg) + \mathcal{O}(\sigma_G^2).$$
\end{restatable}

\begin{rem}{\em
Clearly, due to the chaotic worker behaviors in AFL, one cannot expect that an algorithm for AFL can converge under any arbitrary condition.
Theorem~\ref{thm:arbitrary} and Corollary~\ref{cor:arbitrary} suggest that, as long as the consequence of the chaotic workers behaviors remains ``manageable'' in the sense that i) the maximum delay due to asynchrony is bounded and ii) the learning rates used by the workers and server are sufficiently small, then the iterates produced by AFA-CD can converge to a neighborhood around a stationary point.
Moreover, if the workers are less ``anarchic'' in the sense that they know the $T$-value from the server and are willing to set $\eta_L = \frac{1}{\sqrt{T}}$ accordingly, then the non-vanishing error term $\mathcal{O}(\sigma_G^2)$ in Corollary~\ref{cor:arbitrary} matches the lower bound in Theorem~\ref{thm:lb}.
This implies that the convergence error of AFL-CD is {\em order-optimal} in this setting.
}
\end{rem}

\begin{rem}{\em
Recall that the non-vanishing convergence error $\mathcal{O}(\sigma_G^2)$ in Corollary~\ref{cor:arbitrary} is a consequence of objective function drift under the general worker information arrivals (no assumption on the arrivals of the worker participation in each round of update) and is independent of the choices of learning rates, the number of local update steps, and the number of global update rounds (more discussion in the supplementary material).
Also, for a sufficiently large $T$, the dominant term $\mathcal{O}(\frac{1}{m^{1/2} K^{1/2} T^{1/2}})$ implies that AFA-CD achieves the linear speedup in terms of $m$ and $K$ before reaching a constant error neighborhood with size $\mathcal{O}(\sigma_G^2)$.
}
\end{rem}
%
%

Given the weak convergence result under general workers' information arrivals, it is important to understand what extra conditions on the worker information arrivals are needed under AFL in order to achieve stronger convergence performance.
Toward this end, we consider a special setting where the arrivals of worker returned information in each round for global update is uniformly distributed among the workers.
In this setting, $\mathcal{M}_t$ can be viewed as a subset with size $m$ independently and uniformly sampled from $[M]$ without replacement.
It has been empirically found in~\citep{mcmahan2016communication,li2019federated} that, for FL systems with a massive number of workers, the assumption of uniformly distributed arrivals is a good approximation for worker participation in cross-device FL.
In what follows, we show that the convergence performance of AFA-CD in this special setting can be improved as follows (see proof details in Appendix~\ref{appdx_uniform}):

\begin{restatable} {theorem}{uniform}
\label{thm:uniform}
Under the same delay condition in Theorem~\ref{thm:arbitrary} and suppose  that the server-side and worker-side learning rates $\eta$ and $\eta_L$ are chosen as such that the following relationships hold:
$6 \eta_L^2 (2K_{t,i}^2 - 3K_{t,i} + 1)L^2 \leq 1, \forall t, i$, $L \eta \eta_L + L^2 \eta^2 \eta_L^2 \tau^2 \leq \frac{1}{2M}$, and
$120 L^2 \hat{K}_t^2 \eta_L^2 \tau < 1, \forall t$.
Then, under Assumptions~\ref{a_smooth}--~\ref{a_variance}, the output sequence $\{ \x_t \}$ generated by AFA-CD with uniformly distributed worker information arrivals satisfies:
\begin{align*}
\frac{1}{T} \sum_{t=0}^{T-1} \mathbb{E} \| \nabla f(\x_t) \|^2_2 \leq \frac{4(f_0 - f_*)}{ \eta \eta_L T} + 4 \big( \alpha_L \sigma^2_L + \alpha_G \sigma_G^2 \big),
\end{align*}
where $\alpha_L$ and $\alpha_G$ are defined as following:
\begin{align*}
    \alpha_L &= \frac{L \eta \eta_L}{m} \frac{1}{T} \sum_{t=0}^{T-1} \frac{1}{K_t} + \frac{2 \tau^2 L^2 \eta^2 \eta_L^2}{m} \frac{1}{T} \sum_{t=0}^{T-1} \frac{1}{K_t} \\
    & \quad + 5 \eta_L^2 L^2 \frac{1}{T} \sum_{t=0}^{T-1} \bar{K}_t , \\
    \alpha_G &= 30 L^2 \eta_L^2 \frac{1}{T} \sum_{t=0}^{T-1} \hat{K}_t^2,
\end{align*}
and other parameters are defined the same as in Theorem~\ref{thm:arbitrary}.
\end{restatable}

The requirement for learning rates could be easily satisfied as that in Theorem~\ref{thm:arbitrary}.
Furthermore, with appropriate server- and worker-side learning rates, we immediately have the following linear speedup convergence result for AFA-CD:
\begin{restatable}[Linear Speedup to a Stationary Point] {corollary}{uniformC}
\label{cor:uniform}
Suppose a constant local step $K$, let $\eta_L = \frac{1}{\sqrt{T}}$, and $\eta = \sqrt{mK}$,
the convergence rate of AFA-CD with uniformly distributed worker information arrivals is: 
\begin{align*}
    &\mathcal{O}(\frac{1}{m^{1/2} K^{1/2} T^{1/2}}) + \mathcal{O}\bigg(\frac{\tau^2}{T}\bigg) + \mathcal{O}\bigg(\frac{K^2}{T}\bigg)
\end{align*}
\end{restatable}

\begin{rem}{\em
For a sufficiently large $T$, the linear speedup convergence to a stationary point (rather than a constant error neighborhood) can be achieved under bounded maximum delay $\tau$, i.e., $\mathcal{O}(\frac{1}{m^{1/2} K^{1/2} T^{1/2}})$.
Note that this rate does not depend on the delay $\tau$ after sufficiently many rounds $T$ (i.e., $\tau \leq \min \{ \frac{T^{1/4}}{m^{1/4} K^{1/4}}, \frac{T^{1/2}}{m^{1/2} K^{5/2}} \}$), the negative effect of using outdated information in such an asynchronous setting vanishes asymptotically.
Further, for $\sigma_G=0$ (i.i.d. data) and $K=1$ (single local update step), AFA-CD can be viewed as an extension of the AsySG-con algorithm~\citep{lian2015asynchronous} in asynchronous parallel distributed optimization.
It can be readily verified that AFA-CD achieves the same rate as that of the AsySG-con algorithm.
Furthermore,  
AsyncCommSGD~\cite{avdiukhin2021federated} achieves $O(\frac{1}{\sqrt{mKT}})$ for FL by allowing asynchronous communication assuming an identical computation rate across workers and bounded gradients.
Interestingly, AFA-CD achieves the same convergence rate while allowing flexible worker participation and without such assumptions.
Surprisingly, this rate even matches the best known rate for the general non-convex setting in FL~\cite{karimireddy2020scaffold,reddi2020adaptive}.
It is worth noting that ~\citet{nguyen2021federated} proposed the FedBuff algorithm for FL, which is akin to AFA-CD and boosts FL concurrency.
However, FedBuff achieves an $O(\frac{1}{\sqrt{TK}})$ convergence rate, which does not achieve the linear speedup in terms of $m$.
}
\end{rem}

\subsection{The AFA-CS Algorithm for Cross-Silo AFL} \label{subsec:cs} 

{\bf 1) The AFA-CS Algorithm:} As mentioned earlier, cross-silo FL is suitable for collaborative learning among a relatively small number of (organizational) workers.
Thanks to the relatively small number of workers, each worker's feedback can be stored at the server.
As a result, the server could reuse the historical information of each specific worker in each round of global update.

\begin{algorithm}[t!]
\caption{The AFA-CS Algorithm for Cross-Silo AFL.} \label{alg:FedAvg_silo} 
\begin{algorithmic}
\STATE 
\emph{\bf At the Server (Concurrently with Workers):}
\begin{enumerate}[topsep=0.4em, leftmargin=1.5em, itemsep=-0.5pt]
\item In the $t-$th update round, collect $m$ local updates.
%
\item Update worker $i$'s information in the memory using the returned local update $\G_i$.
\item Aggregate and update:
$\G_t = \frac{1}{M} \sum_{i \in [M]} \G_i$, \quad
$\x_{t+1} = \x_t - \eta \G_t$.
\end{enumerate}
\STATE 
\emph{\bf At Each Worker (Concurrently with Server):} Same as AFA-CD Worker Code.
\end{algorithmic}
\end{algorithm}

As shown in Algorithm~\ref{alg:FedAvg_silo}, the AFA-CS algorithm also closely follows the AFL architecture as shown in Algorithm~\ref{alg:afl}.
In each round of global model update, a subset of workers could participate in the training (Server Code, Line~1).
Compared to AFA-CD, the key difference in AFA-CS is in Line~2 of the Server Code, where the server stores the collected local updates $\{\G_i \}$ for each worker $i \in \mathcal{M}_t$ into the memory space at the server (Server Code, Line~2). 
As a result, whenever a worker $i$ returns  a local update to the server upon finishing its local update steps, the server will update the memory space corresponding to worker $i$ to replace the old information with this newly received update from worker $i$.
Similar to AFA-CD, every $m$ new updates in the AFA-CS algorithm trigger the server to aggregate all the $\G_i, i \in [M]$ 
 and update the global model.
The Worker Code in AFA-CS is exactly the same as AFA-CD and its description is omitted for brevity.

{\bf 2) Convergence Analysis of the AFA-CS Algorithm:} 
We divide stochastic gradient returns $\{\G_i \}$ into two groups.
One is for those without delay ($\G_i(x_t), i \in \mathcal{M}_t, | \mathcal{M}_t | =m'$) and the other is for those with delay ($\G_i(\x_{t-\tau_{t, i}}), i \in \mathcal{M}_t^c, | \mathcal{M}_t^c | = M - m'$).
For cross-silo AFL, the AFA-CS algorithm achieves the following convergence performance (see proof details in Appendix~\ref{appdx_silo}):

\begin{restatable} {theorem}{silo}
\label{thm:silo}
Suppose that the resultant maximum delay in the system is bounded, i.e.,  $\tau := \max_{\t \in [T], i \in \mathcal{M}_t^c} \{ \tau_{t,i} \} < \infty$. Suppose that the server-side and worker-side learning rates $\eta$ and $\eta_L$ are chosen as such that the following relationships hold: $6 \eta_L^2 (2K_{t,i}^2 - 3K_{t,i} + 1)L^2 \leq 1, \forall t, i$, 
$\left( \frac{\eta \eta_L (M - m^{'})^2 L^2 \tau^2}{M^2} + \frac{L}{2} \right) \eta \eta_L \leq \frac{1}{4},$
and $\frac{30 L^2 \eta_L^2 \tau}{M} \left(\sum_{i \in [M]} K_{t,i}^2 \right) \leq \frac{1}{4}.$
Then, under Assumptions~\ref{a_smooth}-~\ref{a_variance}, the output sequence $\{ \x_t \}$ generated by the AFA-CS algorithm under general worker information arrival processes 
satisfies:
\begin{align*}
    \frac{1}{T} \sum_{t=0}^{T-1} \left\| \nabla f(\x_t) \right\|^2 &\leq \frac{4 f(\x_0) - f(\x_T)}{\eta \eta_L T} + \alpha_L \sigma_L^2 + \alpha_G \sigma_G^2,
\end{align*}
where the constants $\alpha_{L}$ and $\alpha_{G}$ are defined as follows:
\begin{align*}
    \alpha_L &= \frac{4}{M} \bigg[5L^2 \eta_L^2 \frac{1}{T} \sum_{t=0}^{T-1} \bar{K}_t \\
    &\quad + \left( \frac{2\eta^2 \eta_L^2 (M - m^{'})^2 L^2 \tau^2}{M^2} + L \eta \eta_L \right) \frac{1}{T} \sum_{t=0}^{T-1} \frac{1}{K_t} \bigg], \\
    \alpha_G &= \frac{120L^2 \eta_L^2}{M} \frac{1}{T} \sum_{t=0}^{T-1} \hat{K}_t^2,
\end{align*}

and other parameters are defined the same as in Theorem~\ref{thm:arbitrary}.
\end{restatable}



With appropriate learning rates, we immediately have {\em stronger} linear speedup convergence:
\begin{restatable}[Linear Speedup] {corollary}{siloC}
\label{cor:silo}
Suppose a constant local step $K$, and let $\eta_L = \frac{1}{\sqrt{T}}$, and $\eta = \sqrt{MK}$,
the convergence rate of the AFA-CS algorithm under general worker information arrival processes is: 
$$ \mathcal{O}\bigg(\frac{1}{M^{1/2} K^{1/2} T^{1/2}}\bigg) + \mathcal{O}\bigg(\frac{K^2}{MT}\bigg) + \mathcal{O}\bigg(\frac{\tau^2 (M-m^{'})^2}{TM^2}\bigg).$$
\end{restatable}

\begin{rem}{\em
Compared to Corollary~\ref{cor:arbitrary}, we can see that, by reusing historical data, AFA-CS can eliminate the non-vanishing $\mathcal{O}(\sigma^2_G)$ error term even under general worker information arrival processes and bounded delay.
The bounded delay implicitly requires each workers at least participate in the training process, eliminating the worst-case scenario in Theorem~\ref{thm:lb}.
On the other hand, although the server only collects $m$ workers' feedback in each round of global model update, the server leverages all $M$ workers' feedback by reusing historical information.
Intuitively, this translates the potential objection function drift originated from general worker information arrival process into the negative effect of delayed returns $\G(\x_{t-\tau_{t, i}})$ from workers.
It can be shown that such a negative effect vanishes asymptotically as the number of communication rounds $T$ gets large and in turn diminishes the convergence error.
This also explains the {\em stronger} linear speedup $\mathcal{O}(1/\sqrt{MT})$.
Specifically, even with partial ($m$) workers participation in each round, AFA-CS achieves a speedup with respect to total number of workers $M$ ($M>m$).
From the lower bound in FL (Proposition 6.1 in~\citet{gu2021fast}), Corollary~\ref{cor:silo} is tight. 
}
\end{rem}

\begin{rem}{\em
AFA-CS generalizes the lazy aggregation strategy in distributed learning (e.g., LAG~\citep{chen2018lag}) by setting $K=1$ (single local update), $\tau=0$ (synchronous setting) and $\sigma_L = 0$ (using full gradient descents instead of stochastic gradients) and further improve the rate of LSAG~\citep{chen2020lasg} from $\mathcal{O}(1/\sqrt{T})$ to $\mathcal{O}(1/\sqrt{MT})$.
We note that ~\citet{gu2021fast} and ~\citet{yan2020distributed} achieved $\mathcal{O}(\frac{1}{\sqrt{MKT}})$ and $\mathcal{O}(\frac{1}{\sqrt{M^{0.5}T}})$ for FL, respectively, by using historical information, which is similar to AFA-CS.
However, they both requires additional assumptions.
Specifically, ~\citet{gu2021fast} required a Lipschitz Hessian assumption and ~\citet{yan2020distributed} needed bounded stochastic gradient assumption.
By contrast, AFA-CS achieves the same optimal rate without such assumptions.
}
\end{rem}

\section{Numerical results} \label{sec:numerical}

In this section, we conduct experiments to verify our theoretical results.
We use i) logistic regression (LR) on manually partitioned non-i.i.d. MNIST dataset~\citep{mnist}, ii) convolutional neural network (CNN) for manually partitioned CIFAR-10~\citep{cifar10}, and iii) recurrent neural network (RNN) on natural non-i.i.d. dataset {\em Shakespeare}~\citep{mcmahan2016communication}.
In order to impose data heterogeneity in MNIST and CIFAR-10 data, we distribute the data evenly into each worker in label-based partition following the same process in the literature (e.g., \citet{mcmahan2016communication,yang2021achieving,li2019convergence}).
Therefore, we can use a parameter $p$ to represent the classes of labels in each worker's dataset, which signifies data heterogeneity: the smaller the $p$-value, the more heterogeneous the data across workers (cf. \citet{yang2021achieving,li2019convergence} for details).
Due to space limitation, we relegate the details of models, datasets and hyper-parameters, and further results of CNN and RNN to the appendix.

\begin{figure}[t!]
    \centering
    \begin{subfigure}[b]{0.45\columnwidth}
        \includegraphics[width=1\textwidth]{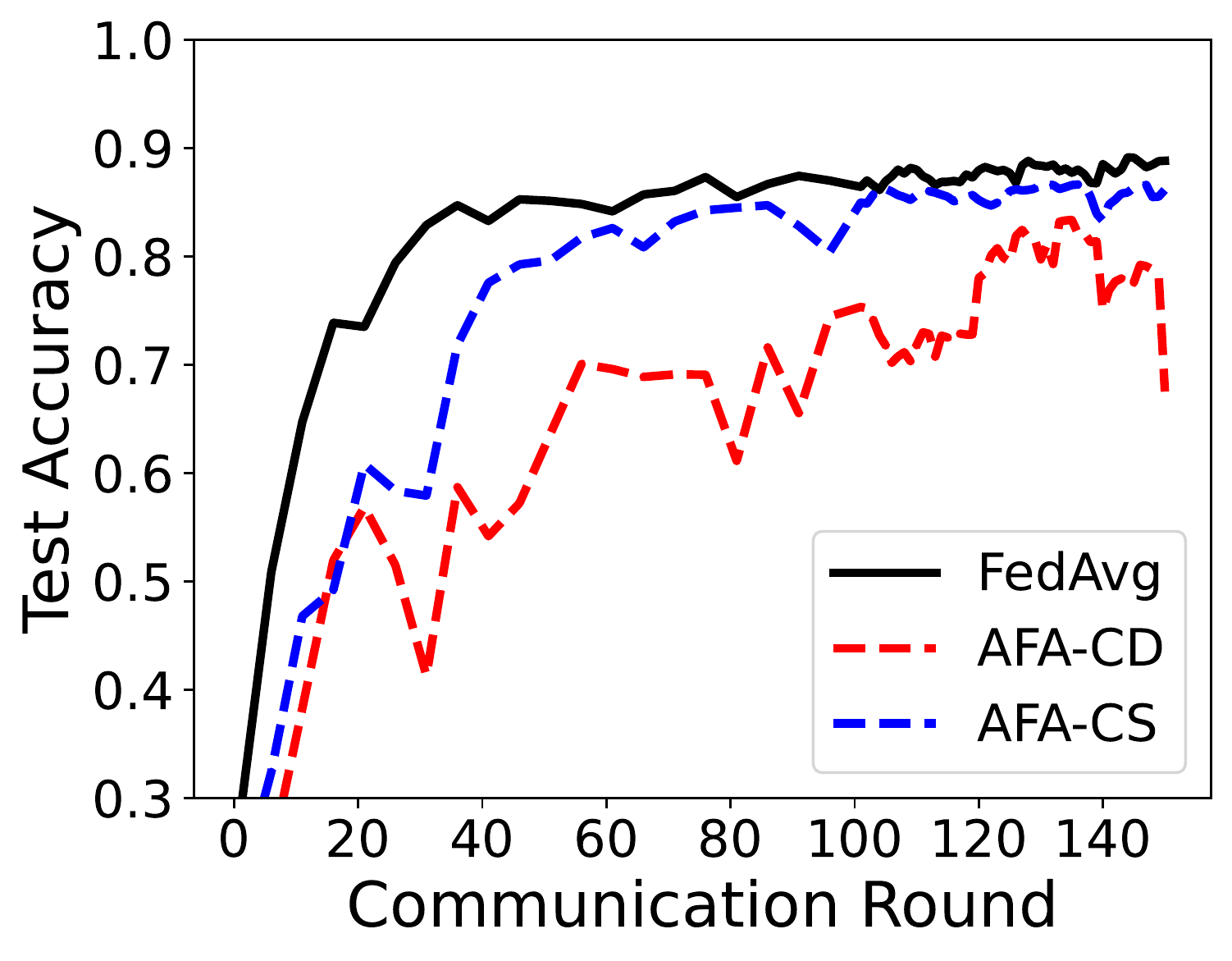}
        \caption{$p=1$.}
        \label{fedavg_afl_mnist_1}
    \end{subfigure}
    \begin{subfigure}[b]{0.45\columnwidth}
        \includegraphics[width=1\textwidth]{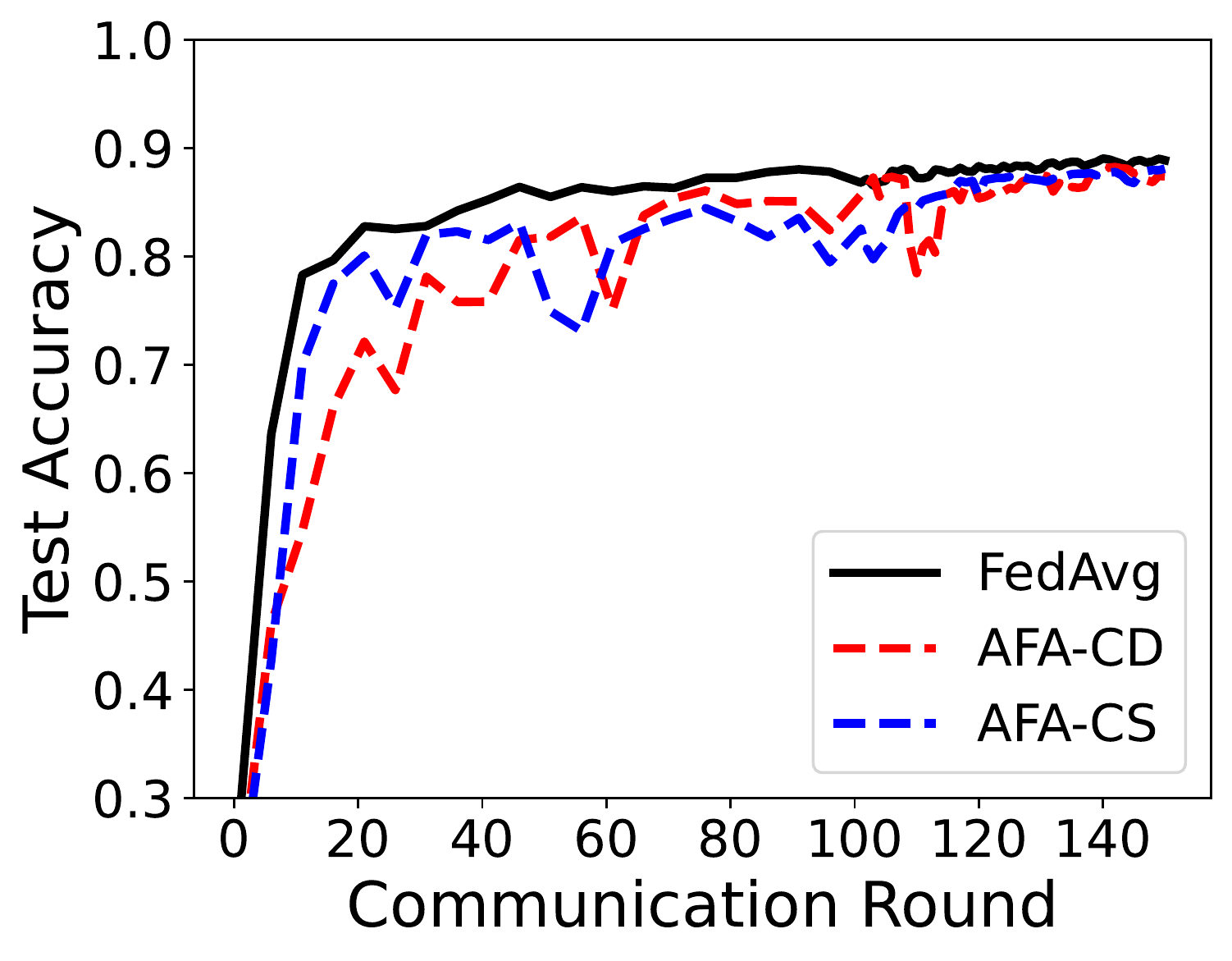}
        \caption{$p=2$.}
        \label{fedavg_afl_mnist_2}
    \end{subfigure}
    \begin{subfigure}[b]{0.45\columnwidth}
        \includegraphics[width=1\textwidth]{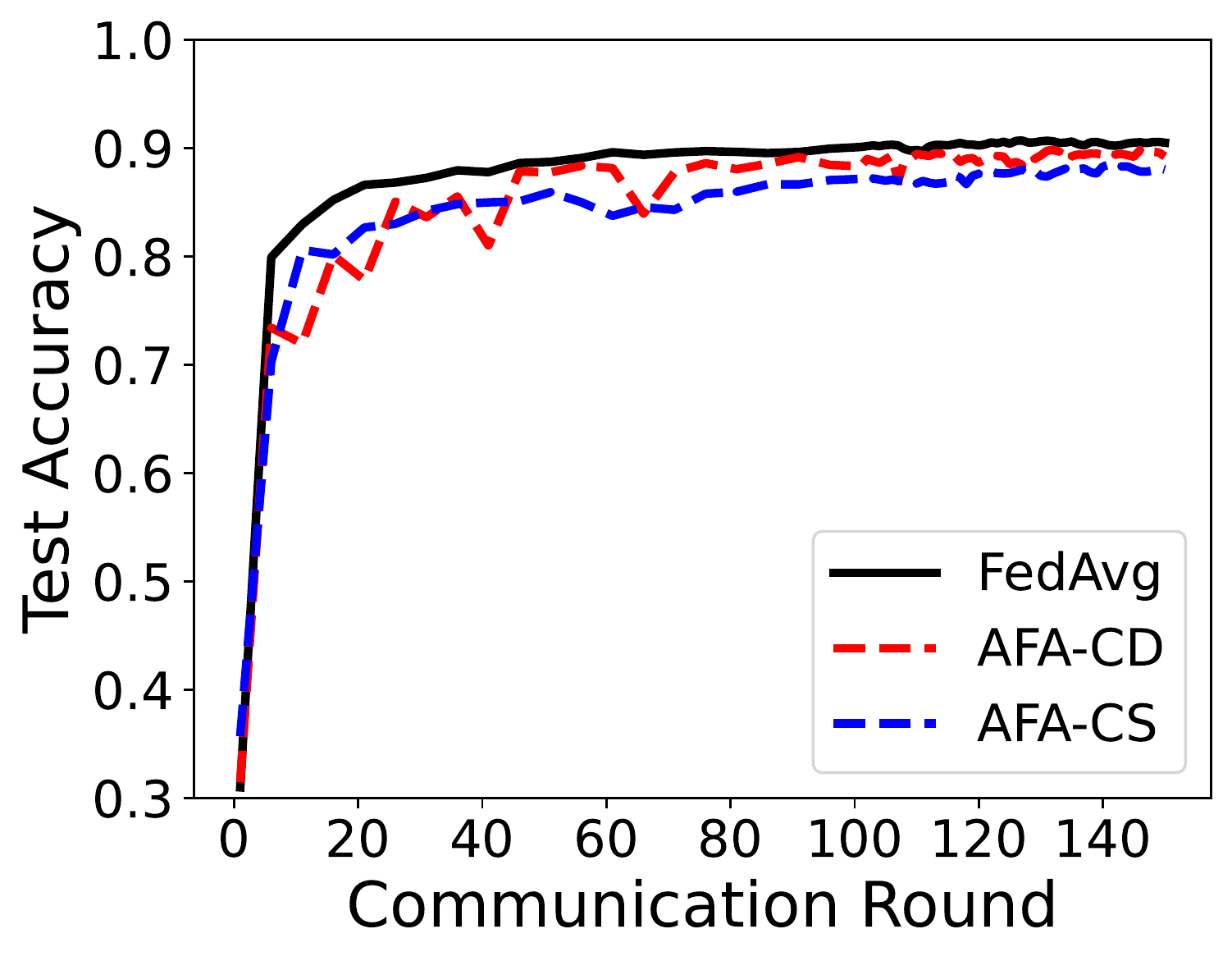}
        \caption{$p=5$.}
        \label{fedavg_afl_mnist_5}
    \end{subfigure}
    \begin{subfigure}[b]{0.45\columnwidth}
        \includegraphics[width=1\textwidth]{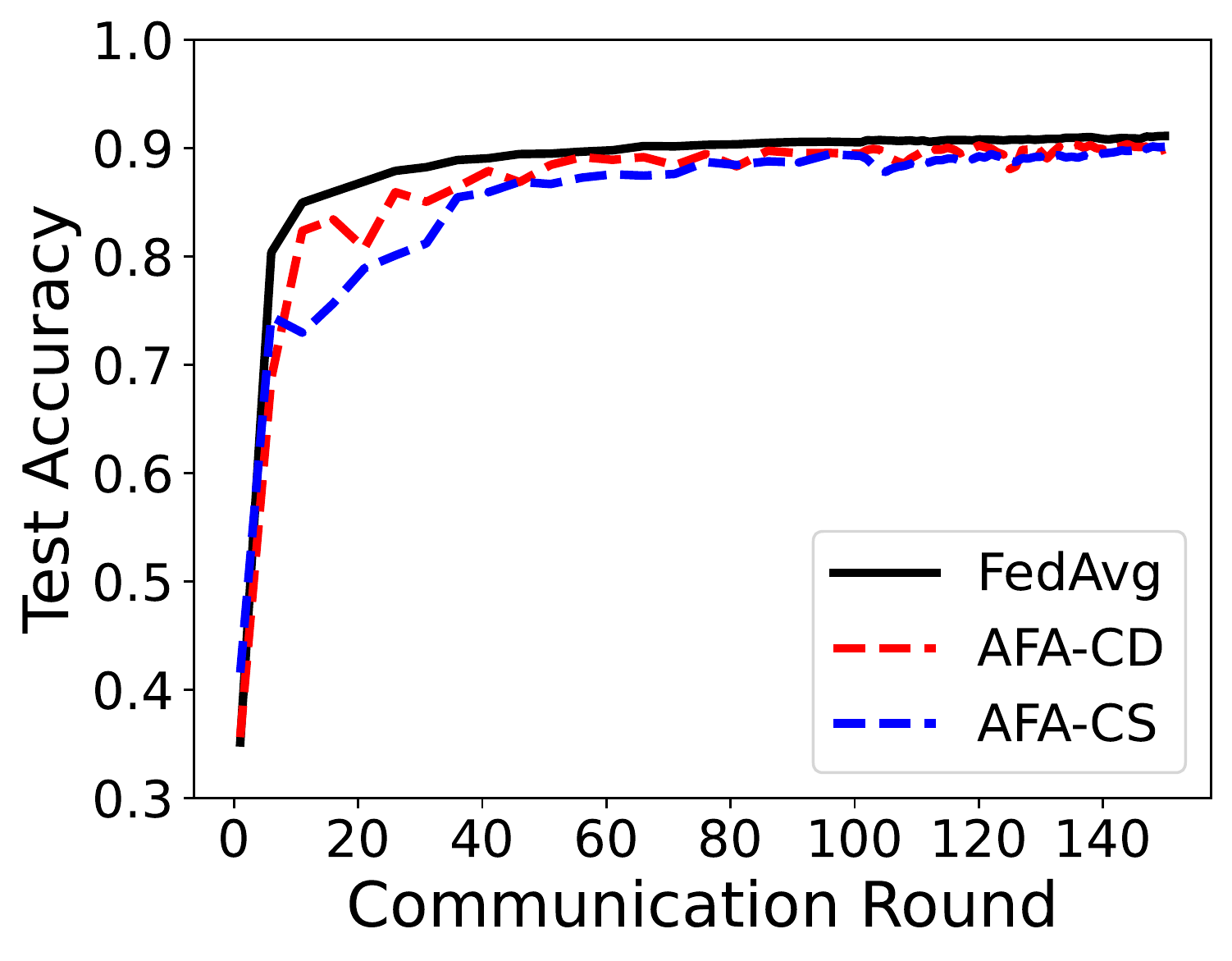}
        \caption{$p=10$.}
        \label{fedavg_afl_mnist_10}
    \end{subfigure} 
    \caption{Test accuracy for logistic regression on non-i.i.d. MNIST with different $p$-values.}
    \label{exp:lr_acc}
\end{figure}

In Figure~\ref{exp:lr_acc}, we illustrate the test accuracy for LR on MNIST with different $p$-values.
We use the classical FedAvg algorithm~\citep{mcmahan2016communication} for conventional FL with uniform worker sampling as a baseline, since it corresponds to the most ideal scenario where workers are fully cooperative with the server.
We examine the learning performance degradation of AFA algorithms (due to anarchic worker behaviors) compared to this ideal baseline.
For our AFA-CD and AFA-CS with general worker information arrival processes, the test accuracy is comparable to or nearly the same as that of FedAvg.
This confirms our theoretical results and validates the effectiveness of our AFA algorithms.
We further evaluate the impacts of various factors in AFL, including asynchrony, heterogeneous computing, worker's arrival process, and non-i.i.d. datasets, on convergence rate of our proposed AFA algorithms.
Note that AFL subsumes FedAvg and many variants when the above hyper-parameters are set as constant. 
Also, AFL coupled with other FL algorithms such as FedProx~\citep{li2018fedprox} and SCAFFOLD~\citep{karimireddy2020scaffold} is tested.
Our results show that the AFA algorithms are robust against all asynchrony and heterogeneity factors in AFL.
Due to space limitation, we refer readers to the appendix for all these experimental results.

\section{Conclusions} \label{sec:conclusion}
In this paper, we propose a new paradigm in FL called ``Anarchic Federated Learning'' (AFL).
In stark contrast to conventional FL models where the server and the worker are tightly coupled, AFL has a much lower server-worker coordination complexity, allowing a flexible worker participation.
We proposed two Anarchic Federated Averaging algorithms with two-sided learning rates for both cross-device and cross-silo settings, which are named AFA-CD and AFA-CS, respectively. 
We showed that both algorithms retain the highly desirable linear speedup effect in the new AFL paradigm.
Moreover, we showed that our AFL framework works well numerically by employing advance FL algorithms FedProx and SCAFFOLD as the optimizer in worker's side.






\section*{Acknowledgements}
This work has been supported in part by NSF grants CAREER CNS-2110259, CNS-2112471, CNS-2102233, CCF-2110252, and a Google Faculty Research Award.




\bibliography{BIB/FederatedLearning,BIB/AsynLearning,BIB/VarianceReduction}
\bibliographystyle{icml2022}

\newpage
\appendix
\onecolumn
\allowdisplaybreaks


\begin{center} 
	{\Large \textbf{Appendix}}
\end{center}

In this supplementary material, we provide the detailed proofs for all theoretical results in this paper.
Before presenting the proofs, we introduce some notations that will be used subsequently..
We assume there exists $M$ workers in total in the FL systems.
In each communication round, we assume a subset $\mathcal{M}_t$ of workers to be used, with $|\mathcal{M}_t| = m$.
We use $\G_i(\x_t)$ to represent the local update returned from worker $i, i \in [M]$ given global model parameter $\x_t^0 = \x_t$.
Also, we define $\G_i(\x_t) \triangleq \frac{1}{K_{t, i}} \sum_{j=0}^{K_{t,i}-1} \nabla f_i(\x^j_t, \xi_{t,i})$, where $\x_t^j$ represents the trajectory of the local model in the worker.
We use $\Delta_i$ to denote the average of the full gradients long the trajectory of local updates, i.e.,
$\Delta_i(\x_{t}) = \frac{1}{K_{t, i}} \sum_{j=0}^{K_{t,i}-1} \nabla f_i(\x^j_{t})$.
With the above notations, we are now in a position to present the proofs of the theoretical results in this paper.

\section{Proofs of Lemma~\ref{lem:1} and Lemma~\ref{lem:2}}
We start with proving two results stated in the following two lemmas, which will be useful in the rest of the proofs.
\begin{lemma} \label{lem:1}
$\mathbb{E}[\G_i(\x_{t})] = \Delta_i(\x_{t})$,
$\mathbb{E}[ \| \G_i(\x_{t}) - \Delta_i(\x_{t}) \|^2] \leq \frac{1}{K_{t,i}} \sigma_L^2, \forall i \in [M]$.
\end{lemma}

\begin{proof}
Taking the expectation of $\G_i(\x_t)$, we have:
\begin{align*}
\mathbb{E} \big[\G_i(\x_{t}) \big] &= \mathbb{E} \bigg[ \frac{1}{K_{t, i}} \sum_{j=0}^{K_{t,i}-1} \nabla f_i(\x^j_{t}, \xi_{t,i}^j) \bigg] \\
&= \frac{1}{K_{t, i}} \sum_{j=0}^{K_{t,i}-1} \mathbb{E} [ \nabla f_i(\x^j_{t}, \xi_{t,i}^j)] \\
&= \Delta_i(\x_{t}). 
\end{align*}
Also, by computing the mean square error between $\G_i(\x_t)$ and $\Delta_i(\x_{t})$, we have:
\begin{align*}
\mathbb{E}[ \| \G_i(\x_{t}) - \Delta_i(\x_{t}) \|^2] 
&= \mathbb{E} \big[ \| \frac{1}{K_{t, i}} \sum_{j=0}^{K_{t,i}-1} \nabla f_i(\x^j_{t}, \xi_{t,i}) - \sum_{j=0}^{K_{t,i}-1} \nabla f_i(\x^j_{t}) \|^2 \big] \\
&= \frac{1}{K_{t, i}^2} \mathbb{E} \big[ \| \sum_{j=0}^{K_{t,i}-1} \nabla f_i(\x^j_{t}, \xi_{t,i}) - \sum_{j=0}^{K_{t,i}-1} \nabla f_i(\x^j_{t}) \|^2 \big] \\
&\leq \frac{1}{K_{t,i}} \sigma_L^2.
\end{align*}
Note $\{ \nabla f_i(\x^j_{t}, \xi_{t,i}) - \nabla f_i(\x^j_{t}) \}$ forms a martingale difference sequence.
This completes the proof of Lemma~\ref{lem:1}.
\end{proof}

\begin{lemma} \label{lem:2}
For a fixed set $\mathcal{M}_t$ with cardinality $m$, $\mathbb{E} \bigg[ \big\| \sum_{i \in \mathcal{M}_t} \G_i(\x_{t-\tau_{t,i}}) \big\|^2 \bigg] \leq 2 \big\| \sum_{i \in \mathcal{M}_t} \Delta_i(\x_{t-\tau_{t,i}}) \big\|^2 + \frac{2m}{K_t} \sigma_L^2,$ where $\frac{1}{K_t} = \frac{1}{m} \sum_{i \in \mathcal{M}_t} \frac{1}{K_{t,i}}$. 
\end{lemma}

\begin{proof}
By adding and subtracting $\Delta_i(\x_{t-\tau_{t,i}})$, we have:
\begin{align*}
\mathbb{E} \bigg[ \big\| \sum_{i \in \mathcal{M}_t} \G_i(\x_{t-\tau_{t,i}}) \big\|^2 \bigg] 
&= 2\mathbb{E} \bigg[ \big\| \sum_{i \in \mathcal{M}_t} \G_i(\x_{t-\tau_{t,i}}) - \Delta_i(\x_{t-\tau_{t,i}}) \big\|^2 \bigg]  + 2\big\| \sum_{i \in \mathcal{M}_t} \Delta_i(\x_{t-\tau_{t,i}}) \big\|^2 \\
&\leq 2\sum_{i \in \mathcal{M}_t} \frac{1}{K_{t,i}} \sigma_L^2 + 2\big\| \sum_{i \in \mathcal{M}_t} \Delta_i(\x_{t-\tau_{t,i}}) \big\|^2 \\
&\leq \frac{2m}{K_t} \sigma_L^2 + 2\big\| \sum_{i \in \mathcal{M}_t} \Delta_i(\x_{t-\tau_{t,i}}) \big\|^2.
\end{align*}
Here the updates among clients $\{ \G_i(\x_{t-\tau_{t,i}}) - \Delta_i(\x_{t-\tau_{t,i}}) \}$ are assumed to be independent.
\end{proof}

\section{Proof of the performance of the AFA-CD algorithm} \label{appdx_device}

In this section, we provide the proofs of the theoretical results of the AFA-CD algorithm.
We consider two cases: i) general worker information arrival processes and ii) uniformly distributed worker information arrivals.
As mentioned earlier, for general worker information arrival processes, we do not make any assumptions on the worker information arrival processes except the independence of workers' participation.
For uniformly distributed worker information arrivals, $\mathcal{M}_t$ can be viewed as a subset with size m independently and uniformly sampled from $[M]$ without replacement. 
The similar convergence analysis for independently and uniformly sampling with replacement can be derived in the same way following the techniques in ~\citep{yang2021achieving,li2019convergence}.

\subsection{Lower Bound for General Worker Information Arrival Processes}\label{appdx_lb}
\lb*

\begin{proof}
	We prove the lower bound by considering a worst-case scenario for simple one-dimensional functions.
	Let the FL system has two workers with the following loss functions: $f_1(\x) = (\x + G)^2, f_2(\x) = (\x - G)^2, f(\x) = \frac{1}{2}(f_1(\x) + f_2(\x)) = \x^2 + G^2.$
	It is easy to verify that $\| \nabla f_1(\x) - \nabla f(\x) \|^2 \leq 4 G^2 = \sigma_G^2$ and $\| \nabla f_2(\x) - \nabla f(\x) \|^2 \leq 4 G^2 = \sigma_G^2$, where $\sigma_G$ is the heterogeneity index.
	We consider a special case for the general worker arrival process when only the first one worker participates in the training as the worst-case scenario, equivalent to optimizing $f_1(\x)$ rather than $f(\x)$.
	In such case, any convergent (and potentially random) algorithm would return $\mathbb{E} \hat{\x} = -G$.
	As a result, $\mathbb{E} \| \nabla f(\hat{\x}) \|^2 = \Omega (\sigma_G^2).$
	\end{proof}

\subsection{General Worker Information Arrival Processes}\label{appdx_arbipdx_uniform}

\arbitrary*

\begin{proof}
Due to the $L$-smoothness assumption, taking expectation of $f(\x_{t+1})$ over the randomness in communication round $t$, we have:
\begin{align*}
	\mathbb{E} [f(\x_{t+1})] &\leq f(\x_t) + \underbrace{ \big< \nabla f(\x_t), \mathbb{E} [\x_{t+1} - \x_t] \big> }_{A_1} + \frac{L}{2} \underbrace{ \mathbb{E} [\| \x_{t+1} - \x_t \|^2 }_{A_2}.
\end{align*}

First, we bound the term $A_2$ as follows:
\begin{align*}
A_2 &= \mathbb{E} \| \x_{t+1} - \x_t \|^2 \\
&= \eta^2 \eta_L^2 \mathbb{E} \left\| \frac{1}{m} \sum_{i \in \mathcal{M}_t}G_i(\x_{t-\tau_{t,i}}) \right\|^2 \\
&\overset{(a1)}{\leq} \frac{2 \eta^2 \eta_L^2}{m^2} \left\| \sum_{i \in \mathcal{M}_t} \Delta_i(\x_{t-\tau_{t,i}}) \right\|^2 + \frac{2 \eta^2 \eta_L^2}{m K_t} \sigma_L^2,
\end{align*}
where $(a1)$ is due to Lemma~\ref{lem:2}.
Next, we bound the term $A_1$ as follows:
\begin{align*}
&A_1 = \big< \nabla f(\x_t), \mathbb{E} [\x_{t+1} - \x_t] \big> \\
&= - \eta \eta_L \big< \nabla f(\x_t), \mathbb{E} \left[\frac{1}{m} \sum_{i \in \mathcal{M}_t} \G_i(\x_{t - \tau_{t,i}}) \right] \big> \\
&\overset{(a2)}{=}  -\frac{1}{2} \eta \eta_L \| \nabla f(\x_t) \|^2 - \frac{1}{2} \eta \eta_L \mathbb{E} \left\| \frac{1}{m} \sum_{i \in \mathcal{M}_t} \Delta_i(\x_{t - \tau_{t,i}}) \right\|^2 + \frac{1}{2} \eta \eta_L \underbrace{ \mathbb{E} \| \nabla f(\x_t) - \frac{1}{m} \sum_{i \in \mathcal{M}_t} \Delta_i(\x_{t - \tau_{t,i}}) \|^2}_{A_3},
\end{align*}
where $(a2)$ is due to Lemma~\ref{lem:1} and the fact that $\langle \x, \y \rangle = \frac{1}{2} (\| \x \|^2 + \| \y \|^2 - \| \x - \y \|^2)$ and Lemma~\ref{lem:1}.
To further bound the term $A_3$, we have:
\begin{align*}
&A_3 =  \mathbb{E} \| \nabla f(\x_t) - \frac{1}{m} \sum_{i \in \mathcal{M}_t} \Delta_i(\x_{t - \tau_{t,i}}) \|^2 \\
&\leq \frac{1}{m} \sum_{i \in \mathcal{M}_t} \mathbb{E} \| \nabla f(\x_t) -  \Delta_i(\x_{t - \tau_{t,i}}) \|^2 \\
&= \frac{1}{m} \sum_{i \in \mathcal{M}_t} \mathbb{E} \| \nabla f(\x_t) - \nabla f(\x_{t - \tau_{t,i}}) + \nabla f(\x_{t - \tau_{t,i}}) - \nabla f_i(\x_{t - \tau_{t,i}}) + \nabla f_i(\x_{t - \tau_{t,i}}) -  \Delta_i(\x_{t - \tau_{t,i}}) \|^2 \\
&\overset{(a3)}{\leq} \frac{1}{m} \sum_{i \in \mathcal{M}_t} \bigg[ 3 \mathbb{E} \| \nabla f(\x_t) - \nabla f(\x_{t - \tau_{t,i}}) \|^2 + 3 \mathbb{E} \| \nabla f(\x_{t - \tau_{t,i}}) - \nabla f_i(\x_{t - \tau_{t,i}}) \|^2 \\
& \quad \quad + 3 \mathbb{E} \| \nabla f_i(\x_{t - \tau_{t,i}}) -  \Delta_i(\x_{t - \tau_{t,i}}) \|^2 \bigg] \\
&\overset{(a4)}{\leq} \underbrace{\frac{3 L^2}{m} \sum_{i \in \mathcal{M}_t} \mathbb{E} \| \x_t - \x_{t - \tau_{t,i}} \|^2 }_{A_4} 
+ 3 \sigma_G^2
+ \frac{3}{m} \sum_{i \in \mathcal{M}_t} \underbrace{ \mathbb{E} \| \nabla f_i(\x_{t - \tau_{t,i}}) -  \Delta_i(\x_{t - \tau_{t,i}}) \|^2 }_{A_5},
\end{align*}

where $(a3)$ followings from the inequality $\| \x_1 + \x_2 + \cdots + \x_n \|^2 \leq n \sum_{i=1}^{n} \| \x_i \|^2$, and $(a4)$ is due to the L-smoothness assumption (Assumption~1) and bounded global variance assumption (Assumption~3).

To further bound the term $A_4$, we have:
\begin{align*}
A_4 &= \frac{1}{m} \sum_{i \in [m]} \mathbb{E} \| \x_t - \x_{t - \tau_{t,i}} \|^2 \\
&\overset{(a5)}{\leq} \mathbb{E} \| \x_t - \x_{t - \tau_{t,u}} \|^2 \\
&= \mathbb{E} \left\| \sum_{k=t - \tau_{t, u}}^{t-1} \x_{k+1} - \x_k \right\|^2 \\
&= \mathbb{E} \left\| \sum_{k=t - \tau_{t, u}}^{t-1} \frac{1}{m} \eta \eta_L \sum_{i \in \mathcal{M}_k} \G_i(\x_{k - \tau_{k,i}}) \right\|^2 \\
&= \mathbb{E} \left[ \frac{\eta^2 \eta_L^2}{m^2} \left\| \sum_{k=t - \tau_{t, u}}^{t-1} \sum_{i \in \mathcal{M}_k} \G_i(\x_{k - \tau_{k,i}}) \right\|^2 \right] \\
&\overset{(a6)}{\leq} \mathbb{E} \left[ \frac{\eta^2 \eta_L^2}{m^2} \tau \sum_{k=t - \tau_{t, u}}^{t-1} \left\| \sum_{i \in \mathcal{M}_k} G_i(\x_{k - \tau_{k,i}}) \right\|^2 \right] \\
&\overset{(a7)}{\leq} \left[ \frac{\eta^2 \eta_L^2 \tau }{m^2} \sum_{k=t - \tau_{t, u}}^{t-1} \left(2 \mathbb{E} \left\| \sum_{i \in \mathcal{M}_k} \Delta_i(\x_{k - \tau_{k,i}}) \right\|^2 +  \frac{2m}{K_k} \sigma_L^2 \right) \right].
\end{align*}
In the derivations above, we let $u := \mathop{\mathrm{argmax}}_{i \in [M]} \| \x_t - \x_{t - \tau_{t,i}} \|^2$, which yields $(a5)$.
Note also that the maximum delay assumption $\tau \geq \tau_{k,i}, \forall i \in [M]$ implies $(a6)$.
Lastly, $(a7)$ follows from Lemma~\ref{lem:2}.

To further bound the term $A_5$, we have:
\begin{align*}
A_5 &= \mathbb{E} \| \nabla f_i(\x_{t - \tau_{t,i}}) -  \Delta_i(\x_{t - \tau_{t,i}}) \|^2 \\
&= \mathbb{E} \left\| \nabla f_i(\x_{t - \tau_{t,i}}) - \frac{1}{K_{t,i}} \sum_{j=0}^{K_{t,i}-1} \nabla f_i(\x^j_{t - \tau_{t,i}}) \right\|^2 \\
&= \frac{1}{K_{t,i}} \sum_{j=0}^{K_{t,i}-1} \mathbb{E} \left\| \nabla f_i(\x_{t - \tau_{t,i}}) - \nabla f_i(\x^j_{t - \tau_{t,i}}) \right\|^2 \\
&\overset{(a8)}{\leq} \frac{L^2}{K_{t,i}} \sum_{j=0}^{K_{t,i}-1} \underbrace{ \mathbb{E} \| \x_{t - \tau_{t,i}} - \x^j_{t - \tau_{t,i}} \|^2}_{A_6} \\
&\overset{(a9)}{\leq} 5 K_{t,i} L^2 \eta_L^2 (\sigma_L^2 + 6 K_{t,i} \sigma_G^2) + 30 K_{t,i}^2 L^2 \eta_L^2 \| \nabla f(\x_{t - \tau_{t,i}}) \|^2,
\end{align*}
where $(a8)$ is due to the $L$-smoothness assumption (Assumption~1),
and $(a9)$ follows from the bound of $A_6$ shown below.
Here, we denote maximum number of local steps of all workers as $K$, i.e., $K_{t, i} \leq K, \forall t, i$.
This definition of $K$ implies $(a10)$.

Now, it remains to bound term $A_6$ in the derivations above.
Note that the bounding proof of $A_6$ in what follows is the same as Lemma 4 in \citep{reddi2020adaptive}. we restate the proof here in order for this paper to be self-contained.
For any worker $i$ in the $k$-th local step, we have the following results for the norm of parameter changes for one local computation:
\begin{align*}
A_6 &=\mathbb{E} [\|\x_{t, k}^i - \x_t \|^2] = \mathbb{E} [\|\x_{t, k-1}^i - \x_t -\eta_L g_{t, k-1}^i \|^2] \\
&\leq \mathbb{E} [\|\x_{t, k-1}^i - \x_t -\eta_L (g_{t, k-1}^i - \nabla f_i(\x_{t, k-1}^i) + \nabla f_i(\x_{t, k-1}^i) - \nabla f_i(\x_{t}) \\
& \quad + \nabla f_i(\x_{t}) - \nabla f(\x_t) + \nabla f(\x_t)) \|^2 ] \\
&\leq (1 + \frac{1}{2K_{t,i}-1}) \mathbb{E} [\| \x_{t, k-1}^i - \x_t \|^2] + \mathbb{E} [\| \eta_L (g_{t, k-1}^i - \nabla f_i(\x_{t, k-1}^i)) \|^2 ] \\
& \quad + 6K_{t,i} \mathbb{E} [\| \eta_L (\nabla f_i(\x_{t, k-1}^i) - \nabla f_i(\x_{t})) \|^2] + 6K_{t,i} \mathbb{E} [\| \eta_L (\nabla f_i(\x_{t}) - \nabla f(\x_t))) \|^2] \\
& \quad + 6K_{t,i} \| \eta_L \nabla f(\x_t)) \|^2 \\
&\leq (1 + \frac{1}{2K_{t,i}-1}) \mathbb{E} [\| \x_{t, k-1}^i - \x_t \|^2] + \eta_L^2 \sigma_L^2 + 6K_{t,i} \eta_L^2 L^2 \mathbb{E} [\| \x_{t, k-1}^i -\x_{t} \|^2] \\
&\quad + 6K_{t,i} \eta_L^2 \sigma_G^2 + 6K_{t,i} \| \eta_L \nabla f(\x_t)) \|^2 \\
&= (1 + \frac{1}{2K_{t,i}-1} + 6K_{t,i} \eta_L^2 L^2) \mathbb{E} [\| \x_{t, k-1}^i - \x_t \|^2] + \eta_L^2 \sigma_L^2 + 6K_{t,i} \eta_L^2 \sigma_G^2 + 6K_{t,i} \| \eta_L \nabla f(\x_t)) \|^2 \\
&\overset{(a11)}{\leq} (1 + \frac{1}{K_{t,i}-1}) \mathbb{E} [\| \x_{t, k-1}^i - \x_t \|^2] + \eta_L^2 \sigma_L^2 + 6K_{t,i} \eta_L^2 \sigma_G^2 + 6K_{t,i} \| \eta_L \nabla f(\x_t)) \|^2,
\end{align*}
where $(a11)$ follows from the fact that $\frac{1}{2K_{t,i}-1} + 6K_{t,i} \eta_L^2 L^2 \leq \frac{1}{K_{t,i}-1}$ if $\eta_L^2 \leq \frac{1}{6(2K_{t,i}^2 - 3K_{t,i} + 1)L^2}$.

Unrolling the recursion, we obtain:
\begin{align}
\mathbb{E} [\|\x_{t, k}^i - \x_t \|^2] &\leq \sum_{p=0}^{k-1}(1+\frac{1}{K_{t,i}-1})^p [\eta_L^2 \sigma_L^2+6K_{t,i}\sigma_G^2 + 6K_{t,i}\eta_L^2 \| \eta_L \nabla f(\x_t)) \|^2] \nonumber \\
&\leq (K_{t,i}-1) [(1 + \frac{1}{K_{t,i}-1})^K_{t,i} - 1] [\eta_L^2 \sigma_L^2+6K_{t,i} \eta_L^2 \sigma_G^2 + 6K_{t,i} \| \eta_L \nabla f(\x_t)) \|^2] \nonumber \\
&\leq  5 K_{t,i} \eta_L^2 (\sigma_L^2 + 6 K_{t,i} \sigma_G^2) + 30 K_{t,i}^2 \eta_L^2 \| \nabla f(\x_t) \|^2. \label{appdx:a6}
\end{align}




With the above results of the terms $A_1$ through $A_5$, we have:
\begin{align*}
&\mathbb{E} [f(\x_{t+1})] - f(\x_t) \leq \underbrace{ \big< \nabla f(\x_t), \mathbb{E} [\x_{t+1} - \x_t] \big> }_{A_1} + \frac{L}{2} \underbrace{ \mathbb{E} [\| \x_{t+1} - \x_t \|^2 }_{A_2} \\
&= -\frac{1}{2} \eta \eta_L \| \nabla f(\x_t) \|^2 -\frac{1}{2} \eta \eta_L \mathbb{E} \left\| \frac{1}{m} \sum_{i \in \mathcal{M}_t} \Delta_i(\x_{t - \tau_{t,i}}) \right\|^2 + \frac{1}{2} \eta \eta_L \underbrace{ \mathbb{E} \left\| \nabla f(\x_t) - \frac{1}{m} \sum_{i \in \mathcal{M}_t} \Delta_i(\x_{t - \tau_{t,i}}) \right\|^2}_{A_3} \\
& \quad + \frac{L \eta^2 \eta_L^2}{m^2} \mathbb{E} \left\| \sum_{i \in \mathcal{M}_t} \Delta_i(\x_{t-\tau_{t,i}}) \right\|^2 + \frac{L \eta^2 \eta_L^2}{m} \sigma_L^2 \\
&\leq -\frac{1}{2} \eta \eta_L \| \nabla f(\x_t) \|^2 - \frac{1}{2} \eta \eta_L \mathbb{E} \left\| \frac{1}{m} \sum_{i \in \mathcal{M}_t} \Delta_i(\x_{t - \tau_{t,i}}) \right\|^2 + \frac{L \eta^2 \eta_L^2}{m^2} \mathbb{E} \left\| \sum_{i \in \mathcal{M}_t} \Delta_i(\x_{t-\tau_{t,i}}) \right\|^2 + \frac{L \eta^2 \eta_L^2}{m K_t} \sigma_L^2 \\
& \quad + \frac{3}{2} \eta \eta_L \sigma_G^2
+ \frac{3 L^2}{2} \eta \eta_L \bigg[ \underbrace{\frac{1}{m} \sum_{i \in \mathcal{M}_t} \mathbb{E} \| \x_t - \x_{t - \tau_{t,i}} \|^2 }_{A_4} \bigg] + \frac{3\eta \eta_L}{2m} \sum_{i \in \mathcal{M}_t} \underbrace{ \mathbb{E} \| \nabla f_i(\x_{t - \tau_{t,i}}) -  \Delta_i(\x_{t - \tau_{t,i}}) \|^2 }_{A_5} \\
&\leq -\frac{1}{2} \eta \eta_L \| \nabla f(\x_t) \|^2 - \frac{\eta \eta_L}{2m^2} \mathbb{E} \left\| \sum_{i \in \mathcal{M}_t} \Delta_i(\x_{t - \tau_{t,i}}) \right\|^2 + \frac{L \eta^2 \eta_L^2}{m^2} \mathbb{E} \left\| \sum_{i \in \mathcal{M}_t} \Delta_i(\x_{t-\tau_{t,i}}) \right\|^2 + \frac{L \eta^2 \eta_L^2}{m K_t} \sigma_L^2 \\
& \quad + \frac{3}{2} \eta \eta_L \sigma_G^2
+ \frac{3 L^2}{2} \eta \eta_L \left[ \frac{\eta^2 \eta_L^2 \tau }{m^2} \sum_{k=t - \tau_{t, u}}^{t-1} \left( 2 \mathbb{E} \left\| \sum_{i \in \mathcal{M}_k} \Delta_i(\x_{k - \tau_{k,i}}) \right\|^2 +  \frac{2m}{K_k} \sigma_L^2 \right) \right] \\
& \quad + \frac{3\eta \eta_L}{2} \frac{1}{m} \sum_{i \in \mathcal{M}_t} \left[5 K_{t,i} L^2 \eta_L^2 (\sigma_L^2 + 6 K_{t,i} \sigma_G^2) + 30 L^2 \eta_L^2 K_{t,i}^2 \| \nabla f(\x_{t - \tau_{t,i}}) \|^2 \right] \\
&\leq - \frac{1}{2} \eta \eta_L \| \nabla f(\x_t) \|^2 + 45 \eta \eta_L^3 L^2 \frac{1}{m} \sum_{i=1}^m K_{t,i}^2 \| \nabla f(\x_{t - \tau_{t,i}}) \|^2 \\
& \quad + \bigg[ - \frac{\eta \eta_L}{2m^2} + \frac{L \eta^2 \eta_L^2}{m^2} \bigg] \mathbb{E} \left\| \sum_{i \in \mathcal{M}_t} \Delta_i(\x_{t - \tau_{t,i}}) \right\|^2 
+ \frac{3 \tau \eta^3 \eta_L^3}{m^2} \sum_{k=t - \tau_{t, u}}^{t-1} \mathbb{E} \left\|  \sum_{i \in \mathcal{M}_t} \Delta_i(\x_{k - \tau_{k,i}}) \right\|^2 \\
& \quad + \bigg[ \frac{L \eta^2 \eta_L^2}{m K_t} + \frac{3 \tau L^2 \eta^3 \eta_L^3 \sum_{k=t - \tau_{t, u}}^{t-1} \frac{1}{K_k}}{m} + \frac{15 \eta \eta_L^3 L^2 \frac{1}{m} \sum_{i \in \mathcal{M}_t} K_{t,i} }{2} \bigg] \sigma_L^2 
+ \left[\frac{3}{2} \eta \eta_L + 45 L^2 \eta \eta_L^3 \frac{1}{m} \sum_{i \in \mathcal{M}_t} K_{t,i}^2 \right] \sigma_G^2.
\end{align*}

Summing the above inequality from $t=0$ to $t= T-1$ yields:
\begin{align*}
&\mathbb{E} f(\x_{T}) - f(\x_0) \\
&\leq \sum_{t=0}^{T-1} \bigg[ -\frac{1}{2} \eta \eta_L \| \nabla f(\x_t) \|^2 + 45 \eta \eta_L^3 L^2 \frac{1}{m} \sum_{i=1}^m K_{t,i}^2 \mathbb{E} \left\| \nabla f(\x_{t - \tau_{t,i}}) \right\|^2 \bigg] \\
& \quad + \sum_{t=0}^{T-1} \bigg[ \left[-\frac{\eta \eta_L}{2m^2} + \frac{L \eta^2 \eta_L^2}{m^2} \right] \mathbb{E} \left\| \sum_{i \in \mathcal{M}_t} \Delta_i(\x_{t - \tau_{t,i}}) \right\|^2 
+ \frac{3 \tau L^2 \eta^3 \eta_L^3}{m^2} \sum_{k=t - \tau_{t, u}}^{t-1} \mathbb{E} \left\| \sum_{i \in \mathcal{M}_k} \Delta_i(\x_{k - \tau_{k,i}}) \right\|^2 \bigg] \\
& \quad + \sum_{t=0}^{T-1} \bigg[ \frac{L \eta^2 \eta_L^2}{m K_t} + \frac{3 \tau L^2 \eta^3 \eta_L^3 \sum_{k=t - \tau_{t, u}}^{t-1} \frac{1}{K_k}}{m} + \frac{15 \eta \eta_L^3 L^2 \frac{1}{m} \sum_{i \in \mathcal{M}_t} K_{t,i}}{2} \bigg] \sigma_L^2 \\
&\quad + \sum_{t=0}^{T-1} \bigg[\frac{3}{2} \eta \eta_L + 45 L^2 \eta \eta_L^3 \frac{1}{m} \sum_{i \in \mathcal{M}_t} K_{t,i}^2  \bigg] \sigma_G^2 \\
&\overset{(a12)}{\leq} \sum_{t=0}^{T-1} \bigg[ -\frac{1}{2} \eta \eta_L + 45 \eta \eta_L^3 K_{t, max}^2 L^2 \tau \bigg] \| \nabla f(\x_t) \|^2 \\
& \quad + \sum_{t=0}^{T-1} \bigg[ -\frac{\eta \eta_L}{2m^2} + \frac{L \eta^2 \eta_L^2}{m^2} + \frac{3 \tau^2 L^2 \eta^3 \eta_L^3}{2m^2} \bigg] \mathbb{E} \left\| \sum_{i \in \mathcal{M}_t} \Delta_i(\x_{t - \tau_{t,i}}) \right\|^2 \\
& \quad + \sum_{t=0}^{T-1} \bigg[ \frac{L \eta^2 \eta_L^2}{m K_t} + \frac{3 \tau^2 L^2 \eta^3 \eta_L^3 \frac{1}{K_t}}{m} + \frac{15 \eta \eta_L^3 \bar{K}_{t} L^2}{2} \bigg] \sigma_L^2 
+ \sum_{t=0}^{T-1} \bigg[\frac{3}{2} \eta \eta_L + 45 \hat{K}_t^2 L^2 \eta \eta_L^3 \bigg] \sigma_G^2 \\
&\overset{(a13)}{\leq} \sum_{t=0}^{T-1} - \frac{1}{4} \eta \eta_L \| \nabla f(\x_t) \|^2 \\
& \quad + \eta \eta_L \left[ \frac{L \eta \eta_L}{m} \sum_{t=0}^{T-1} \frac{1}{K_t} + \frac{3 \tau^2 L^2 \eta^2 \eta_L^2}{m} \sum_{t=0}^{T-1} \frac{1}{K_t} + \frac{15 \eta_L^2 L^2}{2} \sum_{t=0}^{T-1} \bar{K}_t \right] \sigma_L^2 
+ \sum_{t=0}^{T-1} \eta \eta_L \left[\frac{3}{2} + 45 \hat{K}_t^2 L^2 \eta_L^2 \right] \sigma_G^2 \\
&\overset{(a14)}{=} \sum_{t=0}^{T-1} - \frac{1}{4} \eta \eta_L \| \nabla f(\x_t) \|^2
+ T \eta \eta_L \big[ \alpha_L \sigma^2_L + \alpha_G \sigma_G^2 \big],
\end{align*}
where $(a12)$ is due to maximum time delay $\tau$ in the system,
$(a13)$ holds if
$\frac{1}{4} \leq [ \frac{1}{2} - 45 \eta_L^2 K_{t, max}^2 L^2 \tau ]$, i.e., 
$180 \eta_L^2 K_{t, max}^2 L^2 \tau < 1$, and
$\bigg[ -\frac{\eta \eta_L}{2m^2} + \frac{L \eta^2 \eta_L^2}{m^2} + \frac{3 L^2 \tau^2 \eta^3 \eta_L^3}{m^2} \bigg] \leq 0$, i.e.,  $2L \eta \eta_L + 6 \tau^2 L^2 \eta^2 \eta_L^2 \leq 1$.
Note $\bar{K}_t = \frac{1}{m} \sum_{i \in \mathcal{M}_t} K_{t, i}$, $\hat{K}_t^2 = \frac{1}{m} \sum_{i \in \mathcal{M}_t} K_{t, i}^2$, and $K_{t, max} = \max \{ K_{t, i}, i \in [m] \}$.
Lastly, $(a14)$ follows from the following definitions:
$$\alpha_L = \left[ \frac{L \eta \eta_L}{m} \frac{1}{T} \sum_{t=0}^{T-1} \frac{1}{K_t} + \frac{3 \tau^2 L^2 \eta^2 \eta_L^2}{m} \frac{1}{T} \sum_{t=0}^{T-1} \frac{1}{K_t} + \frac{15 \eta_L^2 L^2}{2} \frac{1}{T} \sum_{t=0}^{T-1} \bar{K}_t \right],$$
$$\alpha_G = \left[\frac{3}{2} + 45 L^2 \eta_L^2 \frac{1}{T} \sum_{t=0}^{T-1} \hat{K}_t^2 \right].$$
Rearranging terms, we have:
\begin{align*}
\frac{1}{T} \sum_{t=0}^{T-1} \| \nabla f(\x_t) \|^2 \leq \frac{4(f_0 - f_*)}{ \eta \eta_L T} + 4 \big[ \alpha_L \sigma^2_L + \alpha_G \sigma_G^2 \big],
\end{align*}
and the proof is complete.
\end{proof}


\arbitraryC*

\begin{proof}
Suppose a constant local step $K$ for each worker, and let $\eta_L = \frac{1}{\sqrt{T}}$, and $\eta = \sqrt{mK}$.
It then follows that:
$$\alpha_L = \mathcal{O}(\frac{1}{m^{1/2} K^{1/2} T^{1/2}}) + \mathcal{O}(\frac{\tau^2}{T}) + \mathcal{O}(\frac{K}{T}).$$

$$\alpha_G = \mathcal{O}(\sigma_G^2) + \mathcal{O}(\frac{K^2}{T}).$$
This completes the proof.
\end{proof}


\subsection{Uniformly Distributed Worker Information Arrivals}\label{appdx_uniform}

Now, we consider the special case that the worker information arrivals are uniformly distributed, i.e., the worker in $\mathcal{M}_t$ could be regarded as a uniformly random sample without replacement in $[M]$.
As mentioned earlier, this special case acts as a widely-used assumption in FL and could deepen our understanding on the AFA-CD algorithm's performance in large-scale AFL systems.

\uniform*

\begin{proof}
The one-step update can be rewritten as:
$\x_{t+1} - \x_{t} = - \eta \eta_L \G_t$.
For cross-device FL,
$\G_t = \frac{1}{m} \sum_{i \in \mathcal{M}_t} \G_i(\x_{t-\tau_{t,i}})$, where $\tau_{t,i}$ is the delay for client $i$ in terms of the current global communication round $t$.
When $\tau_{t,i}=0, \forall i \in \mathcal{M}_t$, it degenerates to synchronous FL with partial worker participation.

Due to the $L$-smoothness in Assumption~1 
, taking expectation of $f(\x_{t+1})$ over the randomness in communication round $t$, we have:
\begin{align*}
	\mathbb{E} [f(\x_{t+1})] &\leq f(\x_t) + \underbrace{ \big< \nabla f(\x_t), \mathbb{E} [\x_{t+1} - \x_t] \big> }_{A_1} + \frac{L}{2} \underbrace{ \mathbb{E} [\| \x_{t+1} - \x_t \|^2 }_{A_2}
\end{align*}

We first bound $A_2$ as follows:
\begin{align*}
A_2 &= \mathbb{E} \| \x_{t+1} - \x_t \|^2 \\
&= \eta^2 \eta_L^2 \mathbb{E} \left\| \frac{1}{m} \sum_{i \in \mathcal{M}_t} G_i(\x_{t-\tau_{t,i}}) \right\|^2 \\
&\overset{(b1)}{\leq} \frac{2\eta^2 \eta_L^2}{m^2} \mathbb{E} \bigg[ \left\| \sum_{i \in \mathcal{M}_t} \Delta_i(\x_{t-\tau_{t,i}}) \right\|^2 + \frac{m}{K_t} \sigma_L^2 \bigg] \\
&\overset{(b2)}{\leq} \frac{2\eta^2 \eta_L^2}{m^2} \mathbb{E} \left\| \sum_{i=1}^{M} \mathbb{I}\{i \in \mathcal{M}_t\} \Delta_i(\x_{t-\tau_{t,i}}) \right\|^2 + \frac{2\eta^2 \eta_L^2}{m K_t} \sigma_L^2,
\end{align*}
where $(b1)$ is due to Lemma~\ref{lem:2} and $(b2)$ is due to the uniformly independent information arrival assumption.

To bound the term $A_1$, we have:
\begin{align*}
A_1 =& \big< \nabla f(\x_t), \mathbb{E} [\x_{t+1} - \x_t] \big> \\
=& - \eta \eta_L \left< \nabla f(\x_t), \mathbb{E} \left[\frac{1}{m} \sum_{i \in \mathcal{M}_t} \G_i(\x_{t - \tau_{t,i}}) \right] \right> \\
\overset{(b3)}{=}& - \eta \eta_L \left< \nabla f(\x_t), \frac{1}{M} \sum_{i \in [M]} \Delta_i(\x_{t - \tau_{t,i}}) \right> \\
\overset{(b4)}{=}& -\frac{1}{2} \eta \eta_L \left\| \nabla f(\x_t) \right\|^2 -\frac{1}{2} \eta \eta_L \left\| \frac{1}{M} \sum_{i \in [M]} \Delta_i(\x_{t - \tau_{t,i}}) \right\|^2 \\
&+ \frac{1}{2} \eta \eta_L \underbrace{ \bigg\| \nabla f(\x_t) - \frac{1}{M} \sum_{i \in [M]} \Delta_i(\x_{t - \tau_{t,i}}) \bigg\|^2}_{A_3},
\end{align*}
where $(b3)$ is due to the uniformly independent worker information arrival assumption and Lemma~\ref{lem:1},
$(b4)$ is due to the fact that $\langle \x, \y \rangle = \frac{1}{2} (\| \x \|^2 + \| \y \|^2 - \| \x - \y \|^2)$.

To further bound the term $A_3$, we have:
\begin{align*}
A_3 &=  \left\| \nabla f(\x_t) - \frac{1}{M} \sum_{i \in [M]} \Delta_i(\x_{t - \tau_{t,i}}) \right\|^2 \\
&\overset{(b5)}{=} \left\| \frac{1}{M} \sum_{i \in [M]} [\nabla f_i(\x_t) -  \Delta_i(\x_{t - \tau_{t,i}})] \right\|^2 \\
&\leq \frac{1}{M} \sum_{i \in [M]} \big\| \nabla f_i(\x_t) -  \Delta_i(\x_{t - \tau_{t,i}}) \big\|^2 \\
&= \frac{1}{M} \sum_{i \in [M]} \big\| \nabla f_i(\x_t) - \nabla f_i(\x_{t - \tau_{t,i}}) + \nabla f_i(\x_{t - \tau_{t,i}}) -  \Delta_i(\x_{t - \tau_{t,i}}) \big\|^2 \\
&\overset{(b6)}{\leq} \frac{1}{M} \sum_{i \in [M]} \bigg[ 2 \| \nabla f_i(\x_t) - \nabla f_i(\x_{t - \tau_{t,i}}) \|^2 + 2 \| \nabla f_i(\x_{t - \tau_{t,i}}) -  \Delta_i(\x_{t - \tau_{t,i}}) \|^2 \bigg] \\
%
&\overset{(b7)}{\leq} \underbrace{\frac{2L^2}{M} \sum_{i=1}^{M} \big\| \x_t - \x_{t - \tau_{t,i}} \big\|^2 }_{A_4} 
+ \frac{2}{M} \sum_{i=1}^{M} \underbrace{ \big\| \nabla f_i(\x_{t - \tau_{t,i}}) -  \Delta_i(\x_{t - \tau_{t,i}}) \big\|^2 }_{A_5},
\end{align*}
where $(b5)$ is due to the fact that $\nabla f(\x) = \frac{1}{M} \sum_{i \in [M]} \nabla f_i(\x)$,
$(b6)$ follows from the inequality $\| \x_1 + \x_2 + \cdots + \x_n \|^2 \leq n \sum_{i=1}^{n} \| \x_i \|^2$,
and $(b7)$ follows from the $L$-smoothness assumption (Assumption~1).

For $A_4$ and $A_5$, we have the same bounds as in the case of general worker information arrival processes:
\begin{align*}
A_4 &= \frac{2L^2}{M} \sum_{i=1}^{M} \big\| \x_t - \x_{t - \tau_{t,i}} \big\|^2 \\
&\leq \mathbb{E} \left[ \frac{4 L^2 \eta^2 \eta_L^2 \tau }{m^2} \sum_{k=t - \tau_{t, u}}^{t-1} \left( \| \sum_{i \in \mathcal{M}_k} \Delta_i(\x_{k - \tau_{k,i}}) \|^2 +  \frac{m}{K_k} \sigma_L^2 \right) \right] \\
&\leq \frac{4 L^2\eta^2 \eta_L^2 \tau }{m^2} \sum_{k=t - \tau_{t, u}}^{t-1} \left( \left\| \sum_{i=1}^{M} \mathbb{I}\{ i \in \mathcal{M}_k \} \Delta_i(\x_{k - \tau_{k,i}}) \right\|^2 +  \frac{m}{K_k} \sigma_L^2 \right)
\end{align*}

\begin{align*}
A_5 &= \big\| \nabla f_i(\x_{t - \tau_{t,i}}) -  \Delta_i(\x_{t - \tau_{t,i}}) \big\|^2 \\
&\leq 5 K_{t,i} L^2 \eta_L^2 (\sigma_L^2 + 6 K_{t,i} \sigma_G^2) + 30 K_{t,i}^2 L^2 \eta_L^2 \| \nabla f(\x_{t - \tau_{t,i}}) \|^2,
\end{align*}

With the above results of the term $A_{1}$ through $A_5$, we have:
\begin{align*}
&\mathbb{E}_t [f(\x_{t+1})] - f(\x_t) \leq \underbrace{ \big< \nabla f(\x_t), \mathbb{E}_t [\x_{t+1} - \x_t] \big> }_{A_1} + \frac{L}{2} \underbrace{ \mathbb{E}_t [\| \x_{t+1} - \x_t \|^2 }_{A_2} \\
&= -\frac{1}{2} \eta \eta_L \| \nabla f(\x_t) \|^2 -\frac{1}{2} \eta \eta_L \| \frac{1}{M} \sum_{i \in [M]} \Delta_i(\x_{t - \tau_{t,i}}) \|^2 + \frac{1}{2} \eta \eta_L \underbrace{ \left\| \nabla f(\x_t) - \frac{1}{M} \sum_{i \in [M]} \Delta_i(\x_{t - \tau_{t,i}}) \right\|^2}_{A_3} \\
& \quad + \frac{L \eta^2 \eta_L^2}{m^2} \mathbb{E} \left\| \sum_{i=1}^{M} \mathbb{I}\{i \in \mathcal{M}_t \} \Delta_i(\x_{t-\tau_{t,i}}) \right\|^2 + \frac{L \eta^2 \eta_L^2}{m K_t} \sigma_L^2 \\
&\leq -\frac{1}{2} \eta \eta_L \| \nabla f(\x_t) \|^2 -\frac{1}{2} \eta \eta_L \left\| \frac{1}{M} \sum_{i \in [M]} \Delta_i(\x_{t - \tau_{t,i}}) \right\|^2 
+ \frac{L \eta^2 \eta_L^2}{m^2} \mathbb{E} \left\| \sum_{i=1}^{M} \mathbb{I}\{i \in \mathcal{M}_t\} \Delta_i(\x_{t-\tau_{t,i}}) \right\|^2 \\
& \quad + \frac{1}{2} \eta \eta_L \bigg[\underbrace{\frac{2L^2}{M} \sum_{i=1}^{M} \| \x_t - \x_{t - \tau_{t,i}} \|^2 }_{A_4} 
+ \frac{2}{M} \sum_{i=1}^{M} \underbrace{ \| \nabla f_i(\x_{t - \tau_{t,i}}) -  \Delta_i(\x_{t - \tau_{t,i}}) \|^2 }_{A_5} \bigg]
+ \frac{L \eta^2 \eta_L^2}{m K_t} \sigma_L^2 \\
&\leq -\frac{1}{2} \eta \eta_L \| \nabla f(\x_t) \|^2 -\frac{1}{2} \eta \eta_L \left\| \frac{1}{M} \sum_{i \in [M]} \Delta_i(\x_{t - \tau_{t,i}}) \right\|^2 
+ \frac{L \eta^2 \eta_L^2}{m^2} \mathbb{E} \left\| \sum_{i=1}^{M} \mathbb{I}\{i \in \mathcal{M}_t\} \Delta_i(\x_{t-\tau_{t,i}}) \right\|^2  \\
& \quad + \eta \eta_L L^2 \frac{2 \eta^2 \eta_L^2 \tau }{m^2} \sum_{k=t - \tau_{t, u}}^{t-1} \left( \left\| \sum_{i=1}^{M} \mathbb{I}\{ i \in \mathcal{M}_k \} \Delta_i(\x_{k - \tau_{k,i}}) \right\|^2 +  \frac{m}{K_k} \sigma_L^2 \right) \\
& \quad + \eta \eta_L \frac{1}{M} \sum_{i=1}^{M} \left[5 K_{t, i} L^2 \eta_L^2 (\sigma_L^2 + 6 K_{t, i} \sigma_G^2) + 30 K_{t, i}^2 L^2 \eta_L^2 \| \nabla f(\x_{t - \tau_{t,i}}) \|^2 \right]
+ \frac{L \eta^2 \eta_L^2}{m K_t} \sigma_L^2 \\
%
&\leq \bigg[-\frac{1}{2} \eta \eta_L \| \nabla f(\x_t) \|^2 + (30 \eta L^2 \eta_L^3) \frac{1}{M} \sum_{i=1}^{M} K_{t, i}^2 \| \nabla f(\x_{t - \tau_{t,i}}) \|^2 \bigg] \\
&\qquad + \bigg[ -\frac{\eta \eta_L}{2M^2} \left\| \sum_{i=1}^{M} \Delta_i(\x_{t-\tau_{t,i}}) \right\|^2  + \frac{L \eta^2 \eta_L^2}{m^2} \mathbb{E} \left\| \sum_{i=1}^{M} \mathbb{I}\{i \in \mathcal{M}_t\} \Delta_i(\x_{t-\tau_{t,i}}) \right\|^2 \\
&\qquad + \frac{2 L^2 \eta^3 \eta_L^3 \tau}{m^2} \sum_{k=t - \tau_{t, \mu}}^{t-1} \mathbb{E} \left\| \sum_{i=1}^{M} \mathbb{I}\{i \in \mathcal{M}_k\} \Delta_i(\x_{k-\tau_{k,i}}) \right\|^2 \bigg] \\
&\qquad + \sigma_L^2 \bigg[ \frac{L\eta^2 \eta_L^2}{m K_t} + \frac{2 \tau L^2 \eta^3 \eta_L^3 \sum_{k=t - \tau_{t, \mu}}^{t-1} \frac{1}{K_k}}{m} + 5 \eta L^2 \eta_L^3 \frac{1}{M} \sum_{i=1}^{M} K_{t,i} \bigg] + \left[ 30 \eta L^2 \eta_L^3 \frac{1}{M} \sum_{i=1}^{M} K_{t,i}^2 \right] \sigma_G^2.
\end{align*}

Summing the above inequality from $t=0$ to $t= T-1$ yields:
\begin{align*}
&\mathbb{E} f(\x_{T}) - f(\x_0) \\
\leq& \sum_{t=0}^{T-1} \bigg[-\frac{1}{2} \eta \eta_L \| \nabla f(\x_t) \|^2 + (30 \eta L^2 \eta_L^3) \frac{1}{M} \sum_{i=1}^{M} K_{t, i}^2 \| \nabla f(\x_{t - \tau_{t,i}}) \|^2 \bigg] \\
&+ \sum_{t=0}^{T-1} \bigg[ -\frac{\eta \eta_L}{2M^2} \left\| \sum_{i=1}^{M} \Delta_i(\x_{t-\tau_{t,i}}) \right\|^2  + \frac{L \eta^2 \eta_L^2}{m^2} \mathbb{E} \left\| \sum_{i=1}^{M} \mathbb{I}\{i \in \mathcal{M}_t \} \Delta_i(\x_{t-\tau_{t,i}}) \right\|^2 \\
&+ \frac{2 L^2 \eta^3 \eta_L^3 \tau}{m^2} \sum_{k=t - \tau_{t, \mu}}^{t-1} \mathbb{E} \left\| \sum_{i=1}^{M} \mathbb{I}\{i \in \mathcal{M}_k\} \Delta_i(\x_{k-\tau_{k,i}}) \right\|^2 \bigg] \\
&+ \sum_{t=0}^{T-1} \left[ \sigma_L^2 \left( \frac{L\eta^2 \eta_L^2}{m K_t} + \frac{2 \tau L^2 \eta^3 \eta_L^3 \sum_{k=t - \tau_{t, \mu}}^{t-1} \frac{1}{K_k}}{m} + 5 \eta L^2 \eta_L^3 \frac{1}{M} \sum_{i=1}^{M} K_{t,i} \right) + \left( 30 \eta L^2 \eta_L^3 \frac{1}{M} \sum_{i=1}^{M} K_{t,i}^2 \right) \sigma_G^2 \right] \\
\overset{(b8)}{\leq}& \sum_{t=0}^{T-1} \bigg[ -\frac{1}{2} \eta \eta_L \| \nabla f(\x_t) \|^2 + (30 \eta L^2 \eta_L^3 \tau) \frac{1}{M} \sum_{i=1}^{M} K_{t, i}^2 \| \nabla f(\x_t) \|^2 \bigg] \\
&+ \sum_{t=0}^{T-1} \bigg[ -\frac{\eta \eta_L}{2M^2} \left\| \sum_{i=1}^{M} \Delta_i(\x_{t-\tau_{t,i}}) \right\|^2  + \frac{ L \eta^2 \eta_L^2}{m^2} \mathbb{E}\| \sum_{i=1}^{M} \mathbb{I}\{ i \in \mathcal{M}_t \} \Delta_i(\x_{t-\tau_{t,i}}) \|^2 \\
&+ \frac{2 L^2 \eta^3 \eta_L^3 \tau^2}{m^2} \mathbb{E} \left\| \sum_{i \in [M]} \mathbb{I}\{ i \in \mathcal{M}_t \} \Delta_i(\x_{t - \tau_{t,i}}) \right\|^2 \bigg] \\
&+ \sum_{t=0}^{T-1} \left[ \sigma_L^2 \left( \frac{L\eta^2 \eta_L^2}{m K_t} + \frac{2 \tau L^2 \eta^3 \eta_L^3 \sum_{k=t - \tau_{t, \mu}}^{t-1} \frac{1}{K_k}}{m} + 5 \eta L^2 \eta_L^3 \frac{1}{M} \sum_{i=1}^{M} K_{t,i} \right) + \left( 30 \eta L^2 \eta_L^3 \frac{1}{M} \sum_{i=1}^{M} K_{t,i}^2 \right) \sigma_G^2 \right], \\
\overset{(b9)}{\leq}& \sum_{t=0}^{T-1} \bigg[-\frac{1}{2} \eta \eta_L \| \nabla f(\x_t) \|^2 + (30 \eta L^2 \eta_L^3 \tau) \hat{K}_t^2 \| \nabla f(\x_t) \|^2 \bigg] \\
&+ \sum_{t=0}^{T-1} \bigg[ -\frac{\eta \eta_L}{2M^2} \left\| \sum_{i=1}^{M} \Delta_i(\x_{t-\tau_{t,i}}) \right\|^2  + \frac{ L \eta^2 \eta_L^2}{m^2} \mathbb{E}\left\| \sum_{i=1}^{M} \mathbb{I}\{ i \in \mathcal{M}_t \} \Delta_i(\x_{t-\tau_{t,i}}) \right\|^2 \\
&+ \frac{2 L^2 \eta^3 \eta_L^3 \tau^2}{m^2} \mathbb{E} \left\| \sum_{i \in [M]} \mathbb{I}\{ i \in \mathcal{M}_t \} \Delta_i(\x_{t - \tau_{t,i}}) \right\|^2 \bigg] \\
&+ \sum_{t=0}^{T-1} \left[ \sigma_L^2 \left( \frac{L\eta^2 \eta_L^2}{m K_t} + \frac{2 \tau^2 L^2 \eta^3 \eta_L^3}{m} \frac{1}{K_t} + 5 \eta L^2 \eta_L^3 \bar{K}_t \right) + \left( 30 \eta L^2 \eta_L^3 \hat{K}_t^2 \right) \sigma_G^2 \right],
\end{align*}
where $(b8)$ is due to the fact that the delay in the system is less than $\tau$, $(b9)$ follows from that
$\hat{K}_t^2 = \frac{1}{M} \sum_{i=1}^{M} K_{t, i}^2, \bar{K}_t = \frac{1}{M} \sum_{i=1}^{M} K_{t,i}$.

By letting $\z_{i} = \Delta_i(\x_{t-\tau_{t,i}})$ (omitting the communication round index $t$ for notation simplicity), we have that:

\begin{align*}
\| \sum_{i=1}^{M} \z_i \|^2 &= \sum_{i \in [M]} \|  \z_i \|^2 + \sum_{i \neq j} \langle \z_i, \z_j \rangle, \\
&\overset{(b10)}{=} \sum_{i \in [M]} M \|  \z_i \|^2 - \frac{1}{2} \sum_{i \neq j} \| \z_i - \z_j \|^2, \\
\mathbb{E}\| \sum_{i=1}^{M} \mathbb{I} \{ i \in \mathcal{M}_t \} \z_i \|^2 &= \sum_{i \in [M]} \mathbb{P} \{ i \in \mathcal{M}_t \} \|  \z_i \|^2 + \sum_{i \neq j} \mathbb{P} \{ i, j \in \mathcal{M}_t \} \langle \z_i, \z_j \rangle \\
&\overset{(b11)}{=} \frac{m}{M} \sum_{i \in [M]} \| \z_i \|^2 + \frac{m(m-1)}{M(M-1)} \sum_{i \neq j} \langle \z_i, \z_j \rangle \\
&\overset{(b12)}{=} \frac{m^2}{M} \sum_{i \in [M]} \| \z_i \|^2 - \frac{m(m-1)}{2M(M-1)} \sum_{i \neq j} \| \z_i - \z_j \|^2,
\end{align*}
where $(b10)$ and $(b12)$ are due to the fact that $\big<\x, \y \big> = \frac{1}{2} [ \| \x \|^2 + \| \y \|^2 - \| \x - \y \|^2] \leq \frac{1}{2} [ \| \x \|^2 + \| \y \|^2 ],$
$(b11)$ follows from the fact that $\mathbb{P} \{ i \in \mathcal{M}_t \} = \frac{m}{M}$ and $\mathbb{P} \{ i, j \in \mathcal{M}_t \} = \frac{m(m-1)}{M(M-1)}$.
It then follows that:
\begin{align*}
& -\frac{\eta \eta_L}{2M^2} \| \sum_{i=1}^{M} \z_{i} \|^2  + \frac{L \eta^2 \eta_L^2}{m^2} \mathbb{E} \left\| \sum_{i=1}^{M} \mathbb{I}\{ i \in \mathcal{M}_t \} \z_{i} \right\|^2
+ \frac{L^2 \eta^3 \eta_L^3 \tau^2}{m^2} \mathbb{E} \left\| \sum_{i=1}^{M} \mathbb{I}\{ i \in \mathcal{M}_t \} \z_{i} \right\|^2 \\
&\leq \bigg[ -\frac{\eta \eta_L}{2M} + (\frac{L\eta^2 \eta_L^2}{M} + \frac{ L^2 \eta^3 \eta_L^3 \tau^2}{M}) \bigg] \sum_{i=1}^{M} \| \z_{i} \|^2 + \frac{\eta \eta_L}{4M^2} \sum_{i \neq j} \| \z_i - \z_j \|^2 \\
&\leq \bigg[ -\frac{\eta \eta_L}{2M} + (\frac{L\eta^2 \eta_L^2}{M} + \frac{ L^2 \eta^3 \eta_L^3 \tau^2}{M}) \bigg] \sum_{i=1}^{M} \| \z_{i} \|^2 + \frac{\eta \eta_L (M-1)}{2M^2} \sum_{i=1}^{M} \| \z_{i} \|^2 \\
&= \bigg[ -\frac{\eta \eta_L}{2M^2} + (\frac{L\eta^2 \eta_L^2}{M} + \frac{ L^2 \eta^3 \eta_L^3 \tau^2}{M}) \bigg] \sum_{i=1}^{M} \| \z_{i} \|^2 \\
&\leq 0
%
\end{align*}
The last inequality follows from $L \eta \eta_L + L^2 \eta^2 \eta_L^2 \tau^2 \leq \frac{1}{2M}$.


Using the above results, we finally have:
\begin{align*}
	\mathbb{E} f(\x_{T}) - f(\x_0)
	&\leq \sum_{t=0}^{T-1} \bigg[-\frac{1}{2} \eta \eta_L \| \nabla f(\x_t) \|^2 + (30 \eta L^2 \eta_L^3 \tau) \hat{K}_t^2 \| \nabla f(\x_t) \|^2 \bigg] \\
	&\qquad + \sum_{t=0}^{T-1} \left[ \sigma_L^2 \left( \frac{L\eta^2 \eta_L^2}{m K_t} + \frac{2 \tau^2 L^2 \eta^3 \eta_L^3}{m} \frac{1}{K_t} + 5 \eta L^2 \eta_L^3 \bar{K}_t \right) + \left( 30 \eta L^2 \eta_L^3 \hat{K}_t^2 \right) \sigma_G^2 \right] \\
	&\leq \sum_{t=0}^{T-1} - \frac{1}{4} \eta \eta_L \| \nabla f(\x_t) \|^2 + T \eta \eta_L \big[ \alpha_L \sigma^2_L + \alpha_G \sigma_G^2 \big]
\end{align*}
where $(b13)$ follows from the fact that
$$\frac{1}{4} \leq \frac{1}{2} - 30 L^2 \hat{K}_t^2 \eta_L^2 \tau $$ if
$120 L^2 \hat{K}_t^2 \eta_L^2 \tau < 1, \forall t$;
$\alpha_L$ and $\alpha_G$ are defined as following:
$$\alpha_L = \frac{1}{T} \sum_{t=0}^{T-1} \left[ ( \frac{ L \eta \eta_L}{m K_t} + \frac{ 2 \tau^2 L^2 \eta^2 \eta_L^2}{m K_t} + 5 \bar{K}_t L^2 \eta_L^2) \right],$$
and
$$\alpha_G = \frac{1}{T} \sum_{t=0}^{T-1} \left[30 \hat{K}_t^2 L^2 \eta_L^2 \right].$$

Lastly, by rearranging and telescoping, we have
\begin{align*}
\frac{1}{T} \sum_{t=0}^{T-1} \mathbb{E}\| \nabla f(\x_t) \|^2 \leq \frac{4(f_0 - f_*)}{ \eta \eta_L T} + 4 \big[ \alpha_L \sigma^2_L + \alpha_G \sigma_G^2 \big].
\end{align*}
This completes the proof.
\end{proof}

\uniformC*

\begin{proof}
Suppose a constant local step $K$, let $\eta_L = \frac{1}{\sqrt{T}}$, and $\eta = \sqrt{mK}$, then it follows that:
$$\alpha_L = \mathcal{O}(\frac{1}{m^{1/2} K^{1/2} T^{1/2}}) + \mathcal{O}(\frac{\tau^2}{T}) + \mathcal{O}(\frac{K}{T}) ,$$

$$\alpha_G = \mathcal{O}(\frac{K^2}{T}),$$
and the proof is complete.
\end{proof}

\section{Proof of the performance results of the AFA-CS algorithm} \label{appdx_silo}

\silo*

\begin{proof}
We divide the stochastic gradient returns \{ $\G_i $\} into two groups, one is for those without delay ($\G_i(x_t), i \in \mathcal{M}_t, | \mathcal{M}_t | =m'$) and the other is for those with delay ($\G_i(\x_{t-\tau_{t, i}}), i \in \mathcal{M}_t^c, | \mathcal{M}_t^c | = M - m'$).

Then, the update step can be written as follows:
\begin{align*}
   \x_{t+1} - \x_{t} &= - \frac{\eta \eta_L}{M} \bigg[ \sum_{i \in \mathcal{M}_t} G_i(\x_t) + \sum_{i \in \mathcal{M}^c_t} G_i(\x_{t-\tau_{t, i}}) \bigg].
\end{align*}
Due to the $L$-smoothness assumption, taking expectation of $f(\x_{t+1})$ over the randomness in communication round $t$, we have:
\begin{align*}
	\mathbb{E} [f(\x_{t+1})] &\leq f(\x_t) + \underbrace{ \big< \nabla f(\x_t), \mathbb{E} [\x_{t+1} - \x_t] \big> }_{A_1} + \frac{L}{2} \underbrace{ \mathbb{E} [\| \x_{t+1} - \x_t \|^2 }_{A_2}
\end{align*}

We first bound $A_2$ as follows:
\begin{align*}
A_2 &= \mathbb{E}[\| \x_{t+1} - \x_t \|^2 ] \\
	&= \frac{\eta^2 \eta_L^2}{M^2} \mathbb{E}  \bigg[ \bigg\| \sum_{i \in \mathcal{M}_t} G_i(\x_t) + \sum_{i \in \mathcal{M}^c_t} G_i(\x_{t-\tau_{t, i}}) \bigg\|^2 \bigg]  \\
	&= \frac{\eta^2 \eta_L^2}{M^2} \mathbb{E} \left[ \left\| \sum_{i \in \mathcal{M}_t} \left[ G_i(\x_t) - \Delta_i(\x_t) \right] + \sum_{i \in \mathcal{M}^c_t} \left[ G_i(\x_{t-\tau_{t, i}}) - \Delta_i(\x_{t-\tau_{t, i}}) \right] + \sum_{i \in \mathcal{M}_t} \Delta_i(\x_t) + \sum_{i \in \mathcal{M}^c_t}  \Delta_i(\x_{t-\tau_{t, i}}) \right\|^2 \right] \\
    &\overset{(c1)}{\leq} \frac{2 \eta^2 \eta_L^2}{M^2} \left( \sum_{i \in \mathcal{M}_t} \frac{1}{K_{t,i}} + \sum_{i \in \mathcal{M}^c_t} \frac{1}{K_{t-\tau_{t, i}, i}}\right) \sigma_L^2 + \frac{2 \eta^2 \eta_L^2}{M^2} \left[ \left\| \sum_{i \in \mathcal{M}_t} \Delta_i(\x_t) + \sum_{i \in \mathcal{M}^c_t} \Delta_i(\x_{t-\tau_{t, i}}) \right\|^2 \right] \\
    &= \frac{2 \eta^2 \eta_L^2}{M K_t} \sigma_L^2 + \frac{2 \eta^2 \eta_L^2}{M^2} \left[ \left\| \sum_{i \in \mathcal{M}_t} \Delta_i(\x_t) + \sum_{i \in \mathcal{M}^c_t} \Delta_i(\x_{t-\tau_{t, i}}) \right\|^2 \right],
\end{align*}
where $(c1)$ follows from the similar result in Lemma~\ref{lem:1} and $\frac{1}{K_t} = \frac{1}{M} \left( \sum_{i \in \mathcal{M}_t} \frac{1}{K_{t,i}} + \sum_{i \in \mathcal{M}^c_t} \frac{1}{K_{t-\tau_{t, i}, i}}\right)$.
To bound the term $A_1$, we have:
\begin{align*}
    A_1 &= \mathbb{E} \big< \nabla f(\x_t), \x_{t+1} - \x_t \big> \\
    &= \mathbb{E} \left< \nabla f(\x_t), - \frac{\eta \eta_L}{M} \left[ \sum_{i \in \mathcal{M}_t} \G_i(\x_t) + \sum_{i \in \mathcal{M}^c_t} G_i(\x_{t-\tau_{t, i}}) \right] \right> \\
    &= - \eta \eta_L \left< \nabla f(\x_t), \frac{1}{M} \left[ \sum_{i \in \mathcal{M}_t} \Delta_i(\x_t) + \sum_{i \in \mathcal{M}^c_t} \Delta_i(\x_{t-\tau_{t, i}}) \right] \right> \\
    &= - \frac{\eta \eta_L}{2} \left\| \nabla f(\x_t) \right\|^2 - \frac{\eta \eta_L}{2M^2} \left\| \sum_{i \in \mathcal{M}_t} \Delta_i(\x_t) + \sum_{i \in \mathcal{M}^c_t} \Delta_i(\x_{t-\tau_{t, i}}) \right\|^2 \\
    &\quad + \frac{\eta \eta_L}{2} \left\| \nabla f(\x_t) - \frac{1}{M} \left[ \sum_{i \in \mathcal{M}_t} \Delta_i(\x_t) + \sum_{i \in \mathcal{M}^c_t} \Delta_i(\x_{t-\tau_{t, i}}) \right] \right\|^2 \\
    &= - \frac{\eta \eta_L}{2} \left\| \nabla f(\x_t) \right\|^2 - \frac{\eta \eta_L}{2M^2} \left\| \sum_{i \in \mathcal{M}_t} \Delta_i(\x_t) + \sum_{i \in \mathcal{M}^c_t} \Delta_i(\x_{t-\tau_{t, i}}) \right\|^2 \\
    &\quad + \frac{\eta \eta_L}{2M^2} \left\| \sum_{i \in \mathcal{M}_t} \left[ \nabla f(\x_t) - \Delta_i(\x_t) \right] + \sum_{i \in \mathcal{M}^c_t} \left[ \nabla f(\x_{t-\tau_{t, i}}) - \Delta_i(\x_{t-\tau_{t, i}}) \right] + \sum_{i \in \mathcal{M}^c_t} \left[ \nabla f(\x_t) - \nabla f(\x_{t-\tau_{t, i}}) \right] \right\|^2 \\
    &\leq - \frac{\eta \eta_L}{2} \left\| \nabla f(\x_t) \right\|^2 - \frac{\eta \eta_L}{2M^2} \left\| \sum_{i \in \mathcal{M}_t} \Delta_i(\x_t) + \sum_{i \in \mathcal{M}^c_t} \Delta_i(\x_{t-\tau_{t, i}}) \right\|^2 \\
    &\quad + \frac{\eta \eta_L}{M^2} \left\| \sum_{i \in \mathcal{M}_t} \left[ \nabla f(\x_t) - \Delta_i(\x_t) \right] + \sum_{i \in \mathcal{M}^c_t} \left[ \nabla f(\x_{t-\tau_{t, i}}) - \Delta_i(\x_{t-\tau_{t, i}}) \right] \right\|^2 \\
    &\quad + \frac{\eta \eta_L}{M^2}\left\| \sum_{i \in \mathcal{M}^c_t} \left[ \nabla f(\x_t) - \nabla f(\x_{t-\tau_{t, i}}) \right] \right\|^2 \\
    &\leq - \frac{\eta \eta_L}{2} \left\| \nabla f(\x_t) \right\|^2 - \frac{\eta \eta_L}{2M^2} \left\| \sum_{i \in \mathcal{M}_t} \Delta_i(\x_t) + \sum_{i \in \mathcal{M}^c_t} \Delta_i(\x_{t-\tau_{t, i}}) \right\|^2 \\
    &\quad + \frac{\eta \eta_L}{M} \left[ \sum_{i \in \mathcal{M}_t} \left\| \nabla f(\x_t) - \Delta_i(\x_t) \right\|^2 + \sum_{i \in \mathcal{M}^c_t} \left\| \nabla f(\x_{t-\tau_{t, i}}) - \Delta_i(\x_{t-\tau_{t, i}}) \right\|^2 \right] \\
    &\quad + \frac{\eta \eta_L (M - m^{'})}{M^2} \sum_{i \in \mathcal{M}^c_t} \left\| \nabla f(\x_t) - \nabla f(\x_{t-\tau_{t, i}}) \right\|^2
\end{align*}

For each worker $i$, we have:
\begin{align*}
\| \nabla f_i(\x_t) -  \Delta_i(\x_t) \|^2 
&= \left\| \nabla f_i(\x_t) - \frac{1}{K_{t,i}} \sum_{j=0}^{K_{t,i}-1} \nabla f_i(\x_{t, i}^j) \right\|^2 \\
&= \frac{1}{K_{t,i}} \sum_{j=0}^{K_{t,i}-1} \| \nabla f_i(\x_{t}) - \nabla f_i(\x_{t, i}^j) \|^2 \\
&\leq \frac{L^2}{K_{t,i}} \sum_{j=0}^{K_{t,i}-1} \| \x_{t} - \x_{t, i}^j \|^2 \\
&\overset{(c2)}{\leq} 5 K_{t, i} L^2 \eta_L^2 \sigma_L^2 + 30 K_{t, i}^2 L^2 \eta_L^2 \sigma_G^2 + 30 K_{t, i}^2 L^2 \eta_L^2 \| \nabla f(\x_{t}) \|^2,
\end{align*}
where $(c2)$ follows from the same bound of $A6$ specified in Eq.~\eqref{appdx:a6}.

\begin{align*}
    \left\| \nabla f(\x_t) - \nabla f(\x_{t-\tau_{t, i}}) \right\|^2 &\leq L^2  \left\| \x_t - \x_{t-\tau_{t, i}} \right\|^2 \\
    &\leq L^2 \tau_{t, i} \sum_{u=0}^{\tau_{t, i}- 1} \left\| \x_{t-u} - \x_{t-u-1} \right\|^2.
\end{align*}

\begin{align*}
    A_1 &\leq - \frac{\eta \eta_L}{2} \left\| \nabla f(\x_t) \right\|^2 - \frac{\eta \eta_L}{2M^2} \left\| \sum_{i \in \mathcal{M}_t} \Delta_i(\x_t) + \sum_{i \in \mathcal{M}^c_t} \Delta_i(\x_{t-\tau_{t, i}}) \right\|^2 \\
    &\quad + \frac{\eta \eta_L}{M} \left[ \sum_{i \in \mathcal{M}_t} \left\| \nabla f(\x_t) - \Delta_i(\x_t) \right\|^2 + \sum_{i \in \mathcal{M}^c_t} \left\| \nabla f(\x_{t-\tau_{t, i}}) - \Delta_i(\x_{t-\tau_{t, i}}) \right\|^2 \right] \\
    &\quad + \frac{\eta \eta_L (M - m^{'})}{M^2} \sum_{i \in \mathcal{M}^c_t} \left\| \nabla f(\x_t) - \nabla f(\x_{t-\tau_{t, i}}) \right\|^2 \\
    &\leq - \frac{\eta \eta_L}{2} \left\| \nabla f(\x_t) \right\|^2 - \frac{\eta \eta_L}{2M^2} \left\| \sum_{i \in \mathcal{M}_t} \Delta_i(\x_t) + \sum_{i \in \mathcal{M}^c_t} \Delta_i(\x_{t-\tau_{t, i}}) \right\|^2 \\
    &\quad + \frac{\eta \eta_L}{M} \left[ \left(5L^2 \eta_L^2 \sigma_L^2 \right) \left(\sum_{i \in \mathcal{M}_t} K_{t,i} + \sum_{i \in \mathcal{M}^c_t} K_{t-\tau_{t, i},i} \right) + \left(30L^2 \eta_L^2 \sigma_G^2 \right) \left(\sum_{i \in \mathcal{M}_t} K_{t,i}^2 + \sum_{i \in \mathcal{M}^c_t} K_{t-\tau_{t, i},i}^2 \right) \right] \\
    &\quad + \frac{\eta \eta_L}{M} \left(30L^2 \eta_L^2 \right) \left(\sum_{i \in \mathcal{M}_t} K_{t,i}^2 \left\| \nabla f(\x_t) \right\|^2 + \sum_{i \in \mathcal{M}^c_t} K_{t-\tau_{t, i},i}^2 \left\| \nabla f(\x_{t-\tau_{t,i}}) \right\|^2 \right) \\
    &\quad + \frac{\eta \eta_L (M - m^{'}) L^2}{M^2} \sum_{i \in \mathcal{M}^c_t} \left( \tau_{t, i} \sum_{u=0}^{\tau_{t, i}- 1} \left\| \x_{t-u} - \x_{t-u-1} \right\|^2 \right)
\end{align*}

Combining $A_1$ abd $A_2$, we have:
\begin{align*}
	&\mathbb{E} [f(\x_{t+1})] - f(\x_t) \leq \underbrace{ \big< \nabla f(\x_t), \mathbb{E} [\x_{t+1} - \x_t] \big> }_{A_1} + \frac{L}{2} \underbrace{ \mathbb{E} [\| \x_{t+1} - \x_t \|^2 }_{A_2} \\
    &\leq - \frac{\eta \eta_L}{2} \left\| \nabla f(\x_t) \right\|^2 - \frac{\eta \eta_L}{2M^2} \left\| \sum_{i \in \mathcal{M}_t} \Delta_i(\x_t) + \sum_{i \in \mathcal{M}^c_t} \Delta_i(\x_{t-\tau_{t, i}}) \right\|^2 + \frac{L}{2} \mathbb{E} \| \x_{t-1} - \x_t \|^2 \\
    &\quad + \frac{\eta \eta_L}{M} \left[ \left(5L^2 \eta_L^2 \sigma_L^2 \right) \left(\sum_{i \in \mathcal{M}_t} K_{t,i} + \sum_{i \in \mathcal{M}^c_t} K_{t-\tau_{t, i},i} \right) + \left(30L^2 \eta_L^2 \sigma_G^2 \right) \left(\sum_{i \in \mathcal{M}_t} K_{t,i}^2 + \sum_{i \in \mathcal{M}^c_t} K_{t-\tau_{t, i},i}^2 \right) \right] \\
    &\quad + \frac{\eta \eta_L}{M} \left(30L^2 \eta_L^2 \right) \left(\sum_{i \in \mathcal{M}_t} K_{t,i}^2 \left\| \nabla f(\x_t) \right\|^2 + \sum_{i \in \mathcal{M}^c_t} K_{t-\tau_{t, i},i}^2 \left\| \nabla f(\x_{t-\tau_{t,i}}) \right\|^2 \right) \\
    &\quad +  \frac{\eta \eta_L (M - m^{'}) L^2}{M^2} \sum_{i \in \mathcal{M}^c_t} \left( \tau_{t, i} \sum_{u=0}^{\tau_{t, i}- 1} \left\| \x_{t-u} - \x_{t-u-1} \right\|^2 \right).
\end{align*}

Summing from $t=0$ to $T-1$, we have:
\begin{align*}
	&\mathbb{E} [f(\x_{T})] - f(\x_0) 
    \leq - \frac{\eta \eta_L}{2} \sum_{t=0}^{T-1} \left\| \nabla f(\x_t) \right\|^2 - \frac{\eta \eta_L}{2M^2} \sum_{t=0}^{T-1} \left\| \sum_{i \in \mathcal{M}_t} \Delta_i(\x_t) + \sum_{i \in \mathcal{M}^c_t} \Delta_i(\x_{t-\tau_{t, i}}) \right\|^2 + \frac{L}{2} \sum_{t=0}^{T-1} \mathbb{E} \| \x_{t+1} - \x_t \|^2 \\
    &\quad + \frac{\eta \eta_L}{M} \sum_{t=0}^{T-1} \left[ \left(5L^2 \eta_L^2 \sigma_L^2 \right) \left(\sum_{i \in \mathcal{M}_t} K_{t,i} + \sum_{i \in \mathcal{M}^c_t} K_{t-\tau_{t, i},i} \right) + \left(30L^2 \eta_L^2 \sigma_G^2 \right) \left(\sum_{i \in \mathcal{M}_t} K_{t,i}^2 + \sum_{i \in \mathcal{M}^c_t} K_{t-\tau_{t, i},i}^2 \right) \right] \\
    &\quad + \frac{\eta \eta_L}{M} \left(30L^2 \eta_L^2 \right) \sum_{t=0}^{T-1} \left(\sum_{i \in \mathcal{M}_t} K_{t,i}^2 \left\| \nabla f(\x_t) \right\|^2 + \sum_{i \in \mathcal{M}^c_t} K_{t-\tau_{t, i},i}^2 \left\| \nabla f(\x_{t-\tau_{t,i}}) \right\|^2 \right) \\
    &\quad + \frac{\eta \eta_L (M - m^{'}) L^2}{M^2} \sum_{t=0}^{T-1} \sum_{i \in \mathcal{M}^c_t} \left( \tau_{t, i} \sum_{u=0}^{\tau_{t, i}- 1} \left\| \x_{t-u} - \x_{t-u-1} \right\|^2 \right) \\
    &\overset{(c3)}{\leq} - \frac{\eta \eta_L}{2} \sum_{t=0}^{T-1} \left\| \nabla f(\x_t) \right\|^2 - \frac{\eta \eta_L}{2M^2} \sum_{t=0}^{T-1} \left\| \sum_{i \in \mathcal{M}_t} \Delta_i(\x_t) + \sum_{i \in \mathcal{M}^c_t} \Delta_i(\x_{t-\tau_{t, i}}) \right\|^2  \\
    &\quad + \frac{\eta \eta_L}{M} \sum_{t=0}^{T-1} \left[ \left(5L^2 \eta_L^2 \sigma_L^2 \right) \bar{K}_t + \left(30L^2 \eta_L^2 \sigma_G^2 \right) \hat{K}_t^2 \right] + \frac{\eta \eta_L \tau}{M} \left(30L^2 \eta_L^2 \right) \sum_{t=0}^{T-1} \left(\sum_{i \in [M]} K_{t,i}^2 \right) \left\| \nabla f(\x_t) \right\|^2 \\
    &\quad + \left( \frac{\eta \eta_L (M - m^{'})^2 L^2 \tau^2}{M^2} + \frac{L}{2} \right) \sum_{t=0}^{T-1} \left( \left\| \x_{t+1} - \x_{t} \right\|^2 \right) \\
    &\leq - \frac{\eta \eta_L}{2} \sum_{t=0}^{T-1} \left\| \nabla f(\x_t) \right\|^2 - \left[ \frac{\eta \eta_L}{2M^2} - \left( \frac{\eta \eta_L (M - m^{'})^2 L^2 \tau^2}{M^2} + \frac{L}{2} \right) \frac{2 \eta^2 \eta_L^2}{M^2} \right] \sum_{t=0}^{T-1} \left\| \sum_{i \in \mathcal{M}_t} \Delta_i(\x_t) + \sum_{i \in \mathcal{M}^c_t} \Delta_i(\x_{t-\tau_{t, i}}) \right\|^2  \\
    &\quad + \frac{\eta \eta_L}{M} \sum_{t=0}^{T-1} \left[ 5L^2 \eta_L^2 \bar{K}_t \sigma_L^2 + 30L^2 \eta_L^2 \hat{K}_t^2 \sigma_G^2 \right] + \frac{\eta \eta_L \tau}{M} \left(30L^2 \eta_L^2 \right) \sum_{t=0}^{T-1} \left(\sum_{i \in [M]} K_{t,i}^2 \right) \left\| \nabla f(\x_t) \right\|^2 \\
    &\quad + \left( \frac{\eta \eta_L (M - m^{'})^2 L^2 \tau^2}{M^2} + \frac{L}{2} \right) \frac{2 \eta^2 \eta_L^2}{M} \sigma_L^2 \sum_{t=0}^{T-1} \frac{1}{K_t} \\
    &\overset{(c4)}{\leq} - \sum_{t=0}^{T-1} \left[\frac{\eta \eta_L}{2} - \frac{30 L^2 \eta \eta_L^3 \tau}{M} \left(\sum_{i \in [M]} K_{t,i}^2 \right) \right] \left\| \nabla f(\x_t) \right\|^2  \\
    &\quad + \frac{\eta \eta_L}{M} \sum_{t=0}^{T-1} \left[ \left[5L^2 \eta_L^2 \bar{K}_t + \left( \frac{\eta \eta_L (M - m^{'})^2 L^2 \tau^2}{M^2} + \frac{L}{2} \right) 2 \eta \eta_L \frac{1}{K_t} \right] \sigma_L^2 + \left[ 30L^2 \eta_L^2 \hat{K}_t^2 \right] \sigma_G^2 \right] \\
    &\overset{(c5)}{\leq} - \sum_{t=0}^{T-1} \frac{\eta \eta_L}{4} \left\| \nabla f(\x_t) \right\|^2 + \frac{\eta \eta_L}{M} \sum_{t=0}^{T-1} \left[ \left[5L^2 \eta_L^2 \bar{K}_t + \left( \frac{\eta \eta_L (M - m^{'})^2 L^2 \tau^2}{M^2} + \frac{L}{2} \right) 2 \eta \eta_L\frac{1}{K_t} \right] \sigma_L^2 + \left[ 30L^2 \eta_L^2 \hat{K}_t^2 \right] \sigma_G^2 \right],
\end{align*}
where $(c3)$ follows from the facts that $\bar{K}_t = \left(\sum_{i \in \mathcal{M}_t} K_{t,i} + \sum_{i \in \mathcal{M}^c_t} K_{t-\tau_{t, i},i} \right)$, $\hat{K}_t^2 = \left(\sum_{i \in \mathcal{M}_t} K_{t,i}^2 + \sum_{i \in \mathcal{M}^c_t} K_{t-\tau_{t, i},i}^2 \right)$, and $\tau$ is the maximum delay;
$(c4)$ is due to $\left[ \frac{\eta \eta_L}{2M^2} - \left( \frac{\eta \eta_L (M - m^{'})^2 L^2 \tau^2}{M^2} + \frac{L}{2} \right) \frac{2 \eta^2 \eta_L^2}{M^2} \right] \geq 0$ if $\left( \frac{\eta \eta_L (M - m^{'})^2 L^2 \tau^2}{M^2} + \frac{L}{2} \right) \eta \eta_L\leq \frac{1}{4}$;
and $(c5)$ is due to $\frac{\eta \eta_L}{4} \leq \left[ \frac{\eta \eta_L}{2} - \frac{30 L^2 \eta \eta_L^3 \tau}{M} \left(\sum_{i \in [M]} K_{t,i}^2 \right) \right]$ if $\frac{30 L^2 \eta_L^2 \tau}{M} \left(\sum_{i \in [M]} K_{t,i}^2 \right) \leq \frac{1}{4}$.

By rearranging, we have:
\begin{align*}
    \frac{1}{T} \sum_{t=0}^{T-1} \left\| \nabla f(\x_t) \right\|^2 &\leq \frac{4 f(\x_0) - f(\x_T)}{\eta \eta_L T} + \alpha_L \sigma_L^2 + \alpha_G \sigma_G^2,
\end{align*}
where 
\begin{align*}
    \alpha_L &= \frac{4}{M} \left[5L^2 \eta_L^2 \frac{1}{T} \sum_{t=0}^{T-1} \bar{K}_t + \left( \frac{2\eta^2 \eta_L^2 (M - m^{'})^2 L^2 \tau^2}{M^2} + L \eta \eta_L \right) \frac{1}{T} \sum_{t=0}^{T-1} \frac{1}{K_t} \right], \\
    \alpha_G &= \frac{120L^2 \eta_L^2}{M} \frac{1}{T} \sum_{t=0}^{T-1} \hat{K}_t^2.
\end{align*}

Note 
\begin{align*}
    \frac{1}{K_t} &= \frac{1}{M} \left( \sum_{i \in \mathcal{M}_t} \frac{1}{K_{t,i}} + \sum_{i \in \mathcal{M}^c_t} \frac{1}{K_{t-\tau_{t, i}, i}}\right), \\
    \bar{K}_t &= \left(\sum_{i \in \mathcal{M}_t} K_{t,i} + \sum_{i \in \mathcal{M}^c_t} K_{t-\tau_{t, i},i} \right), \\
    \hat{K}_t^2 &= \left(\sum_{i \in \mathcal{M}_t} K_{t,i}^2 + \sum_{i \in \mathcal{M}^c_t} K_{t-\tau_{t, i},i}^2 \right).
\end{align*}
This completes the proof.
\end{proof}

\siloC*

\begin{proof}
Let $\eta_L = \frac{1}{\sqrt{T}}$, and $\eta = \sqrt{MK}$.
It then follows that:
$$\alpha_L = \mathcal{O}(\frac{1}{M^{1/2} K^{1/2} T^{1/2}}) + \mathcal{O}(\frac{K}{MT}) + \mathcal{O}(\frac{\tau^2 (M-m^{'})^2}{TM^2}),$$

$$\alpha_G = \mathcal{O}(\frac{K^2}{MT}).$$
This completes the proof.
\end{proof}
\section{Discussion} \label{dis}

\textbf{Convergence Error:} 
The case with uniformly distributed worker information arrivals under AFL can be viewed as a uniformly independent sampling process from total workers $[M]$ under conventional FL.
Also, the case with general worker information arrival processes under AFL can be equivalently mapped to an arbitrarily independent sampling under conventional FL.
In each communication round, the surrogate objection function for partial worker participation in FL
is $\tilde{f}(x) := \frac{1}{| \mathcal{M}_t |} \sum_{i \in \mathcal{M}_t} f_i(x)$.
For uniformly independent sampling, the surrogate object function approximately equals to $f(x):= \frac{1}{M} \sum_{i=1}^{M} f_i(x)$ in expectation, i.e., $\mathbb{E}[\tilde{f}(x)] = f(x)$.
However, the surrogate object function $\tilde{f}(x)$ may deviate from $f(x)$ with arbitrarily independent sampling. 
More specifically, for uniformly independent sampling, the bound of $\| \nabla f(\x_t) -  \tilde{f}(\x_t) \|^2$ is independent of $\sigma_G$ ($A_3$ term in ~\ref{appdx_uniform}).
On the other hand, for arbitrarily independent sampling, $\| \nabla f(\x_t) -  \tilde{f}(\x_t) \|^2 \leq \mathcal{O}(\sigma_G^2)$ ($A_3$ term in ~\ref{appdx_arbipdx_uniform}).
This deviation may happen in every communication round, so it is non-vanishing even with infinity communication rounds.
As a result, such deviation is originated from the arbitrary sampling coupling with non-i.i.d. datasets.
In other words, it is irrelevant to the optimization hyper-parameters such as the learning rate, local steps and others, which is different from the objective inconsistency due to different local steps shown in \citet{wang2020fednova}.
When we set $\tau = 0$ and $K_{t,i} = K, \forall t, i$, AFA-CD generalizes FedAvg. 
In such sense, the convergence error also exists in currently synchronous FL algorithms with such arbitrarily independent sampling and non-i.i.d. dataset.
Moreover, this sampling process coupling with non-i.i.d. dataset not only results in convergence issue but also potentially induces a new source of bias/unfairness~\cite{mohri2019agnostic,li2019fair}.
So how to model the practical worker participation process in practice and in turn tackle these potential bias are worth further exploration.

\textbf{Variance Reduction:}
If we view the derivation between local loss function and global loss function as global variance, i.e., $\| \nabla f_i(\x_t) - \nabla f(\x_t) \|^2 \leq \sigma_{G}^2$, $\forall i \in [m], \forall t$ as shown in Assumption~\ref{a_variance}, the AFA-CS algorithm is indeed a variance reduction (VR) method, akin to SAG~\cite{le2012stochastic,schmidt2017minimizing}.
SAG maintains an estimate stochastic gradient $v_i, i \in [n]$ for each data point ($n$ is the size of the dataset).
In each iteration, SAG only samples one data point (say, $j$) and update the stochastic gradient on latest model ($v_j = \nabla f_j(x_t)$) stored in the memory space, but then use the average of all stored stochastic gradients as the estimate of a full gradient to update the model ($x_{t+1} = x_t - \eta_t g_t, g_t = \frac{1}{n} \sum_{i=1}^{n} v_i$).
In such way, SAG is able to have a faster convergence rate by reducing the local variance due to the stochastic gradient. 
AFA-CS algorithm performs in the similar way.
The server in the AFA-CS algorithm maintains a parameter for each worker as an estimate of the returned stochastic gradient.
In each communication round, the server only receives $m$ updates in the memory space but updates the global model by the average of all the $M$ parameters.
As a result, not only can it diminish the convergence error derived from the non-i.i.d. dataset and general worker information arrival processes (arbitrarily independent sampling), but also accelerate the convergence rate with a linear speedup factor $M$.
Previous works have applied VR methods in FL, notably SCAFFOLD~\cite{karimireddy2020scaffold} and FedSVRG~\cite{konevcny2016fedsvrg}.
The key difference is that we apply the VR on the server side to control the global variance while previous works focus on the worker side in order to tackle the model drift due to local update steps.
Applying VR methods on server and worker side are orthogonal, and thus can be used simultaneously.
We believe other variance reduction methods could be similarly extended on the server side in a similar fashion as what we do in AFA-CD.
This will be left for future research.

\section{Experiments} \label{exp}
In this section, we provide the detailed experiment settings as well as extra experimental results that cannot fit in the page limit of the main paper.

\subsection{Model and Datasets}
We run three models on three different datasets, including i) multinomial logistic regression (LR) on manually partitioned non-i.i.d. MNIST, ii) convolutional neural network (CNN) for manually partitioned non-i.i.d. CIFAR-10, and iii) recurrent neural network (RNN) on natural non-i.i.d. Shakespeare datasets.
These dataset are curated from previous FL papers~\cite{mcmahan2016communication,li2018fedprox} and are now widely used as benchmarks in FL studies~\cite{li2019convergence,yang2021achieving}.

For MNIST and CIFAR-10, each dataset has ten classes of images.
To impose statistical heterogeneity, we split the data based on the classes ($p$) of images each worker contains.
We distribute the data to $M = 10 (\text{or } 100)$ workers such that each worker contains only certain classes with the same number of training/test samples.
Specifically, each worker randomly chooses $p$ classes of labels and evenly samples training/testing data points only with these $p$ classes labels from the overall dataset without replacement.
For example,
for $p = 2$, each worker only has training/testing samples with two classes, which causes heterogeneity among different workers.
For $p = 10$, each worker has samples with ten classes, which is nearly i.i.d. case.
In this way, we can use the classes ($p$) in worker's local dataset to represent the non-i.i.d. degree qualitatively.

The Shakespeare dataset is built from {\em The Complete Works of William Shakespeare}~\cite{mcmahan2016communication}.
We use a two-layer LSTM classifier containing 100 hidden units with an embedding layer. 
The learning task is the next-character prediction, and there are 80 classes of characters in total. 
The model takes as input a sequence of 80 characters, embeds each of the characters into a learned 8-dimensional space and outputs one character per training sample after two LSTM layers and a densely-connected layer.
The dataset and model are taken from LEAF~\cite{li2018fedprox}.

For MNIST and CIFAR-10, we use global learning rate $\eta=1.0$ and local learning rate $\eta_L=0.1$.
For MNIST, the batch size is $64$ and the total communication round is $150$.
For CIFAR-10, the batch size is $500$ and the total communication round is $10000$.
For the Shakespeare dataset, the global learning rate is $\eta=50$, the local learning rate is $\eta_L=0.8$, batch size is $b=10$, and the total communication round is $300$.
In the following tables and figure captions, we use ``$m/M$'' to denote that, in each communication round, we randomly choose $m$ workers from $[M]$ to participate in the training.

We emphasis the fact that the goal here is to demonstrate our algorithms give a performance similar to other algorithms.
Note that the baseline FedAvg algorithm is server-centric and under a highly coordinated environment (synchrony, uniformly worker sampling, identical local update number, etc.).
In comparison, our AFL algorithms work in a far more chaotic environment with asynchrony, arbitrary worker's arrival process, heterogeneous local steps, non-i.i.d. data, etc.
The mere fact that AFL algorithms in a highly chaotic environment are {\em still} able to provide {\em comparable performance} to the highly coordinated FedAvg (as shown in our experiments) is surprising. 
In other words, our goal is to show that AFL algorithms can perform {\em nearly as well} in a much more chaotic environment, where traditional FL algorithms are not applicable.
Furthermore, we use communication round instead of wall-clock time to measure the model performance.
With system parameters, the wall-clock time could be easily measured. 
For example, by using random exponential time model $\lambda=1$ to simulate the stragglers~\cite{charles2021large}, 
for LR/MNIST with $p=1$, the AFL/FedAvg ratio of communication rounds to achieve $85\%$ accuracy is $61/46$, while the ratio of wall-clock time is $1/2.6$, i.e., AFL only takes $1/2.6$-fraction of  FedAvg's wall-clock time.

We study the asynchrony and heterogeneity factors in AFL, including asynchrony, heterogeneous computing, worker's arrival process, and data heterogeneity.
To simulate the asynchrony, each participated worker choose one global model from the last recent five models instead of only using the latest global model for synchronous case.
To mimic the heterogeneous computing, we simulate two cases: constant and dynamic local steps.
For constant local steps, each participated worker performs a fixed $c$ local update steps.
In contrast, each worker takes a random local update steps uniformly sampled from $[1, 2 \times c]$ for dynamic local steps.
To emulate the effect of various worker's arrival processes, we use uniform sampling without replacement to simulate the uniformly distributed worker information arrivals, and we use biased sampling with probability $[0.19, 0.19, 0.1, 0.1, 0.1, 0.1, 0.1, 0.1, 0.01, 0.01]$ without replacement for total $10$ workers to investigate potential biases with general worker information arrival processes.
To study the data heterogeneity, we use the value $p$ as a proxy to represent the non-i.i.d. degree for MNIST and CIFAR-10.

\begin{table} [!htb]
\begin{center}
\caption{CNN Architecture for CIFAR-10.}
\begin{tabular}{c c} 
	\hline
	Layer Type & Size \\
	\hline
	Convolution + ReLu & $5 \times 5 \times 32$ \\
	Max Pooling & $2 \times 2$\\
	Convolution + ReLu & $5 \times 5 \times 64$ \\
	Max Pooling & $2 \times 2$\\
	Fully Connected + ReLU & $1024 \times 512$\\
	Fully Connected + ReLU & $512 \times 128$ \\
	Fully Connected & $128 \times 10$ \\
	\hline
\end{tabular}
\label{tab: cnn}
\end{center}
\end{table}

\subsection{Further experimental results}

\begin{table}[!htb]
	\centering
	\addtolength{\tabcolsep}{-2.55pt}
	\setlength\extrarowheight{1.5pt}
	\caption{Test Accuracy for comparison of asynchrony and local steps.}
	\label{tab:AFL}
	\begin{tabular}{|c|c|c|c|c|c|c|c|c|c|}
		\hline
		\multirow{2}[4]{*}{\makecell {Models/ \\Dataset}} & \multirow{2}[4]{*}{ \makecell {Non-i.i.d. \\ index (p) }} & \multicolumn{1}{c|}{\multirow{2}[4]{*}{\makecell {Worker \\ number } }} & \multicolumn{1}{c|}{\multirow{2}[4]{*}{ \makecell {Local \\ steps }}} & \multicolumn{2}{c|}{Synchrony} & \multicolumn{2}{c|}{Asynchrony} \\
		\cline{5-8}          &       &       & &\makecell{Constant \\ steps} & \makecell{Dynamic \\ steps} & \makecell{Constant \\ Steps} & \makecell{Dynamic \\ Steps} \\
		\hline
		\hline
		\multirow{12}{*}{\makecell{LR/ \\ MNIST}} & $p=1$ & 5/10 & 5 & 0.8916 & 0.8915 & 0.8888 & 0.8868 \\
		& $p=2$ & 5/10 & 5 & 0.8906 & 0.8981 & 0.8901 & 0.8931  \\
		& $p=5$ & 5/10 & 5 & 0.9072 & 0.9075 & 0.9059 & 0.9048  \\
		& $p=10$ & 5/10 & 5 & 0.9114 & 0.9111 & 0.9129 & 0.9143 \\
		\cline{2-8}
		 & $p=1$ & 5/10 & 10 & 0.8743 & 0.8786 & 0.8701 & 0.8734 \\
		 & $p=2$ & 5/10 & 10 & 0.8687 & 0.8813 & 0.8661 & 0.8819 \\
		 & $p=5$ & 5/10 & 10 & 0.9016 & 0.9050 & 0.9034 & 0.9065 \\
		 & $p=10$ & 5/10 & 10 & 0.9124 & 0.9135 & 0.9112 & 0.9111 \\
		\cline{2-8}
		 & $p=1$ & 20/100 & 5 & 0.8898 & 0.8973 & 0.8909 & 0.8938 \\
		 & $p=2$ & 20/100 & 5 & 0.8968 & 0.9007 & 0.8955 & 0.9000 \\
		 & $p=5$ & 20/100 & 5 & 0.9088 & 0.9088 & 0.9097 & 0.9078 \\
		 & $p=10$ & 20/100 & 5 & 0.9111 & 0.9106 & 0.9126 & 0.9125 \\
		\hline
		\hline
		\multirow{4}{*}{\makecell{CNN/ \\ CIFAR-10}} & $p=1$ & 5/10 & 5 & 0.7474 & 0.7606 & 0.7319 & 0.7350 \\
		& $p=2$ & 5/10 & 5 & 0.7677 & 0.7944 & 0.7662 & 0.777 \\
		& $p=5$ & 5/10 & 5 & 0.7981 & 0.802 & 0.8065 & 0.799 \\
		& $p=10$ & 5/10 & 5 & 0.8081 & 0.8072 & 0.8065 & 0.8119 \\
		\hline
		\hline
		\makecell{RNN/ \\
		 Shakespeare} & - & 72/143 & 50 & 0.4683 & 0.4831 & 0.4606 & 0.4687 \\
		\hline
	\end{tabular}%
\vspace{-.05in}
\end{table}

\textbf{Effect of asynchrony, local update steps, and non-i.i.d. level.}
In table~\ref{tab:AFL}, we examine three factors by comparing the top-1 test accuracy: synchrony versus asynchrony, constant steps versus dynamic steps and different levels of non-i.i.d. dataset.
The worker sampling process is uniformly random sampling to simulate the uniformly distributed worker information arrivals.
The baseline is synchrony with constant steps.
When using asynchrony or/and dynamic local steps, the top-1 test accuracy shows no obvious differences.
This observation can be observed in all these three tasks.
Asynchrony and dynamic local update steps enable each worker to participate flexibly and loosen the coupling between workers and the server.
As a result, asynchrony and dynamic local steps introduce extra heterogeneity factors, but the performance of the model is as good as that of the synchronous approaches with  constant local steps. 
Instead, the data heterogeneity is an important factor for the model performance.
As the non-i.i.d. level increases (smaller $p$ value), the top-1 test accuracy decreases. 

Next, we study convergence speed of the test accuracy for the model training under different settings.
Figure~\ref{appdx:lr_acc} illustrates the test accuracy for LR on MNIST with different non-i.i.d. levels. 
We can see that asynchrony and dynamic local steps result in zigzagging convergence curves, but the final accuracy results have negligible differences.
The zigzagging phenomenon is more dramatic as the non-i.i.d. level gets higher.
Interestingly, from Figure~\ref{appdx:cnn_acc} and Figure~\ref{appdx:lstm_acc}, we can see that for less non-i.i.d. settings such as $p=10$ and $p=5$, the curves of all algorithms are almost identical. 
Specifically, in Figure~\ref{appdx:lstm_acc}, the test accuracy curves of the LSTM model oscillates under asynchrony and dynamic local steps.
Another observation is that it takes more rounds to converge as the non-i.i.d. level of the datasets increases.
This trend can be clearly observed in Figure~\ref{appdx:cnn_acc}.


\begin{figure*}[!ht]
\centering
	\begin{subfigure}[b]{0.4\textwidth}
        {\includegraphics[width=1.\textwidth]{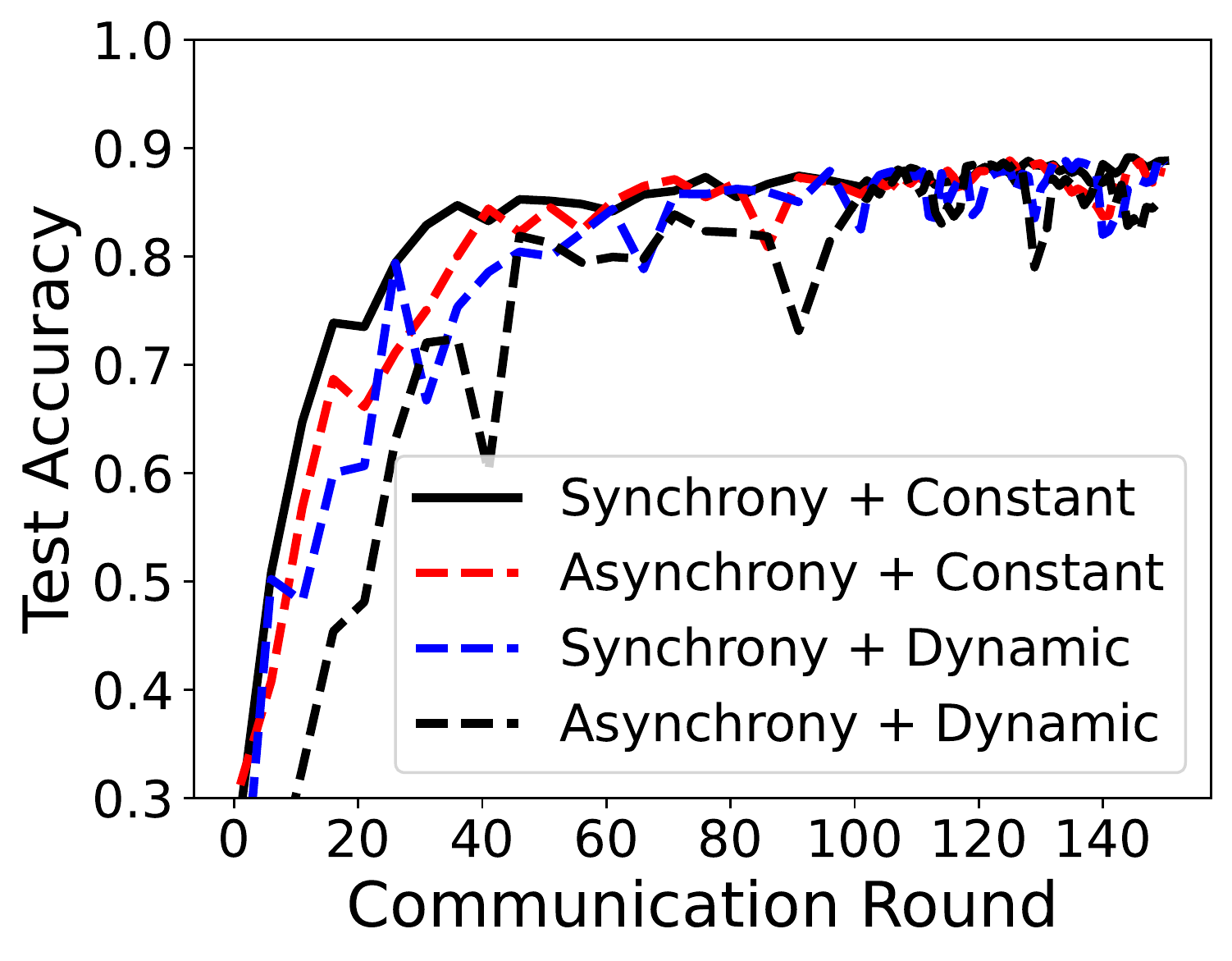}} 
        \caption{$p=1$.}
        \label{mnist_1}
    \end{subfigure}
    \qquad
    \begin{subfigure}[b]{0.4\textwidth}
        {\includegraphics[width=1.\textwidth]{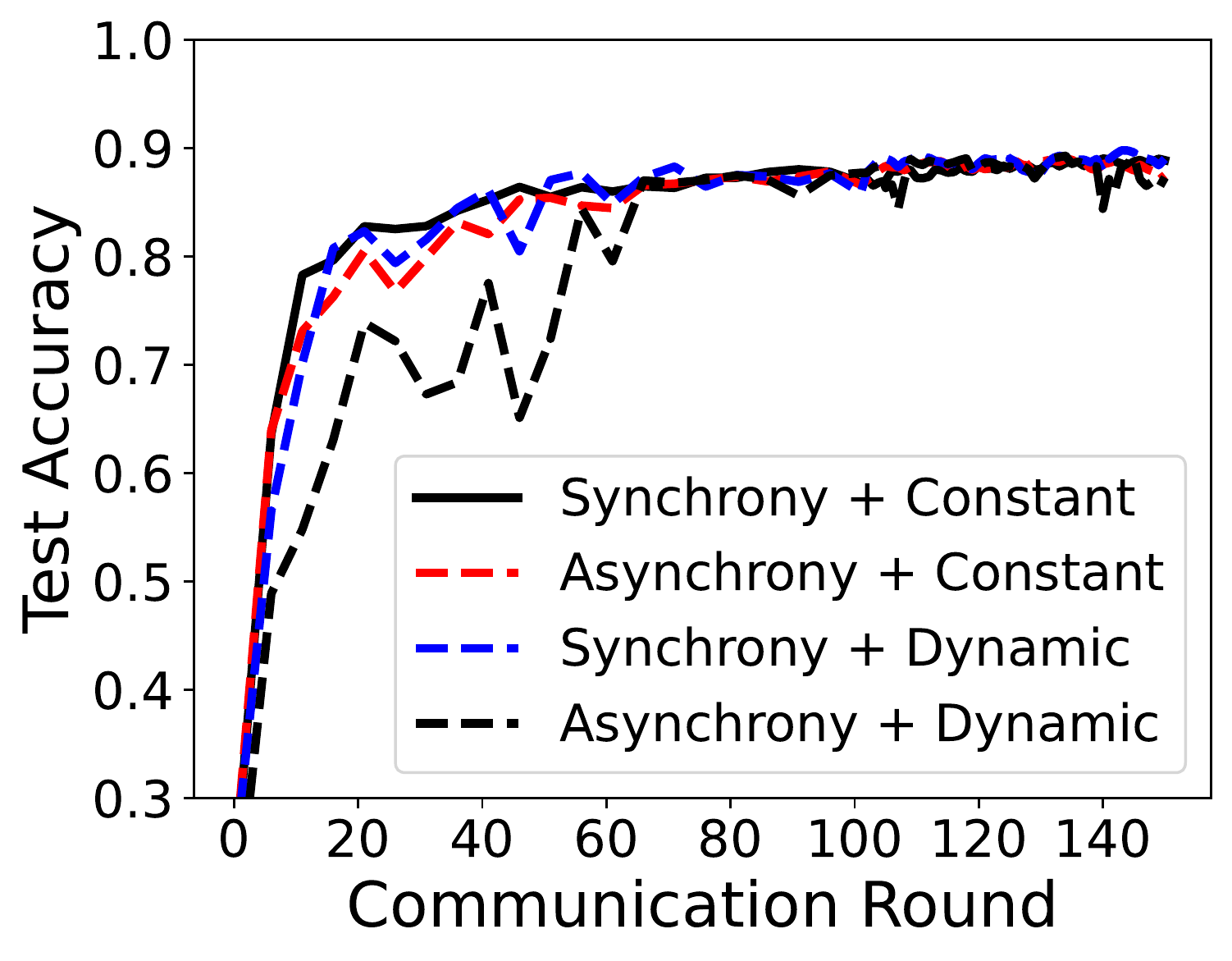}} 
        \caption{$p=2$.}
        \label{mnist_2}
    \end{subfigure} \\
    \begin{subfigure}[b]{0.4\textwidth}
        {\includegraphics[width=1.\textwidth]{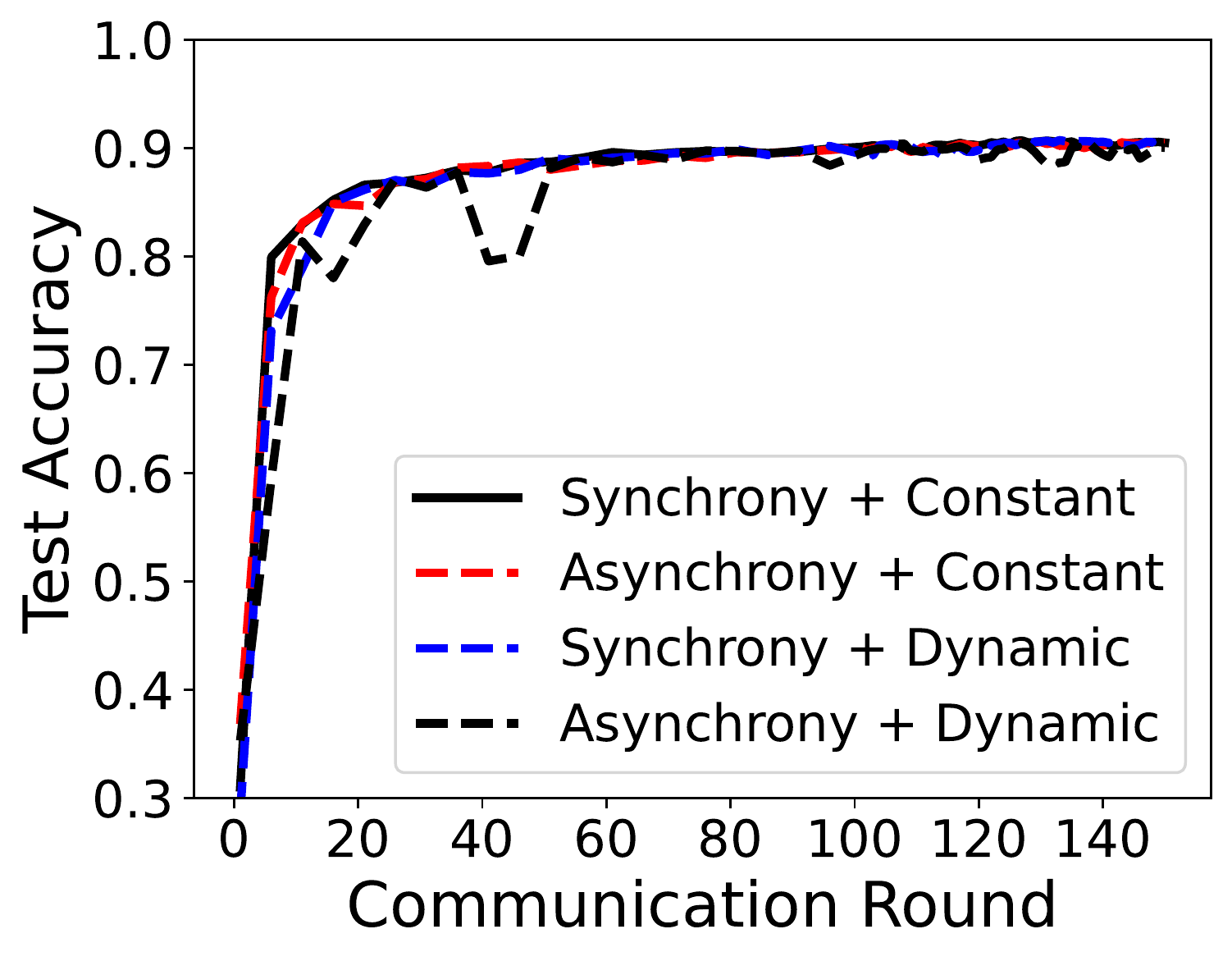}} 
        \caption{$p=5$.}
        \label{mnist_5}
    \end{subfigure}
    \qquad
    \begin{subfigure}[b]{0.4\textwidth}
        {\includegraphics[width=1.\textwidth]{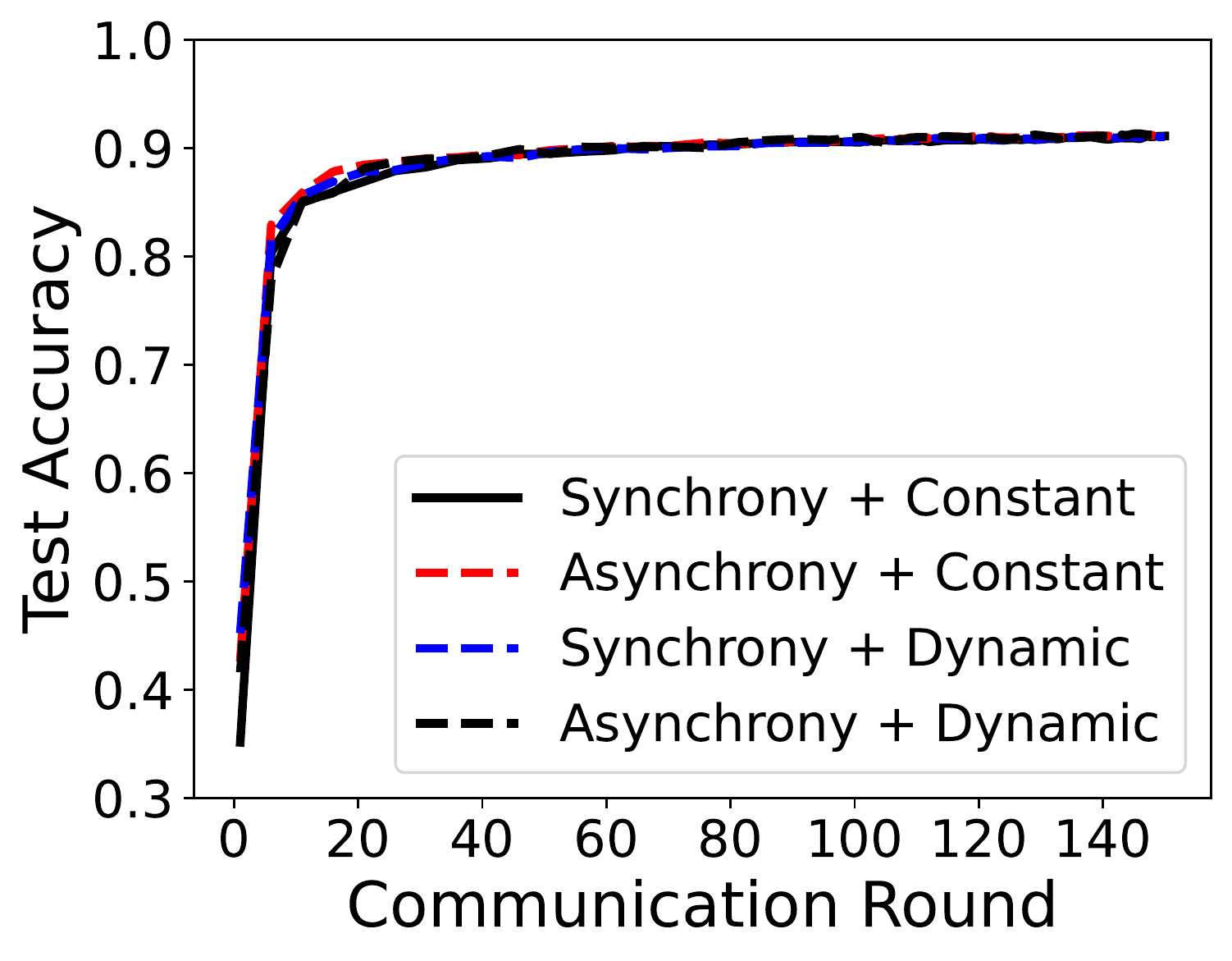}} 
        \caption{$p=10$.}
        \label{mnist_10}
    \end{subfigure}
\caption{Test accuracy for LR on MNIST with worker number $5/10$, local steps $5$.}
\label{appdx:lr_acc}
\vspace{-.1in}
\end{figure*}%

\begin{figure*}[!ht]
\centering
	\begin{subfigure}[b]{0.4\textwidth}
        {\includegraphics[width=1.\textwidth]{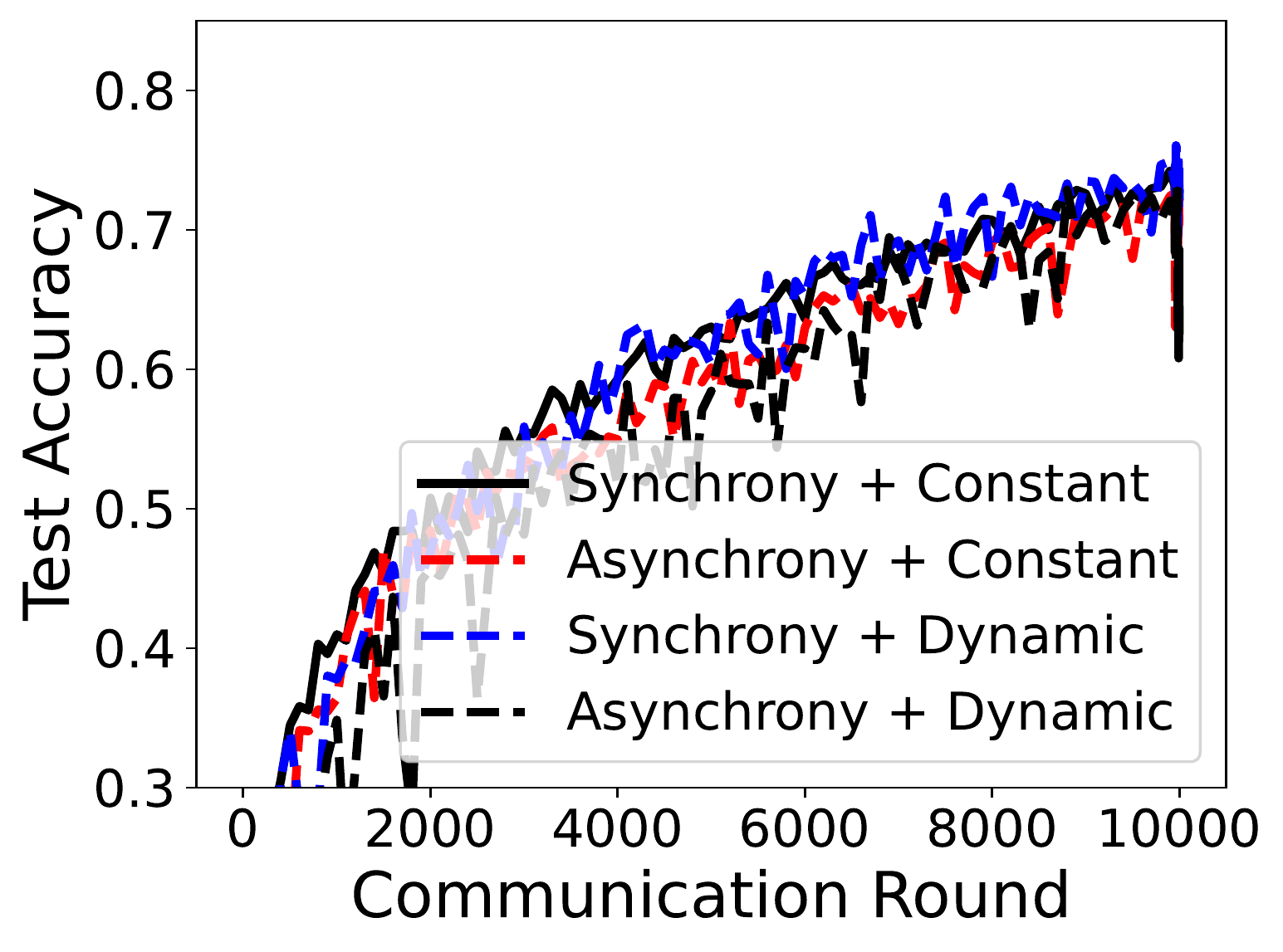}} 
        \caption{$p=1$.}
        \label{cnn_1}
    \end{subfigure}
    \qquad
	\begin{subfigure}[b]{0.4\textwidth}
        {\includegraphics[width=1.\textwidth]{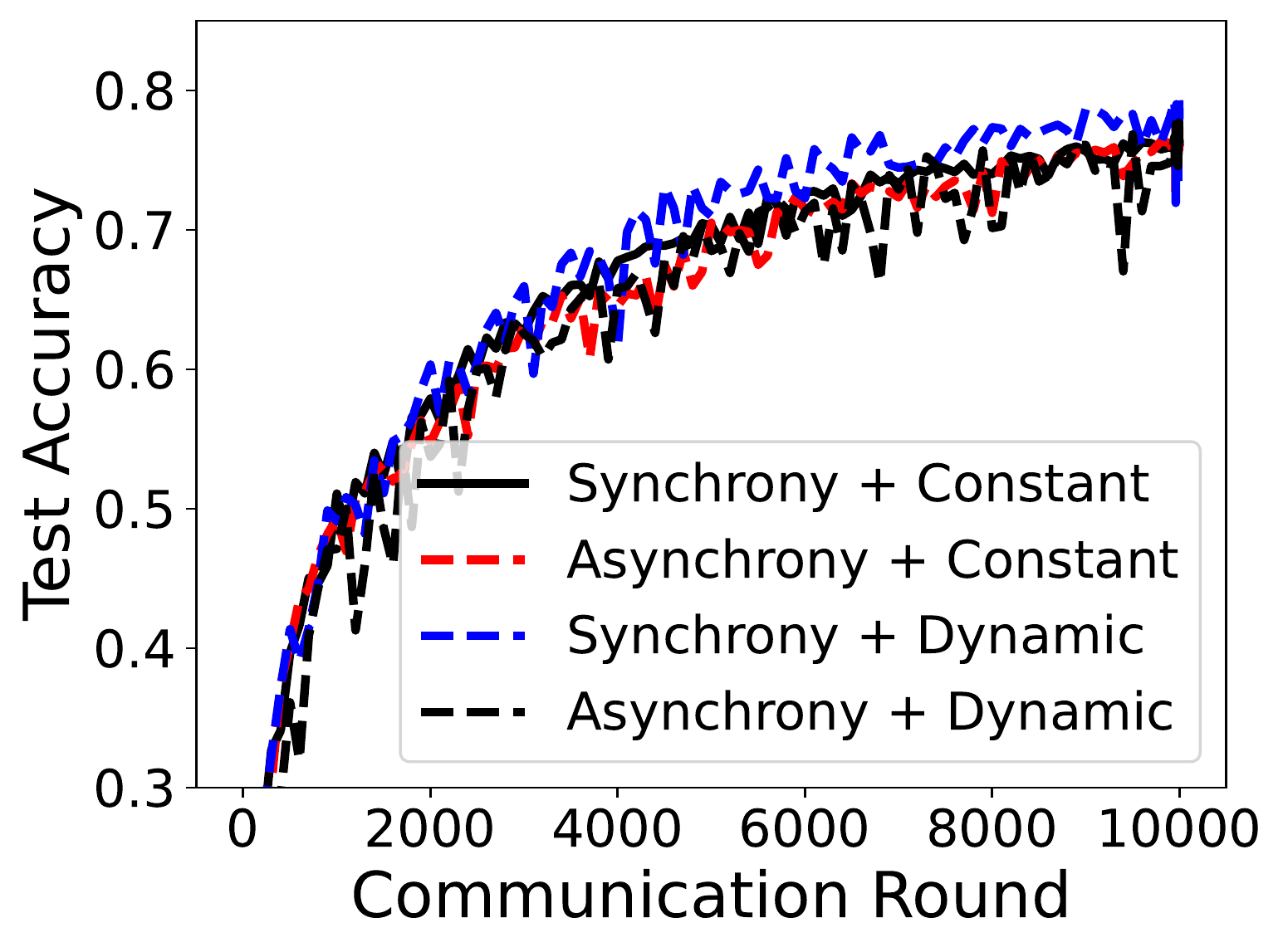}} 
        \caption{$p=2$.}
        \label{cnn_2}
    \end{subfigure} \\
    \begin{subfigure}[b]{0.4\textwidth}
        {\includegraphics[width=1.\textwidth]{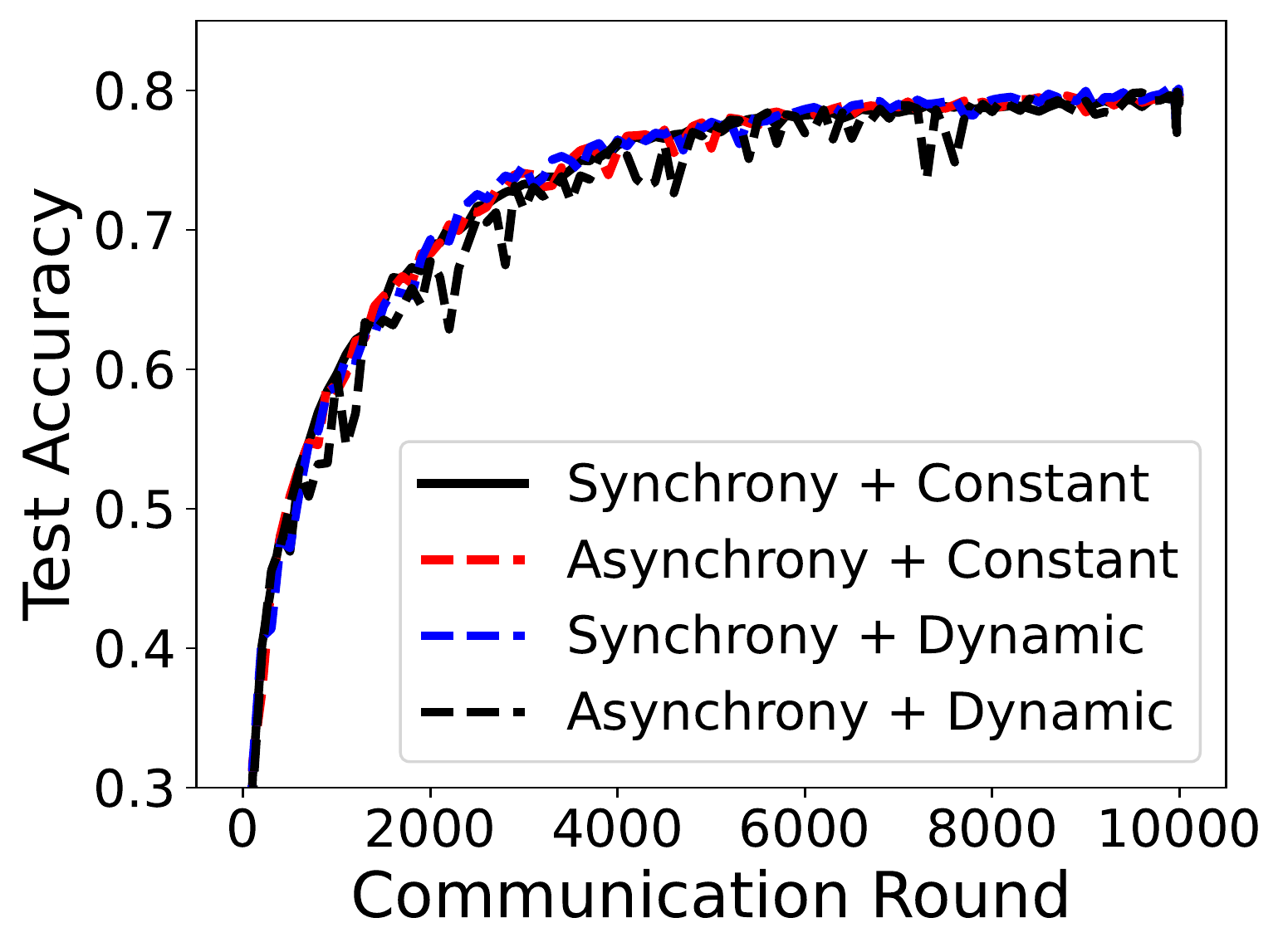}} 
        \caption{$p=5$.}
        \label{cnn_5}
    \end{subfigure}
    \qquad
    \begin{subfigure}[b]{0.4\textwidth}
        {\includegraphics[width=1.\textwidth]{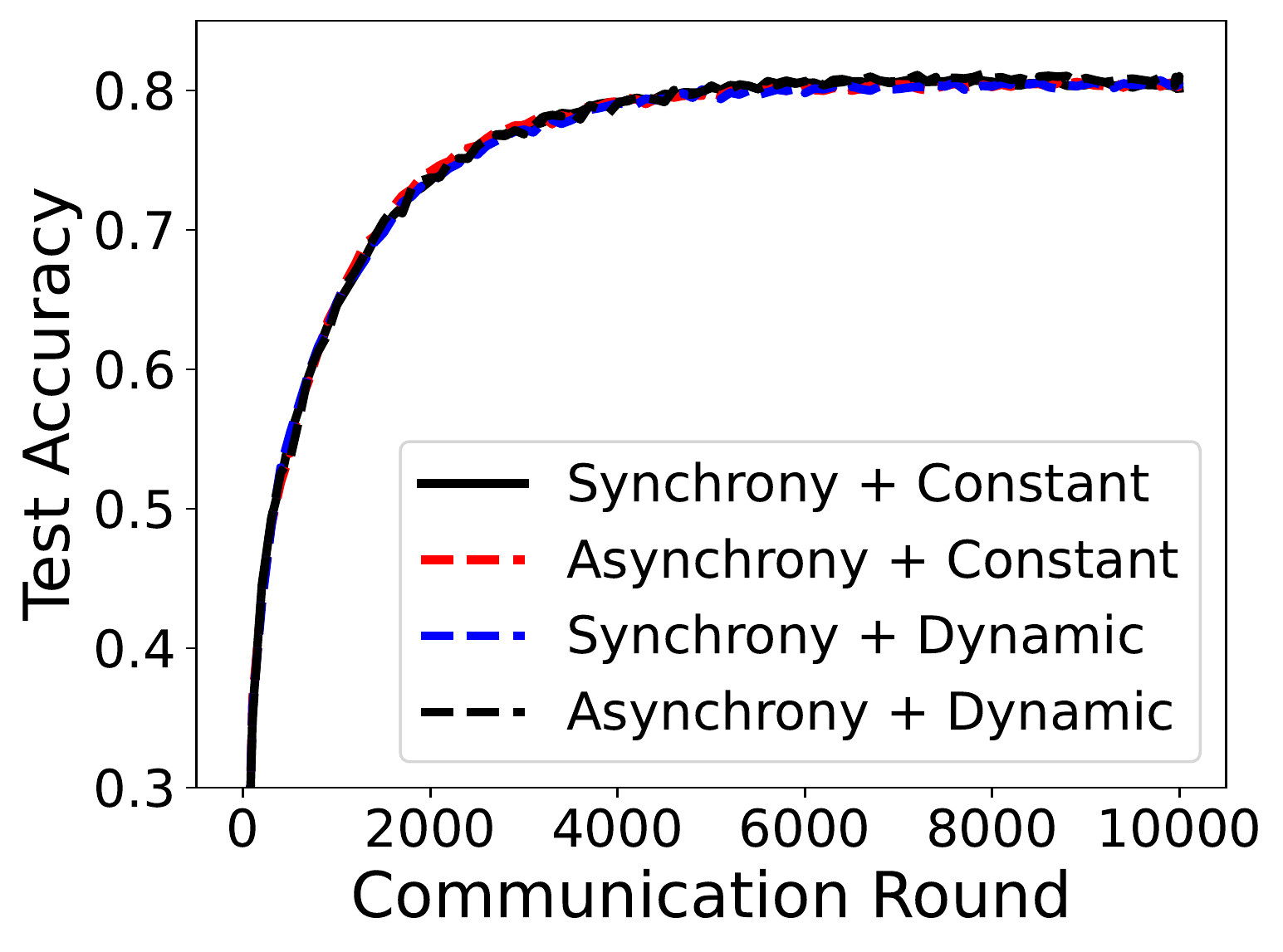}} 
        \caption{$p=10$.}
        \label{cnn_10}
    \end{subfigure}
    \qquad
\caption{Test accuracy for CNN on CIFAR-10 with worker number $5/10$, local steps $5$.}
\label{appdx:cnn_acc}
\vspace{-.1in}
\end{figure*}%

\begin{figure*}[!ht]
\centering
\begin{subfigure}[b]{0.4\textwidth}
\includegraphics[width=1.\textwidth]{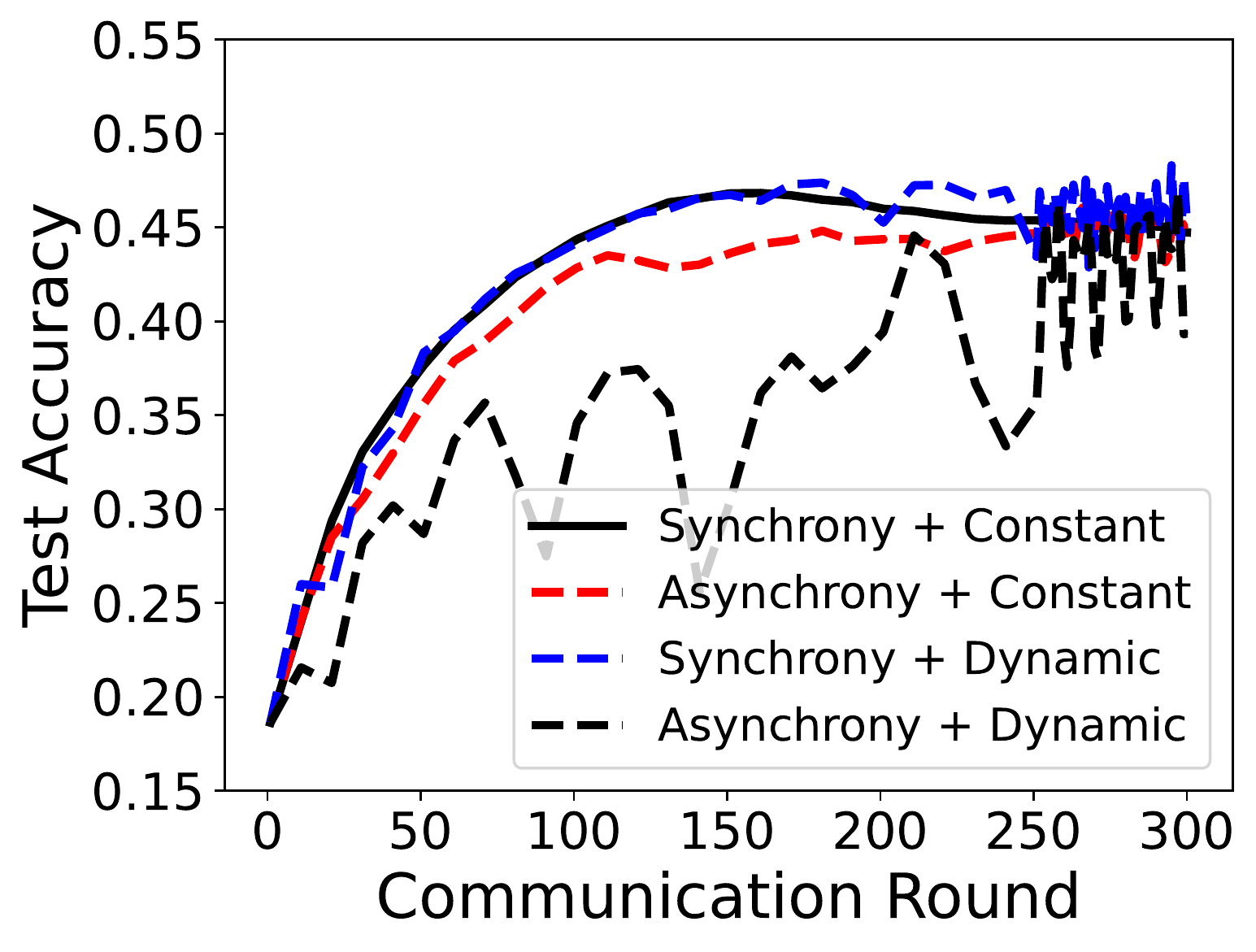}
\end{subfigure}
\caption{Test accuracy for LSTM on Shakespeare with worker number $72/143$, local steps $50$.}
\label{appdx:lstm_acc}
\vspace{-.1in}
\end{figure*}

\begin{table}[!ht]
	\centering
	\addtolength{\tabcolsep}{-2.55pt}
	\setlength\extrarowheight{1.5pt}
	\caption{Test Accuracy of FedProx and SCAFFOLD.}
	\label{tab:AFL+}
	\begin{tabular}{|c|c|c|c|c|c|c|c|c|c|}
		\hline
		\multirow{2}[1]{*}{\makecell {Models/ \\Dataset}} & \multirow{2}[1]{*}{ \makecell {Non-i.i.d. \\ index (p) }} & \multicolumn{1}{c|}{\multirow{2}[1]{*}{\makecell {Worker \\ number } }} & \multicolumn{1}{c|}{\multirow{2}[1]{*}{ \makecell {Local \\ steps }}} & \multirow{2}[1]{*}{\makecell {FedProx}} & \multirow{2}[1]{*}{\makecell{SCAFFOLD}} & \multirow{2}[1]{*}{ \makecell{AFL + \\ FedProx }} &   \multirow{2}[1]{*}{ \makecell {AFL + \\ SCAFFOLD }}\\
		 &       &       & & & &  &  \\
		\hline
		\hline
		\multirow{8}{*}{\makecell{LR/ \\ MNIST}} & $p=1$ & 5/10 & 5 & 0.8893 & 0.8928 & 0.8775 & 0.8946 \\
		& $p=2$ & 5/10 & 5 & 0.8868 & 0.8970 & 0.8832 & 0.8954   \\
		& $p=5$ & 5/10 & 5 & 0.9036 & 0.9032 & 0.9004 & 0.9019   \\
		& $p=10$ & 5/10 & 5 & 0.9075 & 0.9057 & 0.9054 & 0.9022  \\
		\cline{2-8}
		& $p=1$ & 5/10 & 10 & 0.8752 & 0.8789 & 0.8669 & 0.8838 \\
		& $p=2$ & 5/10 & 10 & 0.8685 & 0.8967 & 0.8789 & 0.8978 \\
		& $p=5$ & 5/10 & 10 & 0.9019 & 0.9047 & 0.8998 & 0.9029 \\
		& $p=10$ & 5/10 & 10 & 0.9072 & 0.9071 & 0.9052 & 0.9038 \\
		\hline
		\hline
		\multirow{4}{*}{\makecell{CNN/ \\ CIFAR-10}} & $p=1$ & 5/10 & 5 & 0.7488 & 0.1641 & 0.7415 & 0.3935  \\
		& $p=2$ & 5/10 & 5 & 0.7728 & 0.6315 & 0.7890 & 0.6971  \\
		& $p=5$ & 5/10 & 5 & 0.7931 & 0.7828 & 0.8031 & 0.7884  \\
		& $p=10$ & 5/10 & 5 & 0.8150 & 0.8083 & 0.8143 & 0.8051  \\
		\hline
		\hline
		\makecell{RNN/ \\ Shakespeare} & - & 72/143 & 50 & 0.4690 & 0.4794 & 0.4550 & 0.4515  \\
		\hline
	\end{tabular}%
\vspace{-.05in}
\end{table}

\textbf{Utilizing FedProx and SCAFFOLD as the optimizer on the worker-side.}
Here, we choose FedProx and SCAFFOLD as two classes of algorithms in existing FL algorithms.
FedProx represents these algorithms that modifies the local objective function.
Other algorithms belonging to this category includes FedPD~\cite{zhang2020fedpd} and FedDyn~\cite{acar2021feddyn}.
In such algorithms, no extra information exchange between worker and server is needed.
On the other hand, SCAFFOLD represents VR-based (variance reduction) algorithms.
It needs an extra control variate to perform the ``variance reduction'' step, so extra parameters are required in each communication round.
Other algorithms in this class includes FedSVRG~\cite{konevcny2016fedsvrg}.

In Table~\ref{tab:AFL+}, we show the effectiveness of utilizing existing FL algorithms, FedProx and SCAFFOLD, in the AFL framework.
For FedProx and SCAFFOLD, we examine synchrony and constant local steps settings.
When incorporating these two advanced FL algorithms in the AFL framework, we study the effects of asynchrony and dynamic local steps.
We set $\mu=0.1$ as default in FedProx algorithm.
We can see from Table~\ref{tab:AFL+} that FedProx performs as good as FedAvg does (compare with the results in Table~\ref{tab:AFL}).
Also, there is no performance degradation in AFL framework by utilizing FedProx as the worker's optimizer.
However, while SCAFFOLD performs well for LR on MNIST, it dose not work well for CNN on CIFAR-10, especially in cases with higher non-i.i.d. levels.
One possible reason is that the control variates can become stale in partial worker participation and in turn degrade the performance.
Previous work also showed similar results~\cite{acar2021feddyn,reddi2020adaptive}.
If we view the SCAFFOLD ( in synchrony and constant steps setting) as the baseline, no obvious performance degradation happens under AFL with SCAFFOLD being used as the worker's optimizer.

\begin{figure*}[!ht]
\centering
	\begin{subfigure}[b]{0.4\textwidth}
        {\includegraphics[width=1.\textwidth]{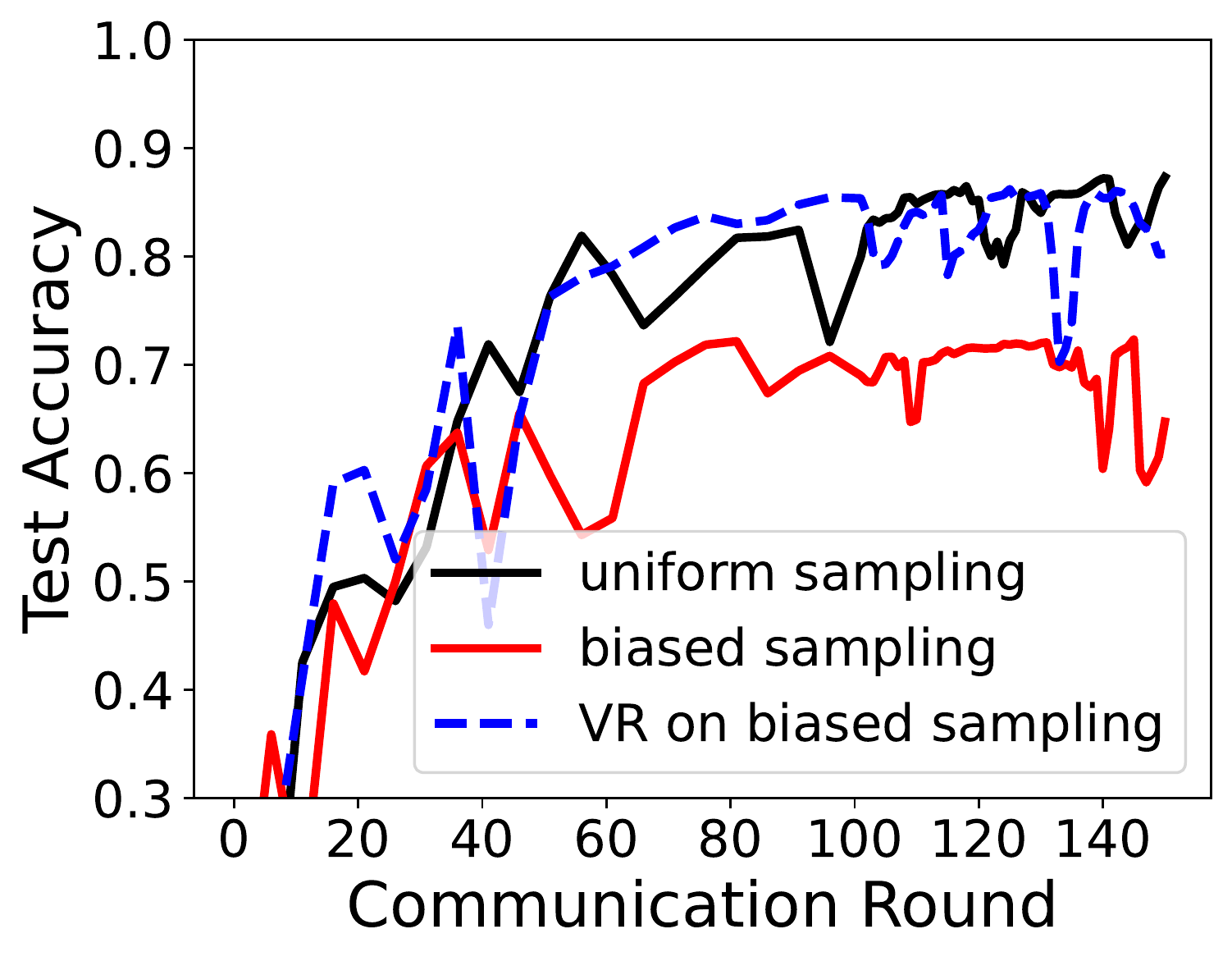}}
        \caption{$p=1$.}
        \label{lr_bias_1}
    \end{subfigure}
    \qquad
	\begin{subfigure}[b]{0.4\textwidth}
        {\includegraphics[width=1.\textwidth]{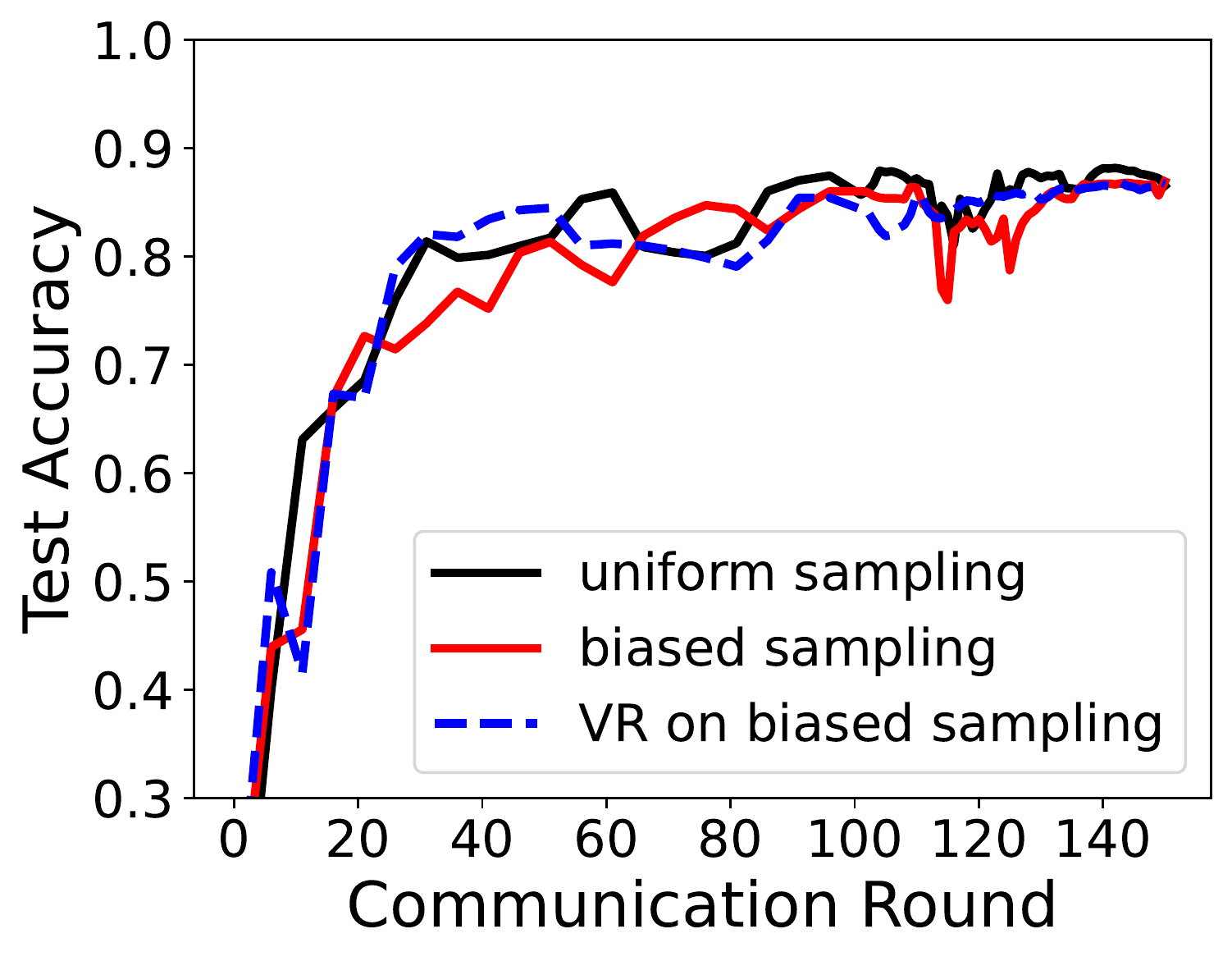}}
        \caption{$p=2$.}
        \label{lr_bias_2}
    \end{subfigure} \\
    \begin{subfigure}[b]{0.4\textwidth}
        {\includegraphics[width=1.\textwidth]{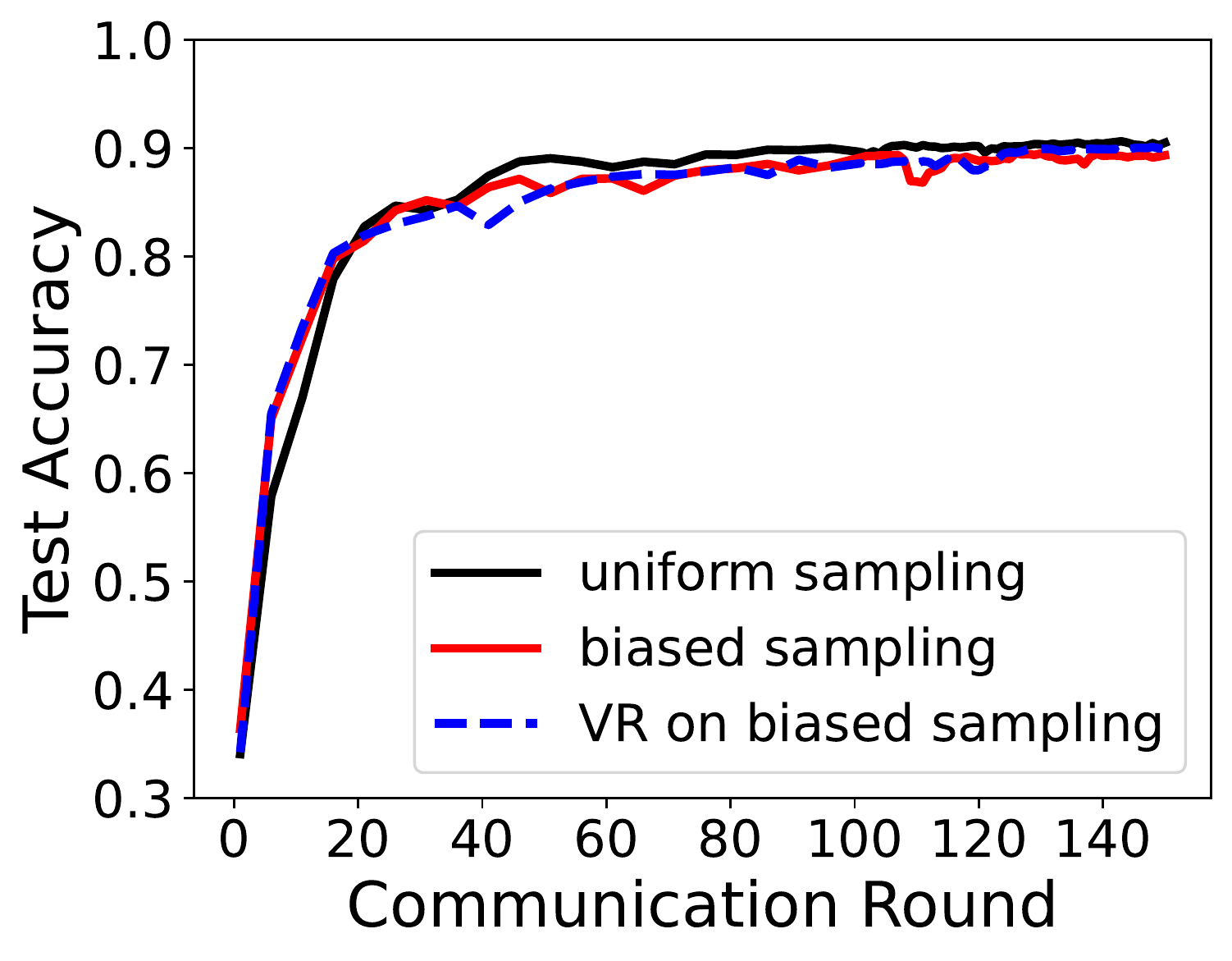}}
        \caption{$p=5$.}
        \label{lr_bias_5}
    \end{subfigure}
    \qquad
    \begin{subfigure}[b]{0.4\textwidth}
        {\includegraphics[width=1.\textwidth]{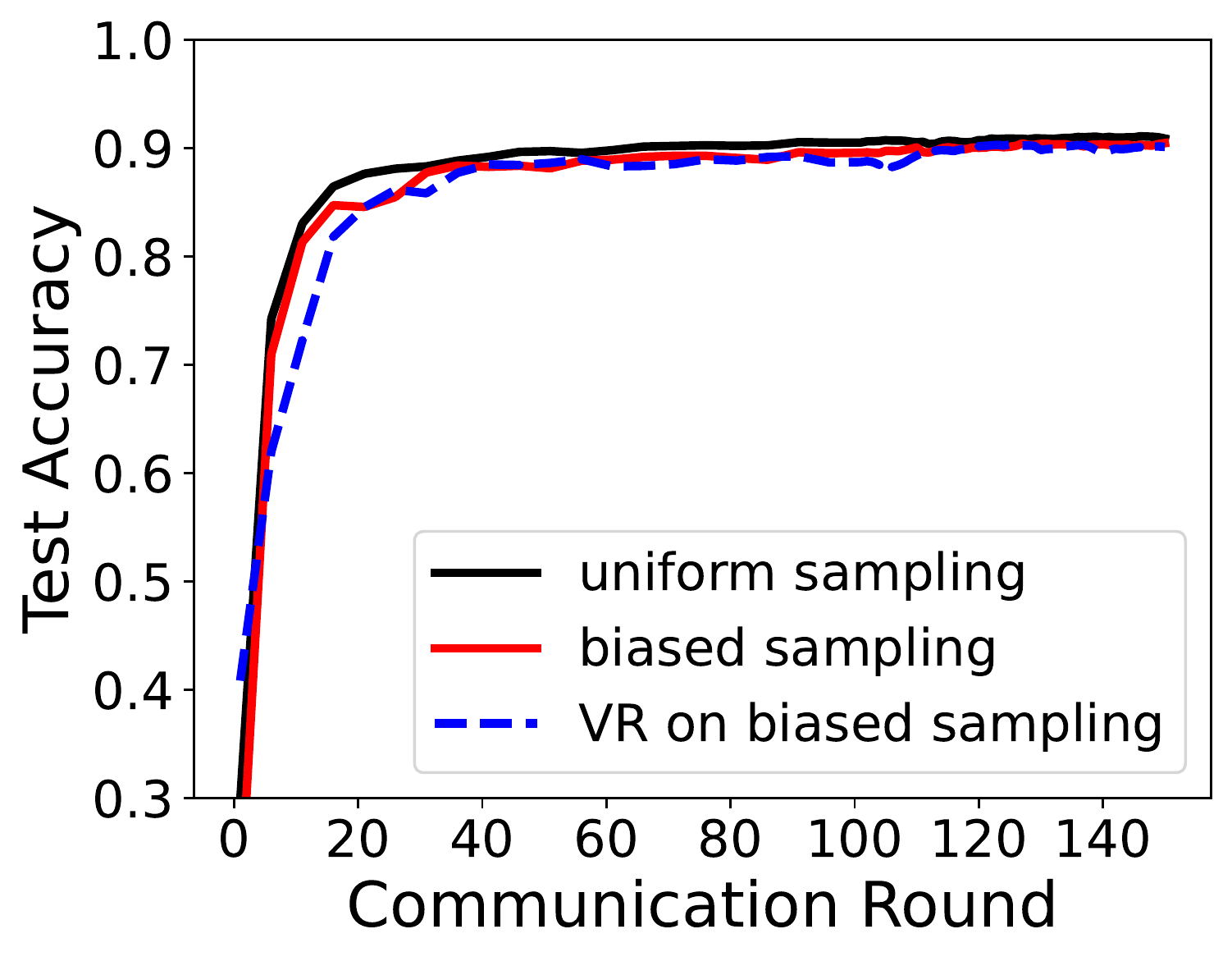}}
        \caption{$p=10$.}
        \label{lr_bias_10}
    \end{subfigure}
\caption{Test accuracy for LR on MNIST with asynchrony and dynamic local steps.}
\label{appdx:lr_acc_biased}
\vspace{-.1in}
\end{figure*}%

\begin{figure*}[!ht]
\centering
	\begin{subfigure}[b]{0.4\textwidth}
        {\includegraphics[width=1.\textwidth]{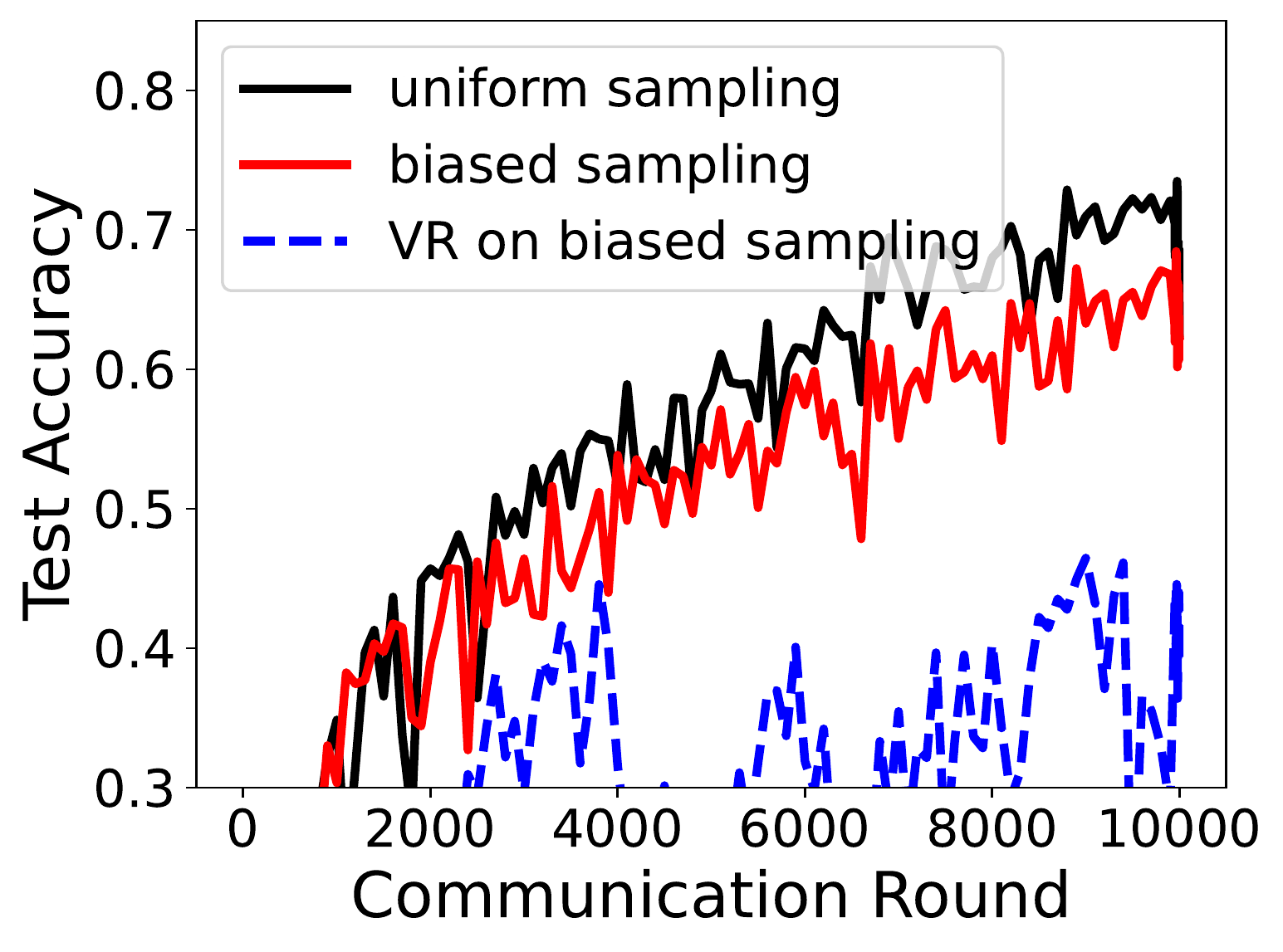}}
        \caption{$p=1$.}
        \label{cnn_bias_1}
    \end{subfigure}
    \qquad
	\begin{subfigure}[b]{0.4\textwidth}
        {\includegraphics[width=1.\textwidth]{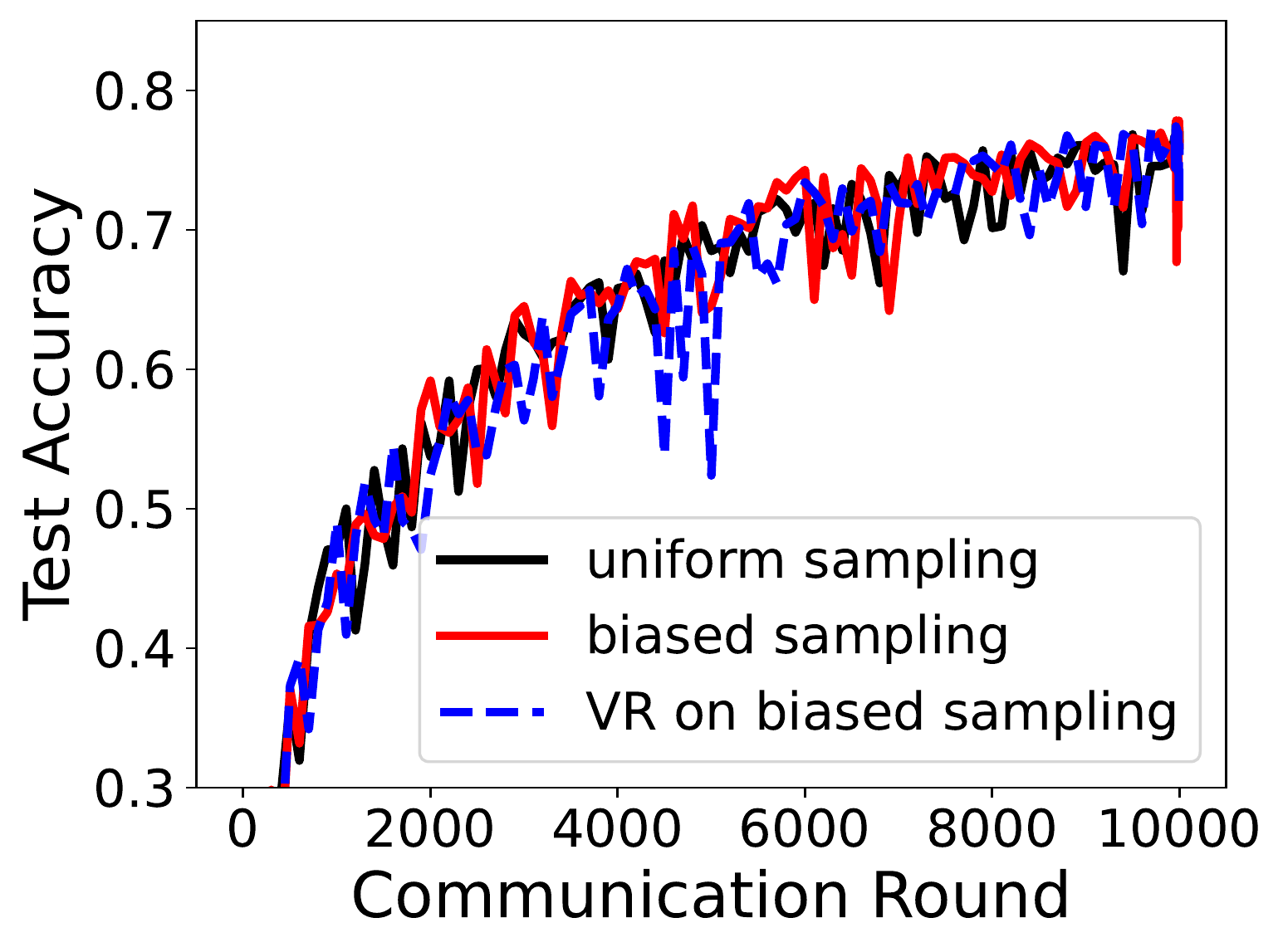}}
        \caption{$p=2$.}
        \label{cnn_bias_2}
    \end{subfigure} \\
    \begin{subfigure}[b]{0.4\textwidth}
        {\includegraphics[width=1.\textwidth]{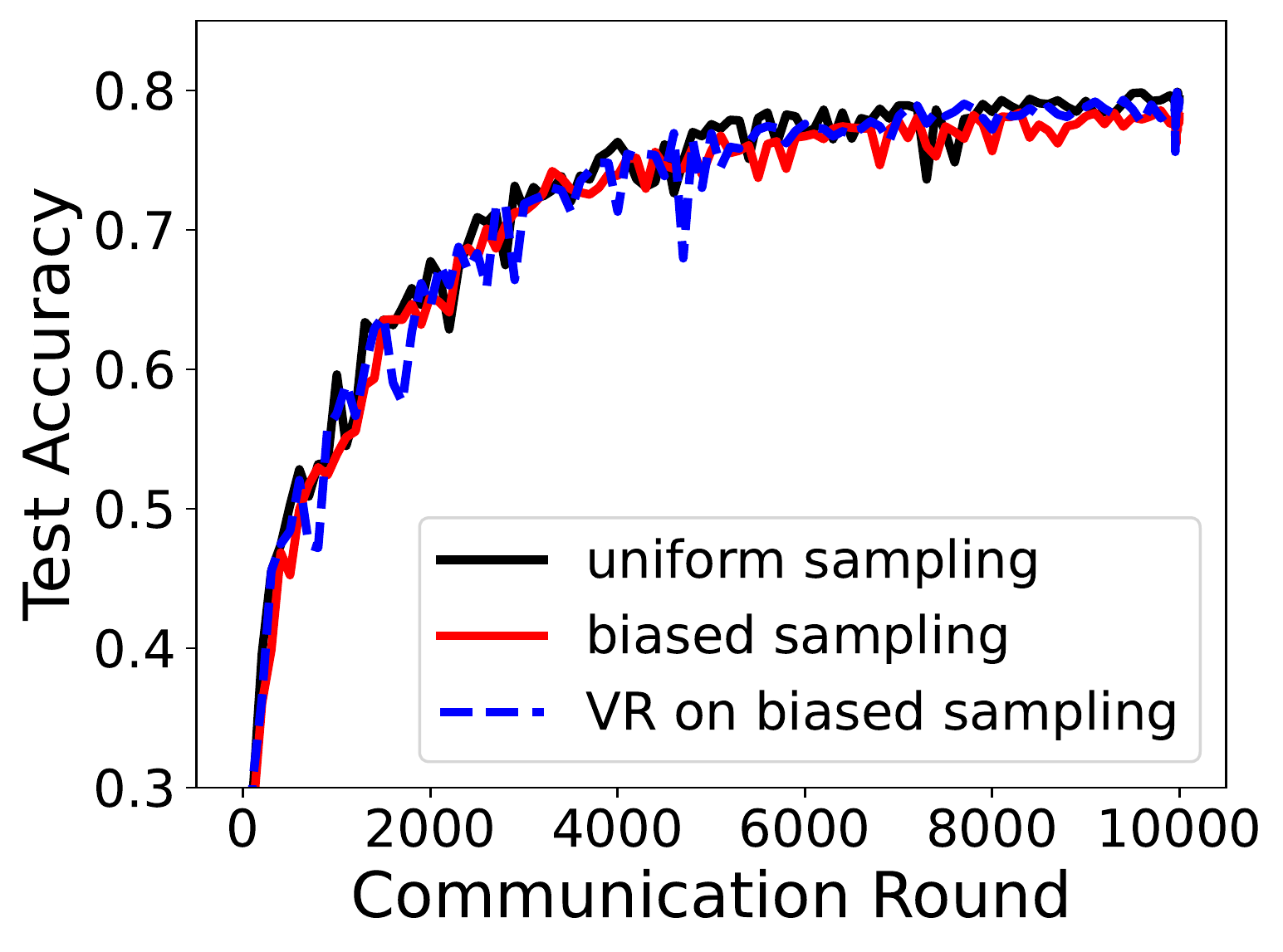}}
        \caption{$p=5$.}
        \label{cnn_bias_5}
    \end{subfigure}
    \qquad
    \begin{subfigure}[b]{0.4\textwidth}
        {\includegraphics[width=1.\textwidth]{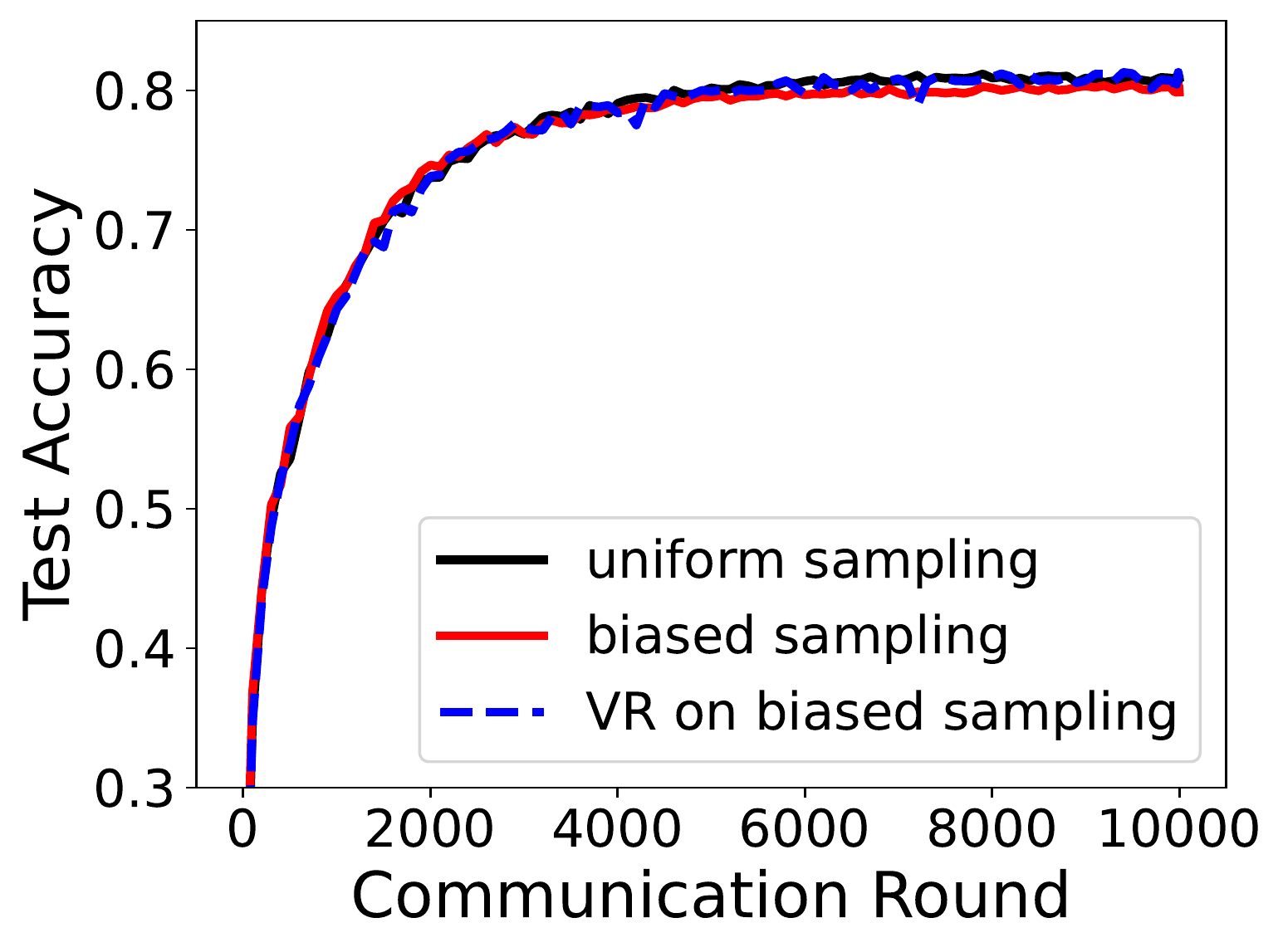}}
        \caption{$p=10$.}
        \label{cnn_bias_10}
    \end{subfigure}
\caption{Test accuracy for CNN on CIFAR-10 with asynchrony and dynamic local steps.}
\label{appdx:cnn_acc_baised}
\vspace{-.1in}
\end{figure*}%

\textbf{Effects of different worker information arrival processes.}
In order to generate different workers' arrival processes, we use uniform sampling without replacement to simulate the uniformly distributed worker information arrivals and use biased sampling to simulate the potential bias in general worker information arrival processes.
In Figures~\ref{appdx:lr_acc_biased} and ~\ref{appdx:cnn_acc_baised}, we illustrate the effect of the sampling process for LR on MNIST and CNN on CIFAR-10 with asynchrony and dynamic local steps.
For highly non-i.i.d. datasets ($p=1$), the biased sampling process degrades the model performance.
This is consistent with the larger convergence error as shown in our theoretical analysis.
On the other hand,
for other non-i.i.d. cases with $p=2, 5, 10$, such biased sampling dose not lead to significant performance degradation.
When applying variance reduction on such biased sampling process by reusing old gradients as shown in AFA-CS, we can see that AFA-CS performs well on MNIST, but not on CIFAR-10. 
We conjecture that AFA-CS, as a variance reduction method, does not always perform well in practice.
This observation is consistent with the previous work~\cite{defazio2018ineffectiveness,reddi2020adaptive}, which also demonstrated the ineffectiveness of variance reduction methods in deep learning and some cases of FL.



\end{document}